\documentclass[10pt,journal,compsoc]{IEEEtran}

\usepackage{amsmath}
\usepackage{amssymb}
\usepackage{amsthm}
\usepackage{xcolor}
\usepackage{booktabs}
\usepackage{graphicx}
\usepackage{hyperref}
\usepackage{dblfloatfix}
\usepackage{oubraces}

\newtheorem{proposition}{Proposition} %[section]
\title{Disentangled Explanations of Neural Network Predictions by Finding Relevant Subspaces}

\author{Pattarawat Chormai, Jan Herrmann, Klaus-Robert M\"uller, Gr\'egoire Montavon${^*}$ 
\thanks{P. Chormai is with the Machine Learning Group, Technische Universit\"{a}t Berlin (TU Berlin), 10587 Berlin, Germany, with the Max Planck School of Cognition, Max Planck Institute for Human Cognitive and Brain Sciences, 04103 Leipzig, Germany, and also with the Konrad Zuse School of Excellence in Learning and Intelligent Systems (ELIZA), 64289 Darmstadt, Germany. E-mail: p.chormai@tu-berlin.de.}
\thanks{J. Herrmann is with BASF SE, Statistics and Machine Learning, Carl-Bosch-Straße 38, 67056 Ludwigshafen am Rhein, Germany. E-mail: jan.herrmann@basf.com.}
\thanks{K.-R. M\"uller is with the Machine Learning Group,  Technische Universit\"{a}t Berlin (TU Berlin), 10587 Berlin, Germany, with the Department of Artificial Intelligence, Korea University, Seoul 136-713, Korea, with  the Max Planck Institute for Informatics, 66123 Saarbr{\"u}cken, Germany, and also with BIFOLD---Berlin Institute for the Foundations of Learning and Data, 10587 Berlin, Germany. E-mail: klaus-robert.mueller@tu-berlin.de.}
\thanks{G. Montavon is with the Department of Mathematics and Computer Science,
Freie Universit\"{a}t Berlin (FU Berlin), 14195 Berlin, Germany, with the Machine Learning Group, Technische Universit\"{a}t Berlin (TU Berlin), 10587 Berlin, Germany, and also with BIFOLD---Berlin Institute for the Foundations of Learning and Data, 10587 Berlin, Germany. E-mail: gregoire.montavon@fu-berlin.de.}
\thanks{(Corresponding Author: Gr\'egoire Montavon)}}

\DeclareMathOperator*{\maximize}{maximize}
\DeclareMathOperator{\Tr}{Tr}

\newcommand{\R}{\mathbb{R}}

\newcommand{\ba}{\boldsymbol{a}}
\newcommand{\bx}{\boldsymbol{x}}
\newcommand{\br}{\boldsymbol{r}}

\newcommand{\bbu}{\boldsymbol{U}}

\newcommand{\bc}{\boldsymbol{c}}
\newcommand{\bh}{\boldsymbol{h}}

\newcommand{\lrpgamma}[1]{\text{LRP}-$\gamma_{#1}$}

\begin{document}

\IEEEtitleabstractindextext{%
\begin{abstract}
Explainable AI aims to overcome the black-box nature of complex ML models like neural networks by generating explanations for their predictions. Explanations often take the form of a heatmap identifying input features (e.g.\ pixels) that are relevant to the model's decision. These explanations, however, entangle the potentially multiple factors that enter into the overall complex decision strategy. We propose to \textit{disentangle explanations} by extracting at some intermediate layer of a neural network, subspaces that capture the multiple and distinct activation patterns (e.g.\ visual concepts) that are \textit{relevant} to the prediction. To automatically extract these subspaces, we propose two new analyses, extending principles found in PCA or ICA to explanations. These novel analyses, which we call principal relevant component analysis (PRCA) and disentangled relevant subspace analysis (DRSA), maximize \textit{relevance} instead of e.g.\ variance or kurtosis. This allows for a much stronger focus of the analysis on what the ML model actually uses for predicting, ignoring activations or concepts to which the model is invariant. Our approach is general enough to work alongside common attribution techniques such as Shapley Value, Integrated Gradients, or LRP. Our proposed methods show to be practically useful and compare favorably to the state of the art as demonstrated on benchmarks and three use cases.
\end{abstract}

}

\maketitle

\section{Introduction}

Machine learning techniques, especially deep neural networks, have been successful at converting large amounts of data into complex and highly accurate predictive models. As a result, these models have been considered for a growing number of applications. Yet, their complex nonlinear structure makes their decisions opaque, and the model behaves as a black-box. In the context of sensitive and high-stakes applications, the necessity to thoroughly verify the decision strategy of these models before deployment is crucial. This important aspect has contributed to the emergence of a research field known as Explainable AI \cite{baehrens2010explain, DBLP:journals/pieee/SamekMLAM21,DBLP:journals/inffus/ArrietaRSBTBGGM20,DBLP:journals/aim/GunningA19,DBLP:series/lncs/11700,DBLP:journals/access/RoscherBDG20} that aims to make ML models and their predictions more transparent to the user.

A popular class of Explainable AI techniques, commonly referred to as `attribution', identifies for a given data point the contribution of each input feature to the prediction \cite{DBLP:journals/jmlr/StrumbeljK10,bach-plos15,DBLP:conf/icml/SundararajanTY17,DBLP:conf/nips/LundbergL17}. Attribution techniques have demonstrated usefulness in a broad range of applications. They can identify contributing features in nonlinear relations of scientific interest \cite{DBLP:series/lncs/KratzertHKHK19, EbertUphoff2020,Lundberg2020,Keyl2022,klauschen2024toward}, or enable further validation of the models at hand \cite{hae-srep20,DBLP:journals/inffus/AndersWNSML22}. However, for certain applications and data types, a simple attribution of the decision function to the input features may be of limited use. Specifically, it may fail to expose the \textit{multiple} reasons why a particular input feature contributes or which component of the decision strategy is responsible for that contribution.

These limitations have led to the advance of richer structured explanations. The development encompasses `higher-order explanations' \cite{Lundberg2020, DBLP:journals/pami/SimonRDD20, bilrp, Schnake2021} that aim to extract the contribution of input features in relation to other input features, and `hierarchical explanations' \cite{DBLP:journals/pami/ZhangWCWSZ21,DBLP:journals/corr/abs-2108-12204,Achtibat2023} where concepts (e.g.\ directions in activation space) are first extracted and then leveraged to identify joint feature-concept contributions. Proposals for hierarchical or concepts-based explanations typically construct a latent space that maximally correlates with some ground-truth annotations \cite{kim18,DBLP:journals/inffus/AndersWNSML22} or learn a latent space that maximizes some statistics of projected activations \cite{Vielhaben2022,DBLP:journals/corr/abs-2108-12204}. These approaches, however, do not guarantee a specific focus on features that are most relevant for the model to arrive at its decision; they may in some cases extract directions in activation space to which the model is mostly invariant.

\begin{figure*}[t!]

    \centering
        \includegraphics[width=\textwidth,clip,trim=0 60 0 0]{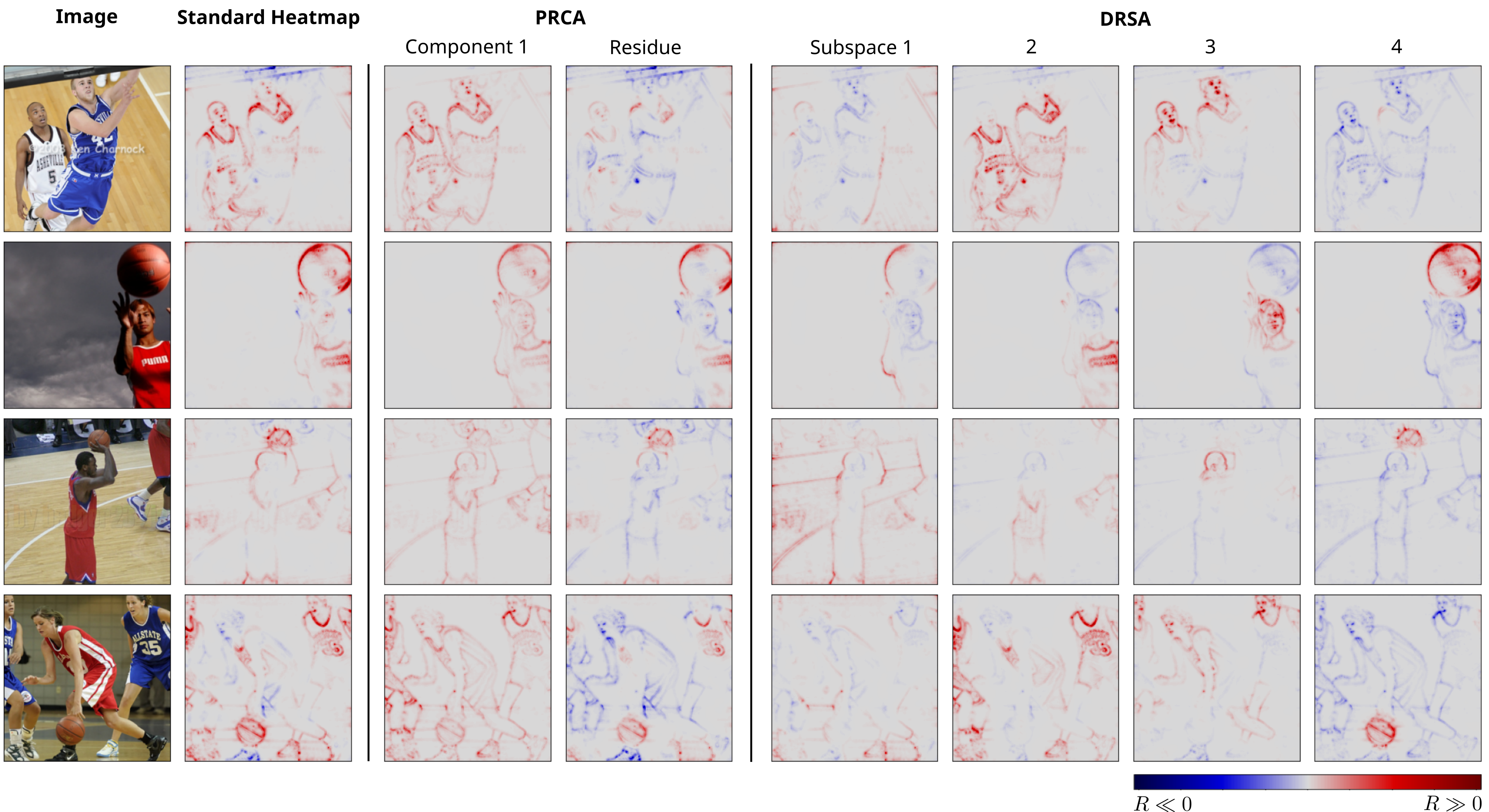}    
    \caption{Disentangled explanations produced by our proposed methods for the logit `basketball' of a pretrained VGG16 network \cite{DBLP:journals/corr/SimonyanZ14a,torchvision2016}. From left to right: the input image; a standard explanation (LRP); a decomposition onto the first PRCA component and the residue (computed at Conv4\_3); the disentangled explanation produced by DRSA (4 subspaces extracted at Conv4\_3). Red color {\color{red} $\bullet$} indicates pixels that contribute evidence for the prediction through the given component or subspace, and blue color {\color{blue} $\bullet$} indicates pixels that speak against it. The outcome of DRSA can be interpreted as decomposing the prediction strategy into four sub-strategies, here, the detection of the basketball field, the outfits, the faces, and the ball, as highlighted in red in each column.}
    \label{fig:overview}
\end{figure*}

To address the general need for more structured and focused explanations, we propose to equip Explainable AI with a new form of representation learning: extracting a collection of subspaces that \textit{disentangles} the explanation (i.e.\ separates it into multiple semantically distinct components contributing to the model's overall prediction strategy). Technically, we contribute two novel analyses: \textit{principal relevant component analysis} (PRCA) and \textit{disentangled relevant subspace analysis} (DRSA), that achieve two particular flavors of the representation learning objective. PRCA can be seen as an extension of the well-known PCA that incorporates model response into the analysis \cite{braun2008relevant}. In a similar fashion, DRSA can be seen as an extension of ICA-like subspace analysis \cite{DBLP:journals/nn/HyvarinenO00,DBLP:conf/cvpr/LeZYN11,DBLP:series/civ/HyvarinenHH09}. As a result, PRCA and DRSA inherit advantageous properties of the methods they build upon, such as simplicity and ease of optimization.

Furthermore, our contributed PRCA and DRSA methods integrate transparently into a number of popular attribution techniques, in particular, Integrated Gradients \cite{DBLP:conf/icml/SundararajanTY17}, Shapley Values \cite{shapley:book1952, DBLP:journals/jmlr/StrumbeljK10, DBLP:conf/nips/LundbergL17}, and Layer-wise Relevance Propagation \cite{bach-plos15, DBLP:journals/pr/MontavonLBSM17, DBLP:series/lncs/MontavonBLSM19, DBLP:conf/icml/AliSEMMW22}. Hence, any explanation produced by these common attribution techniques can now be disentangled by our method into several meaningful components. Moreover, our proposed methods  preserve useful properties of the underlying attribution techniques such as conservation \cite{shapley:book1952,bach-plos15} (aka.\ completeness or efficiency) and their computational/robustness profile. Fig.\ \ref{fig:overview} shows examples of disentangled explanations produced by our PRCA/DRSA approach, where we observe that the overall prediction strategy of a VGG16 network for the class  `basketball' decomposes into multiple sub-strategies including detecting the basketball field, the outfits, the faces, and the ball.

Through an extensive set of experiments on state-of-the-art models for image classification, we demonstrate qualitatively and quantitatively that our approach yields superior explanatory power compared to a number of baselines proposed by us or other approaches from the literature. In particular, we observe that our disentangled explanations capture more distinctly the multiple visual patterns used by the model to predict.

Lastly, we present three use cases for the proposed disentangled explanations:  (1) We show that disentangled explanations enable a simple and efficient interface for the user to identify and remove Clever Hans effects \cite{Lapuschkin2019} in some model of interest. (2) We demonstrate on a subset of ImageNet containing different butterfly classes how disentangled explanations can enrich our understanding of the relation between visual features and butterfly classes. (3) We use PRCA to analyze explanations that are adversarially manipulated through a perturbation of the input image (see e.g.\ \cite{DBLP:conf/nips/DombrowskiAAAMK19,dombrowski2022towards}), allowing us to disentangle the original explanation from its adversarial component.

\section{Related Work}

We discuss below the work on Explainable AI that is most closely related to our contribution, in particular, concept-based and structured explanations. For a broader overview of Explainable AI techniques, general discussions of Explainable AI and applications, we refer to reviews, e.g.\ \cite{DBLP:journals/pieee/SamekMLAM21}, \cite{DBLP:journals/inffus/ArrietaRSBTBGGM20}, \cite{DBLP:series/lncs/11700}, \cite{DBLP:journals/dsp/MontavonSM18}, \cite{DBLP:conf/icml/2020xxai}, \cite{holzinger2019causability}, \cite{DBLP:journals/entropy/LinardatosPK21}, \cite{DBLP:journals/cacm/Lipton18}.

\subsubsection*{Concept-Based Explanations}

Building on findings that deep neural networks encode useful intermediate concepts in their intermediate layers \cite{DBLP:journals/jmlr/MontavonBM11,DBLP:journals/ploscb/CadieuHYPASMD14, DBLP:conf/eccv/ZeilerF14}, a first set of related works considers the problem of explanation in terms of abstract concepts that are represented well in intermediate layers. For example, IBD \cite{DBLP:conf/eccv/ZhouSBT18} learns from ground-truth concept annotations an interpretable basis, allowing to decompose the prediction in terms of the different concepts. The TCAV method \cite{kim18} builds for each concept a linear classifier in activation space, using ground-truth annotations, and generates per-instance concept sensitivity scores using derivatives in activation space.  \cite{ghorbani19} extends the TCAV framework by using clustering algorithms to find directions in latent space, bypassing the need of having a concept dataset. Alternatively, \cite{DBLP:conf/aaai/ZhangM0ER21} finds that using non-negative factorization yields higher fidelity than using clustering approaches.  \cite{Vielhaben2022} views concepts as subspaces of the representation formed in some intermediate layer and proposes a sparse clustering algorithm to extract such subspaces. \cite{Vielhaben2023} introduces ``virtual layers'', converting the input signal to frequency space and back, in order to produce explanations in frequency domain. The NetDissect framework \cite{DBLP:journals/pami/ZhouBO019} offers a way of matching hidden neurons to concepts (obtained from the Broden dataset \cite{BauZKO017}). It enables partitioning the space of activation into multiple subspaces, each of them representing a distinct concept. Concept-based explanations have also been proposed to investigate similarity predictions \cite{DBLP:journals/tomccap/ChenLZSLC23}. We refer to \cite{poeta2023concept} for a broader overview and taxonomy of recent developments in concept-based explanations.

\subsubsection*{Structured and Higher-Order Explanations}

Another line of work aims to extract structured explanations, either joint contributions of input features or of input features and concepts. Higher-order methods \cite{Cui19, bilrp, Lundberg2020, DBLP:journals/jmlr/JanizekSL21, Schnake2021} enable the construction of these explanations and also better account for interaction effects present in ML models such as graph neural networks \cite{Schnake2021}.  \cite{DBLP:journals/corr/abs-2108-12204} builds an interpretable model, called prototype networks, that support joint explanations in terms of concepts and input features.  \cite{Achtibat2023} enables such structured explanation in a post-hoc manner, by extending the LRP framework to filter the explanation signal that passes through different activation maps representing different concepts. Another propagation-filtering approach is applied at each layer in \cite{DBLP:journals/ijcv/ChengJHWT23} in order to build a hierarchical explanation. \cite{DBLP:journals/pami/ZhangWCWSZ21} learns a surrogate graph-based model at multiple layers of a trained neural network in order to produce hierarchical explanations.  \cite{DBLP:conf/iclr/SinghMY19} proposes the context decomposition approach to extract hierarchies of input features that explain the prediction of an NLP model. The contextual decomposition approach is further extended in \cite{DBLP:conf/iclr/JinWDXR20}, in particular, addressing the question of how to explain combinations of two phrases.  \cite{DBLP:conf/acl/ChenZJ20} extracts a hierarchical explanation through the use of a second-order attribution method. In contrast to structured or higher-order explanations, other methods such as \cite{DBLP:conf/iclr/ChenZLLC024} aim to simplify a standard first-order explanation, by making it sparse.

\subsubsection*{Applications to Model Validation and Improvement}

A number of works leverage the joint usage of explanation techniques and intermediate representations for the purposes of model validation or improvement.   \cite{bykov2022dora} proposes a data-agnostic framework that uses synthetic images to investigate whether the intermediate representation of a trained model exhibits any potential Clever Hans effects. \cite{DBLP:journals/inffus/AndersWNSML22,DBLP:journals/inffus/LinhardtMM24} model Clever Hans phenomena in a trained model by an application of LRP and builds a transformation  in activation space to prune these Clever Hans effects, while \cite{DBLP:conf/iclr/KirichenkoIW23} mitigates such effects by relearning the last layer of the model using a reweighting dataset.

\subsubsection*{Representation Learning and Disentanglement}

Beyond the field of Explainable AI, a broad range of works have addressed the question of learning disentangled representations \cite{DBLP:journals/nn/HyvarinenO00,DBLP:conf/cvpr/LeZYN11,DBLP:journals/pami/BengioCV13}. Related to our focus on relevant subspaces, some of these works take label information or model response into account \cite{DBLP:journals/pr/BarshanGAJ11} or study the guarantees of concept discovery algorithms \cite{DBLP:journals/corr/abs-2206-13872}. Other works focus not on learning disentangled representations, but on evaluating them (e.g.\ \cite{DBLP:conf/iclr/EastwoodW18}, \cite{DBLP:journals/tbe/MeineckeZKM02}), including the study of how representations evolve from layer to layer in neural networks \cite{DBLP:journals/jmlr/MontavonBM11}, \cite{DBLP:conf/itw/TishbyZ15}, \cite{DBLP:journals/nn/GuoCDHST20,
DBLP:journals/ml/CaoLHWT22}. In contrast to the works above, our paper proposes techniques that specifically address the problem of disentangling explanations.

\section{Joint Pixel-Concept Explanations}

\begin{figure*}[t]
\centering
\includegraphics[width=.98\textwidth]{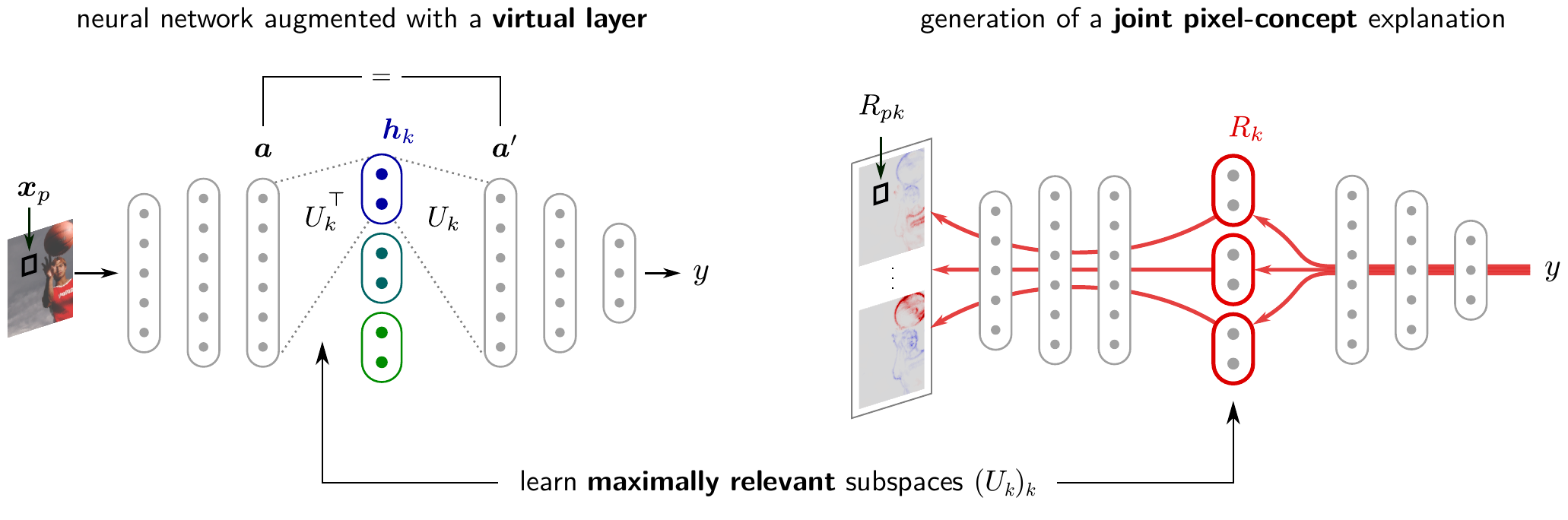}

\caption{
Overview of our proposed approach for generating disentangled explanations. \textit{Left:} The neural network to explain is augmented with a virtual layer performing an orthogonal transformation and back onto subspaces representing distinct concepts. The transformation matrices $(U_k)_k$ are optimized to extract subspaces maximizing some statistics of associated relevance scores $(R_k)_k$ (Eqs. {\eqref{eq:prca-objective}} or {\eqref{eq:drsa-objective}} with $R_k$ defined in {\eqref{eq:Rk}}). \textit{Right:} Once the virtual layer has been built, it is used to support the generation (via Eqs. \eqref{eq:twostep-1} and \eqref{eq:twostep-2}) of more informative pixel-concept explanations.}
\label{fig:neuralnet}
\end{figure*}

We place our focus on the commonly studied problem of \textit{attribution}, which asks the extent by which each input feature has contributed to a given prediction. Shapley value \cite{shapley:book1952, DBLP:journals/jmlr/StrumbeljK10, DBLP:conf/nips/LundbergL17},
Integrated Gradients \cite{DBLP:conf/icml/SundararajanTY17}, or Layer-wise Relevance Propagation (LRP) \cite{bach-plos15,DBLP:series/lncs/MontavonBLSM19} can be called `standard' attribution methods as they solve the problem of decomposing the output score into contributions of individual input pixels (or patches). An overview of these attribution techniques is provided in Supplementary Note A. One common limitation of these methods is that they do not provide information on the underlying reason (or concept) that makes a particular pixel relevant \cite{kim18,DBLP:conf/eccv/ZhouSBT18}.

More recent Explainable AI approaches, such as \cite{Achtibat2023}, \cite{gautam2023looks}, \cite{Vielhaben2023} or \cite{Fel_2023_CVPR} have thus aimed at resolving input features contributions in terms of intermediate concepts that the ML model uses to produce the overall decision strategy. In a favorable case where those concepts are readily identifiable at some intermediate layer of the neural network, more specifically, when the input-output relation can be formulated via the two-step mapping:
$$
\bx \mapsto \bh \mapsto y,
$$
where $\bx = (\bx_p)_{p=1}^{P}$ is the collection of $P$ pixels (or patches) forming the input image, where $\bh = (\bh_k)_{k=1}^K$ are the groups of neurons encoding each of the $K$ concepts, and $y$ is the output of the network (e.g.\ the evidence built by the network for the image's class). From this two-step mapping, one can generate a richer joint pixel-concept explanation via the corresponding two-step explanation process:
\begin{align}
(R_k)_{k=1}^K &= \mathcal{E}(y,\bh),\label{eq:twostep-1}\\
(R_{pk})_{p=1}^P &= \mathcal{E}(R_k,\bx). \label{eq:twostep-2}
\end{align}
The notation $\mathcal{E}(a,b)$ reads ``explain $a$ in terms of $b$'', or ``attribute $a$ onto $b$''. The score $R_k$ indicates the contribution of concept $k$ to the prediction. The score $R_{pk}$ can then be interpreted as the \textit{joint} contribution of input pixel $p$ and concept $k$ to the prediction. As an example, in Fig.\ \ref{fig:overview}, for some given input image, the score $R_k$ (with $k=2$) would measure the contribution of the `outfit' to the prediction `basketball', and $R_{pk}$ would be the contribution of a particular pixel within the outfit. The collection of all the $R_{pk}$'s forms the joint pixel-concept explanation.

In practice, when using backpropagation methods such as LRP, such a two-step explanation process takes the form of a filtering of the backward pass through specific neurons, a procedure we highlight graphically in Fig.\ \ref{fig:neuralnet} (right). The filtering approach is also found e.g.\ in \cite{DBLP:journals/dsp/MontavonSM18}, \cite{Schnake2021}, \cite{Achtibat2023}.
We provide the derivations of our two-step procedure for non-propagation attribution techniques, such as the Shapley value \cite{DBLP:journals/jmlr/StrumbeljK10,DBLP:conf/nips/LundbergL17} and Integrated Gradients \cite{DBLP:conf/icml/SundararajanTY17}, in Supplementary Note B.1.

When the attribution technique used at each step obeys a conservation principle (which is the case for all the attribution methods stated above up to some approximation), one gets the conservation equation $\sum_{pk} R_{pk} = y$. Furthermore, under some reasonable assumptions about the model and the attribution technique (cf.\ Supplementary Note B.1), one gets the stronger form of conservation $\forall p \colon \sum_{k} R_{pk} = R_p$. The joint pixel-concept explanation then becomes a decomposition of the standard pixel-wise explanation into multiple sub-explanations, and conversely, the standard pixel-wise explanation can be seen as a reduction (or coarse-graining) of the joint explanation.

\subsection{Concepts as Orthogonal Subspaces}
\label{section:linking}

We have so far assumed an intermediate representation in the model, with the variables $(\bh_k)_{k=1}^K$ encoding distinctly the multiple concepts used by the model to predict. In most deep neural networks, however, we only have a sequence of layers, each of which is a large---mainly unstructured---collection of neurons whose role or contribution to the neural network output is not always easily identifiable and possibly not well-disentangled (see e.g.\ \cite{DBLP:journals/corr/SzegedyZSBEGF13,DBLP:journals/pami/ZhouBO019}).
As some earlier works have demonstrated \cite{kim18, DBLP:conf/cvpr/FongV18,neon,Vielhaben2022}, meaningful concepts are typically recoverable (e.g.\ using a linear transformation) from the collection of activations $\ba = (a_i)_{i=1}^D$ at some well-chosen layer.

\medskip

We propose to append to such a layer of activations a \textit{virtual layer} that maps the activations to some latent representation using some orthogonal matrix $\bbu$ of size $D \times D$ and back. The matrix is structured as
\begin{align}
\bbu = \big(U_1 \big| \hdots \big| U_k \big| \hdots \big| U_K\big),
\label{eq:proj-matrix}
\end{align}
where each block $U_k$, a matrix of size $D \times d_k$, is associated with the concept $k$ and defines a projection onto a subspace of dimensions $d_k$. The virtual layer is depicted in Fig.\ \ref{fig:neuralnet} (left) and its mapping can be expressed as:
\begin{align}
\ba^\prime = 
\overunderbraces{& \br{2}{\displaystyle I} &}{
& \sum_{k=1}^K U_k & U_k^\top & \ba
}{& & \br{2}{\displaystyle  \bh_k}}~,
\label{eq:ak}
\end{align}
highlighting (1) its property of keeping the overall decision function unchanged due to the orthogonality property, and (2) the extraction of the variable $\bh_k$ which is required for producing the joint pixel-concept explanation according to Eqs.\ \eqref{eq:twostep-1} and \eqref{eq:twostep-2}.

\subsection{Expressing Concept Relevance}
\label{section:concept-relevance}

There are many ways to learn matrices $U_k$'s, for example, using PCA or other unsupervised analysis on a set of activation vectors $\ba$'s, similar to \cite{Vielhaben2022}. However, we aim in this work to learn subspaces that are specifically relevant to the decision function, i.e.\ with high relevance scores $R_k$'s. To achieve this, we will first need to express $R_k$ (the outcome of Eq.\ \eqref{eq:twostep-2}) in terms of the transformation matrix $U_k$. We give the demonstration for the LRP \cite{bach-plos15} attribution technique.

Let us recall the definition of the virtual layer in Eq.\ \eqref{eq:ak}, and observe that each reconstructed activation $a_j^\prime$ can be expressed in terms of concepts $\bh_k$ in the layer below as:
$$
a_j^\prime = \sum_{k=1}^K \bh_k^\top (U_k)_j,
$$
with $(U_k)_j$ the $j$th row of the matrix $U_k$. Assume we have propagated the neural network output using LRP down to the layer of the reconstructed activations and obtained for each $a_j^\prime$ a relevance score $R_j$\footnote{For our method to be applicable, one further requires that any activation $a_j^\prime=0$ has relevance $R_j = 0$. This property is satisfied when applying standard LRP rules (cf.\ Supplementary Note A).}. The LRP-0 rule  \cite{bach-plos15, DBLP:series/lncs/MontavonBLSM19} lets us propagate these scores to the layer below representing concepts:
\begin{align}
R_k &= \sum_{j} \frac{\bh_k^\top (U_k)_j }{a_j^\prime}  R_j,
\label{eq:lrp-rule}
\end{align}
where $\sum_j$ runs over all activated neurons $j$. After some rearranging, the same $R_k$ can be restated as:
\begin{align}
R_k = \big(U_k^\top \ba\big)^\top \big(U_k^\top \bc\big),
\label{eq:Rk}
\end{align}
where $\bc$ is a $D$-dimensional vector whose elements are given by $c_j = R_j / a_j^\prime$ for all activated neurons $j$ and $c_j=0$ otherwise.  The vector $\bc$ can be interpreted as the model response to a local variation of the activations, and we refer to it in the following as the `context vector'.

In Supplementary Note B, we show that other attribution methods such as Gradient$\,\times\,$Input \cite{DBLP:journals/corr/ShrikumarGSK16} and Integrated Gradients \cite{DBLP:conf/icml/SundararajanTY17} (with 
zero reference value) also produce relevance scores of the form of Eq.\ \eqref{eq:Rk}, and we provide an expression for their respective context vector $\bc$. Because methods based on the Shapley value \cite{DBLP:journals/jmlr/StrumbeljK10,DBLP:conf/nips/LundbergL17} do not yield the form of Eq.\ \eqref{eq:Rk}, one requires an approximation of the latter. In particular, our solution consists of first computing Shapley values w.r.t.\ the activations $a_j$ (or groups of it), and then performing the propagation step onto concepts using the LRP rule of Eq.\ \eqref{eq:lrp-rule}.

\section{Learning Relevant Subspaces}
\label{section:learning}

Having expressed concept relevance $R_k$'s in terms of known quantities, specifically for each data point, (i) activations vectors $\ba$ that we can collect and (ii) context vectors $\bc$ representing the model response and that we can compute (e.g.\ using LRP), we can formulate various concept relevance maximization objectives over the transformation matrices $U_k$'s.

\subsection{Principal Relevant Component Analysis (PRCA)}
\label{section:prca}

The first objective we propose is to extract a subspace that is maximally relevant to the model prediction. Consider our virtual layer has the simple block structure
\begin{align}
\bbu = (U\,|\,\widetilde{U}),
\label{eq:block-prca}
\end{align}
where $U \in \R^{D \times d}$ defines a projection onto a subspace of fixed dimensions $d$ and where $\widetilde{U} \in \R^{D \times (D-d)}$ projects to the orthogonal complement. We ask ``what matrix $U$ yields a subspace that is maximally relevant for the prediction''. Starting from the expression of relevance in Eq.\ \eqref{eq:Rk}, we can formulate the search for such a maximally relevant subspace via the optimization problem:
\begin{align}
\maximize_{U}~\mathbb{E}[(U^\top \ba)^\top (U^\top \bc\big)]
\label{eq:prca-objective}\\
\text{subject to:}~~U^\top U = I_{d}, \nonumber
\end{align}
where the expectation denotes an average over some dataset (e.g.\ images of a given class and the associated activation and context vectors). Using linear algebra identities, the same optimization problem can be reformulated as:
$\max_{U} \Tr(U^\top \mathbb{E}[ \ba \bc^\top] U)$ subject to $U^\top U = I_{d}$, where a cross-covariance term between the activations $\ba$ and the context vector $\bc$ appears. The solution to this optimization problem is the eigenvectors associated to the $d$ largest eigenvalues of the symmetrized cross-covariance matrix $$\boldsymbol{\Sigma} = \mathbb{E}[\ba \bc^\top + \bc \ba^\top]$$ (cf.\ Supplementary Note C.2 for the derivation). In practice, the orthogonal matrix of Eq.\ \eqref{eq:block-prca} can therefore be computed using a common eigenvalue solver:
\begin{align*}
\bbu &= \text{eigvecs}(\boldsymbol{\Sigma}),
\end{align*}
and, assuming eigenvectors are sorted by decreasing eigenvalues, we recover the blocks of that matrix as
\begin{align*}
    U &= \bbu_{:,1\dots d},\\
    \widetilde{U} &= \bbu_{:,d+1\dots D}.
\end{align*}
The proposed PRCA differs from standard PCA, by also taking into account---via the context vector $\bc$---how the output of the network responds to the activations $\ba$. Thus, PRCA is able to ignore high variance directions in the data when the model is invariant or responds negatively to these variations. \mbox{Fig.\ \ref{fig:toy-prca}} provides an illustration of this effect on two-dimensional data, showing that the PRCA subspace aligns more closely to the model response than PCA.

\begin{figure}
    \centering
    \includegraphics[clip,trim=0 0 0 1cm,width=\linewidth]{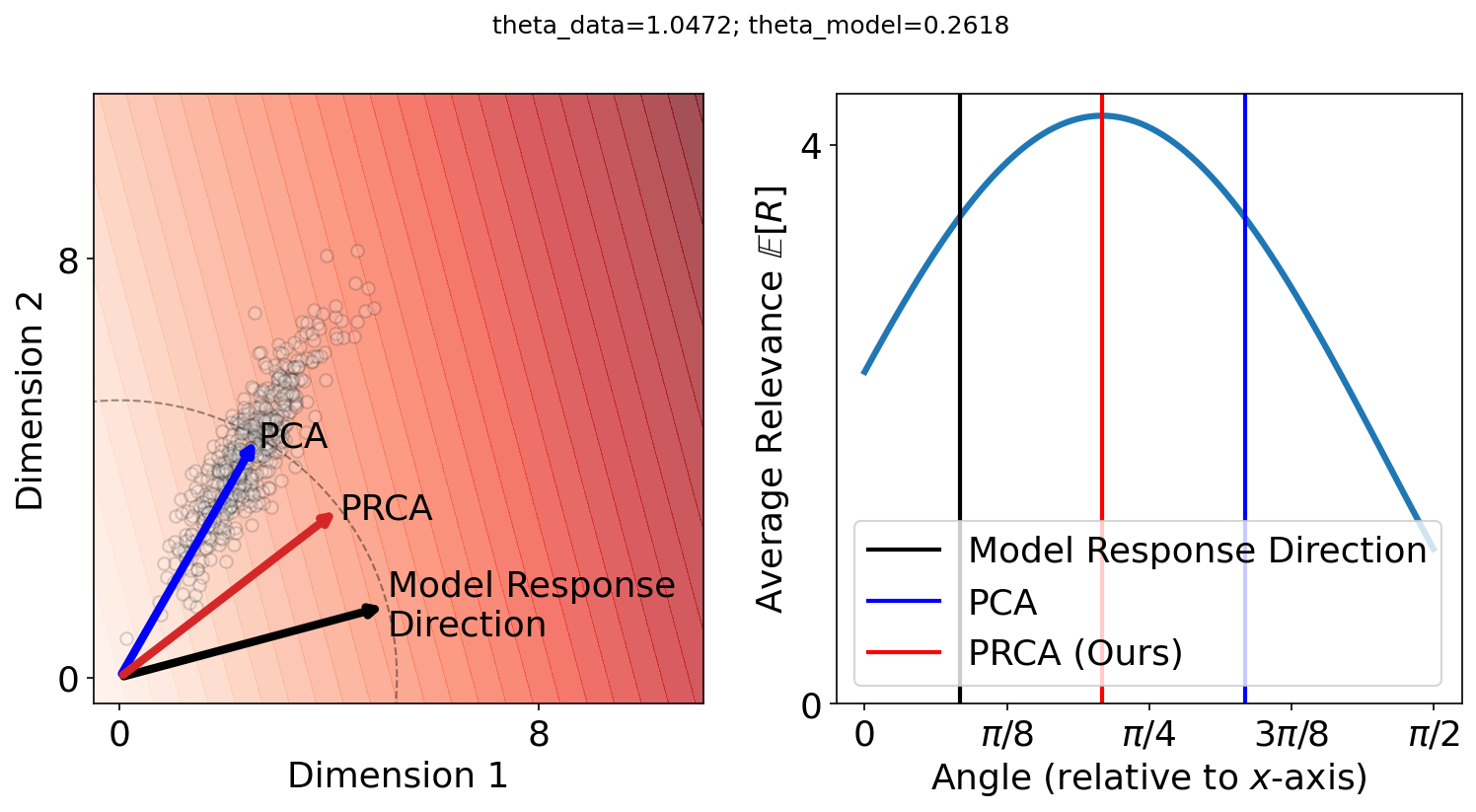}
    \caption{Left: Comparison between the one-dimensional subspace extracted by (uncentered) PCA and PRCA in a synthetic two-dimensional activation space. Right: Average relevance as a function of the angle of the subspace. The vertical lines correspond to the angles of the vectors in the left panel. By design, PRCA maximizes average relevance.}
    \label{fig:toy-prca}
\end{figure}

Note that several related approaches to refocus PCA on task-specific features have been proposed, although in a different context than Explainable AI. This includes `directed PCA' \cite{imhoff1982use,Fraser1987}, where a subset of task-related features are selected before running PCA. It also includes `supervised PCA' \cite{DBLP:journals/pr/BarshanGAJ11} which formulates a trace maximization problem involving both input and labels, and methods based on partial least squares \cite{DBLP:books/sp/HastieFT01}.

\subsection{Disentangled Relevant Subspace Analysis (DRSA)}
\label{section:irca}

\begin{figure}[t!]
    \centering
    \includegraphics[clip,trim=0.25cm 13.7cm 13cm 0.8cm,width=0.24\textwidth]{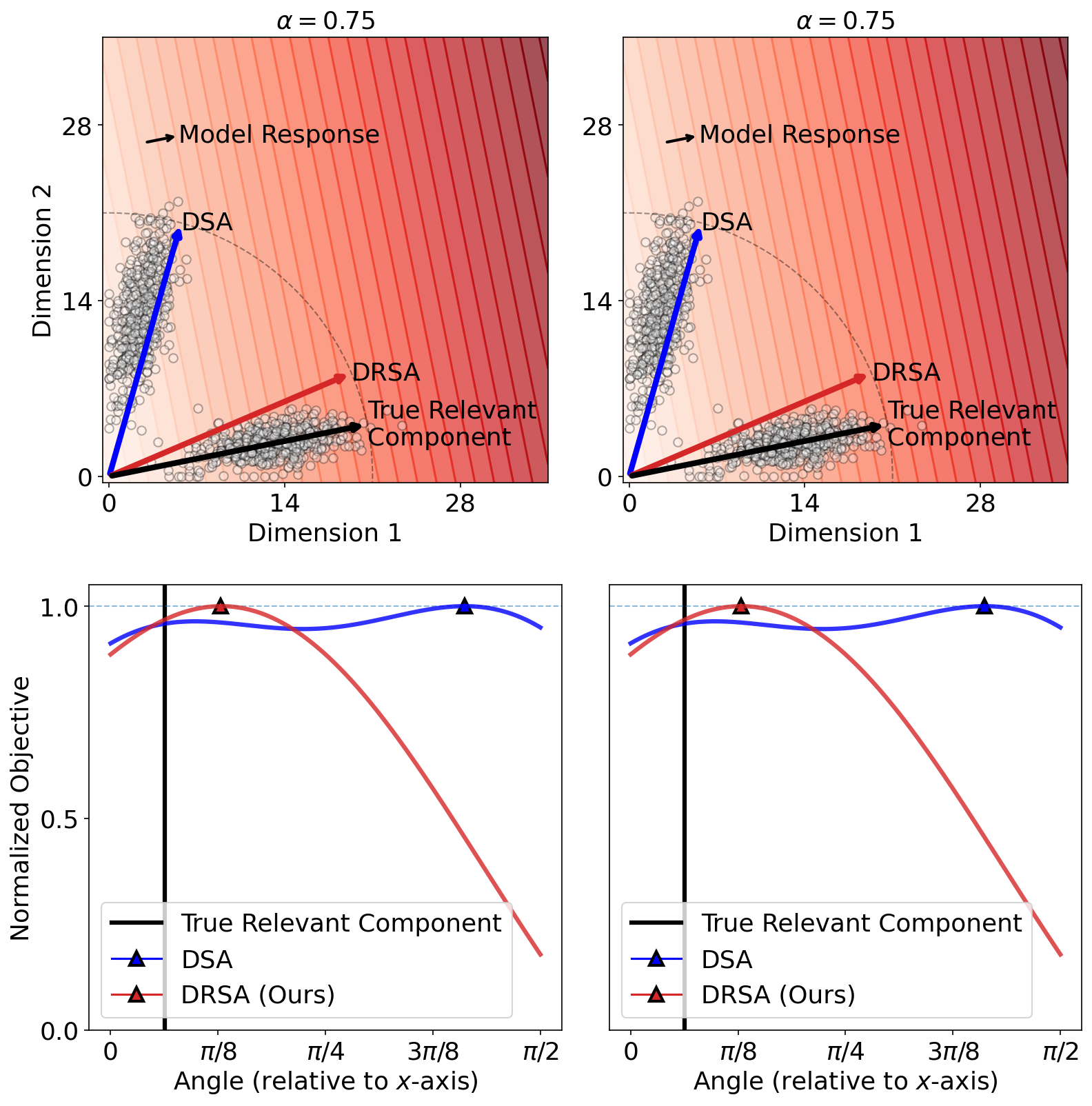}
     \includegraphics[clip,trim=0.25cm 0.25cm 13cm 14.2cm,width=0.24\textwidth]{figures/toy-example/toy-example-drsa}
    \caption{Left: One-dimensional subspaces extracted by DRSA and DSA (a reduction of DRSA where $\ba$ is used in place of $\bc$) in a synthetic two-dimensional activation space. Unlike DSA, our proposed approach DRSA takes the model response (shown as a contour plot) into account, thereby being able to focus on the relevant component. Right: Normalized objective of DSA and DRSA at different angles.}
    \label{fig:toy-drsa}
\end{figure}

We now extend the approach above for the purpose of separating an explanation into multiple components representing the different concepts contributing to the overall decision strategy. Specifically, we partition the activation space into multiple subspaces, with subspaces optimized in a way that they maximize some higher-order statistics of relevance scores.

Let the orthogonal matrix $\bbu$ associated to our virtual layer be structured in the general form of Eq.\ \eqref{eq:proj-matrix}. Let $n \in \mathcal{D}$ be indices for different data points and
\begin{align}
R_{k,n}^+(\bbu) = \max(0,(U_k^\top \ba_n\big)^\top \big(U_k^\top \bc_n\big)\big)
\end{align}
be the positive part of the relevance associated to subspace $k$ according to Eq.\ \eqref{eq:Rk} for a given data point $n$. The rectification operation allows us to focus our subsequent analysis on extracting salient positive contributions. In particular, we define our disentanglement-inducing objective as:
\begin{align}
\maximize_{\bbu}~~\mathbb{M}^{q}_{k \in \{1,\dots,K\}} \big[ \mathbb{M}^2_{n \in \mathcal{D}} \big[ R_{k,n}^+(\bbu) \big]\big]
\label{eq:drsa-objective}\\
\text{subject to:}~~\bbu^\top \bbu = I_D, \nonumber
\end{align}
where we use in practice $q=0.5$. The operator $\mathbb{M}^p$ denotes a generalized F-mean with function $F(t) = t^p$. The special cases $\mathbb{M}^{0.5}$ and $\mathbb{M}^{2}$ can be interpreted as soft min- and max-poolings respectively.  The soft max-pooling over data points serves to encourage subspaces to align with instances $n$'s with particularly high relevance scores. These instances can be interpreted as prototypes for each identified component of the decision strategy. The soft min-pooling over subspaces, on the other hand, serves to favor solutions that balance the total relevance attributed to the different subspaces. Other nested pooling structures are found in the Independent Subspace Analysis algorithms of \cite{DBLP:series/civ/HyvarinenHH09,DBLP:conf/cvpr/LeZYN11}. These nested structures have also been studied in more depth in \cite{DBLP:journals/jmlr/SinzB10}. The behavior of DRSA on a synthetic two-dimensional activation space and a single one-dimensional subspace is illustrated in Fig.\ \ref{fig:toy-drsa}.

While the optimization problem above is non-convex and does not have a closed form solution, we can proceed iteratively, starting from a random orthogonal matrix\footnote{We can sample such an orthogonal matrix from the orthogonal group using e.g.\ the `ortho\_group' function from SciPy \cite{2020SciPy-NMeth}.}, and similarly to \cite{DBLP:conf/cvpr/LeZYN11}, alternating gradient ascent and orthogonalization steps ($\bbu \gets \bbu(\bbu^\top \bbu)^{-1/2}$ \cite{hyvarinen-ica-ch6}).

\subsection{Theoretical Properties}

The proposed relevant subspace analyses have a number of desirable theoretical properties:
\begin{proposition}
Let $\bbu = (U_k)_k$ be an orthogonal matrix formed by $U_k$'s.  Using the formulation of relevance $R_k = (U_k^\top \ba)^\top(U_k^\top \bc)$ with $\bc$ such that $R_j = a_j^\prime c_j$, we have the conservation property $\sum_k R_k = \sum_j R_j$. Furthermore, when $\bc = \xi \ba$ with $\xi \geq 0$, then we necessarily have $R_k\geq 0$.
\label{proposition:conservation}
\end{proposition}
\noindent (A proof can be found in Supplementary Note C.1.)
These properties follow from the orthogonality constraint on the matrices $(U_k)_k$. The first property (conservation) ensures that the two-step explanation produced by our method retains the conservation properties of the original explanation technique it is based on. The second property (positivity) ensures that an absence of contradiction in the decision function (e.g.\ a perfect alignment between activations and model response) results in a similar absence of contradiction w.r.t.\ concepts.

The following result links the proposed PRCA and DRSA algorithms to well-known analyses such as PCA and ICA.
\begin{proposition}
When the context vector $\bc$ is equivalent to the activation vector $\ba$, the PRCA analysis reduces to uncentered PCA. Furthermore, if we assumed whitened activations, i.e., $\mathbb{E}[\ba] = \boldsymbol{0}$ and $\mathbb{E}[\ba\ba^\top] = I$, and each matrix $U_k$ projecting to a subspace of dimension $1$, then the DRSA analysis with parameter $q=2$ reduces to ICA with kurtosis as a measure of subspace independence.
\label{proposition:reduction}
\end{proposition}
\noindent (A proof can be found in Supplementary Note C.1.)
In other words, with some restrictions on the parameters, our proposed algorithms reduce to PCA and ICA when the model response is perfectly aligned with the activations. Unlike PCA and ICA, our analyses are able to extract subspaces that are relevant to the prediction even when the model response does not align well with the activations. 

\medskip

When considering the level of access to the model our method requires, it can be characterized as a white-box method. Specifically, our method requires access to at least one intermediate layer to perform PRCA and DRSA. When using our method together with attribution methods such as LRP, access to all layers of the model is required in order to implement the LRP propagation rules.

\section{Quantitative Evaluation}
\label{section:evaluation}

To evaluate the performance of our PRCA and DRSA methods at extracting relevant subspaces and producing disentangled explanations respectively, we perform experiments on the ImageNet \cite{imagenet_cvpr09} and Places365 \cite{zhou2017places} datasets. For ImageNet, we consider a subset of 50 classes\footnote{For ease of reproducibility and maximizing class coverage, we choose ImageNet classes with indices $\{0, 20, 40, \dots, 980\}$.} and three publicly available  pre-trained models. These models are two \mbox{VGG16} \cite{DBLP:journals/corr/SimonyanZ14a} models---which are from the TorchVision (TV) \cite{torchvision2016} and NetDissect (ND) \cite{DBLP:journals/pami/ZhouBO019} repositories, denoted by VGG16-TV and VGG16-ND\footnote{We remark that VGG16-ND is our PyTorch version of the  model (originally in Caffe \cite{jia2014caffe}'s format) on which the relation between concepts and feature maps has been analyzed in \cite{DBLP:journals/pami/ZhouBO019}, allowing for a more direct comparison between the previous work and our DRSA approach. The original model is available at \url{http://netdissect.csail.mit.edu/dissect/vgg16_imagenet/}.} respectively---and the NFNet-F0 model (the smallest variant of a more recent architecture called Normalizer-Free Networks  (NFNets) \cite{pmlr-v139-brock21a}) from PyTorch Image Models \cite{rw2019timm}. For Places365, we consider a subset of seven classes\footnote{These Places365 classes are similar to the ones visualized in Ref. \cite{DBLP:conf/eccv/ZhouSBT18}'s Fig.\ 4.} and the ResNet18 \cite{He_2016_CVPR} model provided by Ref. \cite{DBLP:conf/eccv/ZhouSBT18}.

We evaluate our proposed methods with Shapley Value Sampling---an approximation of the classic Shapley value---and LRP; these two attribution techniques are chosen based on patch-flipping experiments \cite{samek-tnnls17} (see Supplementary Note D). We use the implementation of Shapley Value Sampling from Captum \cite{kokhlikyan2020captum}. Our LRP implementation for VGG16 is based on \mbox{\lrpgamma{}} used in \cite{bilrp}. For NFNets,  we contribute a novel LRP implementation (see Supplementary Note K).
For ResNet18, we use the LRP implementation from Zennit \cite{anders2021software}. We provide the details of these attribution methods' hyperparameters in Supplementary Note D.

In the following, we focus on evaluating the proposed approaches using activations from VGG16 at Conv4\_3 (after ReLU), NFNet-F0 at Stage 2, and ResNet18 at Layer 4\footnote{We adopt the layer-name conventions of  VGG16 from \cite{DBLP:journals/pami/ZhouBO019}, of NFNet-F0 from \cite{pmlr-v139-brock21a}, and of ResNet18 from \cite{torchvision2016}.}. We refer to ablation studies on the choice of layers and the number of subspaces in Supplementary Note F.

To extract subspaces, we randomly choose 500 training images of each class and take their feature map activations at the layer of interest. For each image, we randomly pick 20 spatial locations in the feature maps. The procedure results in a collection of $10000$ activation vectors for each class. We also collect corresponding context vectors (w.r.t. the target class) associated to Shapley Value Sampling and LRP attribution methods. We summarize these details in Supplementary Note E. All our evaluations are performed on a validation set disjoint from the data used for training the networks and optimizing the PRCA/DRSA subspaces.

\subsection{Evaluation of PRCA}
\label{sec:experiment-1}

\begin{table*}[t!]
\centering
    \caption{Patch-Flipping evaluation of PRCA for subspace size $d=1$. Evaluation is performed over different combinations of models, datasets and underlying attribution techniques (columns). We report for each method in our evaluation (rows) the AUPC score. The AUPC scores are computed by averaging over instances within each class and then averaging over classes. The best  method for each setting is shown in bold, and the second best with underline. The proposed PRCA method performs best in all settings. ($\dagger$) average from three seeds.}
    \label{table:aupc-single-subspace}
    \begin{tabular}{lcccccc}
    \toprule

     & \multicolumn{5}{c}{ImageNet} & \multicolumn{1}{c}{Places365} \\
 
    \cmidrule(lr){2-6} \cmidrule(l){7-7}
    
     & \rotatebox{90}{\parbox{1.5cm}{VGG16-TV\\+ LRP}}
     & \rotatebox{90}{\parbox{1.5cm}{VGG16-ND\\+ LRP}}
     & \rotatebox{90}{\parbox{1.5cm}{NFNet-F0\\+ LRP}}
     & \rotatebox{90}{\parbox{1.5cm}{VGG16-TV\\+ Shapley}}
     & \rotatebox{90}{\parbox{1.5cm}{VGG16-ND\\+ Shapley}}
     & \rotatebox{90}{\parbox{1.5cm}{ResNet18\\+ LRP}}\\
     
     \midrule

%% Version: 2023-12-20 22:02:02.061993
%% function: make_table_single_subspace_pf_aucs
%% artifact-dir=../artifacts/2023-12-revision-v1.20.0/raw-main-experiment
%% pixel-flipping: patch_size=16

% > identity-ns1-ss[SS]
              \textit{No subspace projection ($U=I_D$)} &   \textit{5.97} &   \textit{4.98} &   \textit{3.75} &   \textit{5.30} &   \textit{4.98} &   \textit{3.60}   \\
\midrule

% > random*-ns1-ss1
          Random subspace$^\dagger$ &   9.28 &   8.22 &   5.46 &   8.82 &   8.06 &   4.81   \\

% > max-rel-ns1-ss1
         Most relevant feature maps \cite{Achtibat2023} &   \underline{7.67} &   6.81 &   \underline{4.55} &   7.30 &   6.93 &   3.98   \\

% > pca-ns1-ss1
                      PCA &   7.92 &   \underline{6.45} &   6.64 &  \underline{6.07} &   \underline{5.85} &   \underline{3.89}   \\

% > prca-ns1-ss1
              PRCA (Ours) &   \textbf{5.36} &   \textbf{4.91} &   \textbf{3.84} &   \textbf{5.75} &   \textbf{5.56} &   \textbf{3.80}   \\[1mm]

         % largest stderr &   0.30 &   0.28 &   0.10 &   0.29 &   0.27 &   0.42

%% end table %%
 \textit{Error bars (max)} & $\pm$ 0.30 & $\pm$ 0.28 & $\pm$ 0.10 & $\pm$ 0.29 & $\pm$ 0.27 & $\pm$ 0.42 \\
    \bottomrule
   \end{tabular}
\end{table*}

We test the ability of our PRCA method to extract a low-dimensional subspace of the activations, that retains input features used by the model to build its prediction. The extraction of what is \textit{relevant} (vs. \textit{irrelevant}) in an ML model has recently found application in the context of model compression (e.g.\ \cite{DBLP:journals/pr/YeomSLBWMS21}). We first recall from Section \ref{section:concept-relevance} that any subspace of the activations, and the matrix $U$ of size $D \times d$ that projects onto it, embodies an amount of relevance expressible as:
$$
R = (U^\top \ba)^\top(U^\top \bc).
$$
The relevance can then be attributed to the input space, using LRP or Shapley instantiations of $\mathcal{E}(R,\bx)$. We quantify how closely (in spite of the dimensionality bottleneck) the produced explanation describes the neural network prediction strategy using pixel-flipping \cite{bach-plos15, samek-tnnls17}, a common evaluation scheme, sometimes also referred to as deletion/insertion experiments \cite{DBLP:conf/bmvc/PetsiukDS18}.

Pixel-flipping (in our case, `patch-flipping'), proceeds by removing patches from the input image, from most to least relevant, according to the explanation\footnote{We opt for the removal-based rather than the insertion-based variant of patch-flipping because it makes the task of inpainting missing patches easier and thereby reduces evaluation bias. The replacement values for removed patches are set according to the TELEA \cite{telea2004image} algorithm---a neighborhood-based inpainter---from OpenCV \cite{opencv_library}.}. As patches are being removed iteratively, we keep track network output and then compute the ``area under the patch-flipping curve'' (AUPC) \cite{samek-tnnls17}
\begin{align}
    \text{AUPC} = \mathbb{E} \Big[ \sum_{\tau=1}^T w(\tau) \bigg( \frac{f(\bx^{(\tau-1)})  - f(\bx^{(\tau)}) }{2} \bigg) \Big],
     \label{eq:aupc-single-subspace}
 \end{align}
where $\bx^{(\tau)}$ denotes the image after $\tau$ removal steps, $T$ is the number of steps until all patches have been removed from the image, $\mathbb{E}$ denotes an average over images of class $t$ in the validation set, and $f$ is the neural network output for class $t$ to which we have applied the rectification function to focus on positive evidence. The weighting function $w(\tau) \in (0, 1)$ is the difference between the percentage of patches flipped at the $\tau$-th and ($\tau$-1)-th steps. To make experiments executable in a reasonable time, we measure relevance over patches of size $16 \times 16$, and we flip $\tau^2$ such patches at each step. The lower the AUPC score, the better the explanation, and the better the subspace $U$.
 
To the best of our knowledge, there are no existing baselines from the literature that are designed to extract a subspace of the activations that can preserve the decision strategy of a given class\footnote{Ref.\ \cite{DBLP:conf/nips/YehKALPR20} studies a related problem: completeness in latent space, although its objective is to verify whether extracted concept vectors can restitute the full accuracy of the model. In contrast, our objective is to extract a subspace that maximally expresses the predicted evidence for a certain class.}. Hence, for comparison, we consider several baselines: 1) a random subspace\footnote{We take the first $d$ columns of a random orthogonal matrix.}; 2) the subspace of the first $d$ eigenvectors of the standard (uncentered) PCA on the activations; and 3) the subspace corresponding to the $d$ most relevant feature maps (Max-Rel) \cite{Achtibat2023}.

We can view the choice of baselines as a special ablation study of PRCA. Specifically, PCA corresponds to PRCA with the context vector $\bc$ representing model response set to the activation vector $\ba$ (Proposition \ref{proposition:reduction}); Max-Rel is a reduction of PRCA where the basis of the subspaces are canonical; and the random approach can be thought of as an `untrained' PRCA.

Results are given in Table \ref{table:aupc-single-subspace} for subspace size $d=1$. We observe that, PRCA strongly surpasses the baseline methods across configurations. The observation indicates that PRCA can identify a  subspace of the activations that is relevant to the prediction.

\begin{figure}[t!]
    \centering
    \includegraphics[width=0.8\linewidth,clip,trim={0 0 0 1.2cm}]{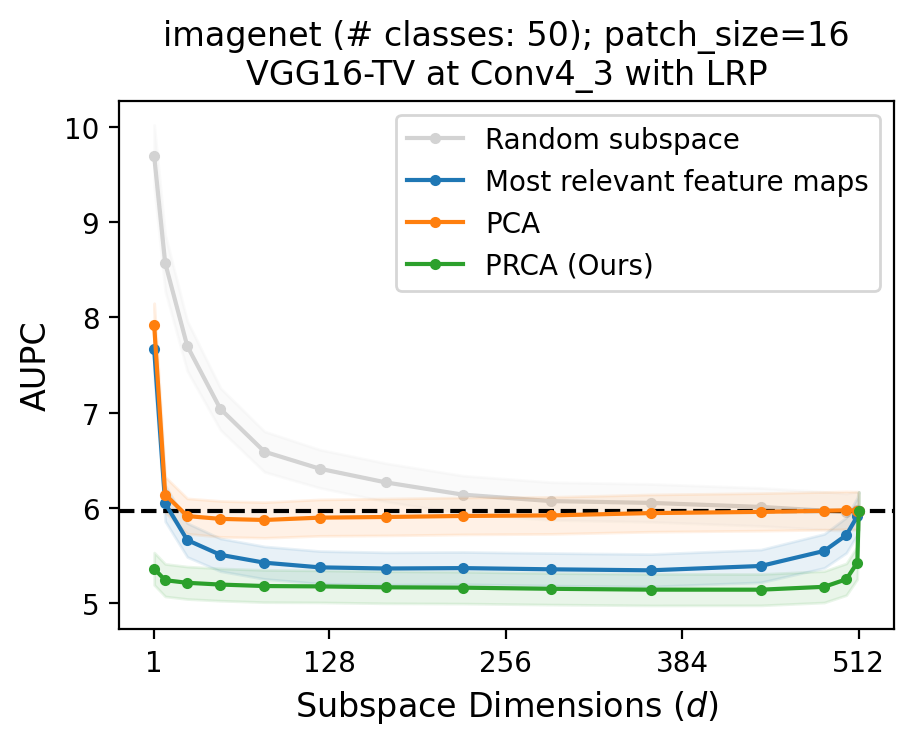}
    \caption{AUPC scores of different subspace methods when varying the subspace dimensionality (the variable $d$); lower is better.  The analysis is performed on VGG16-TV with LRP (same as column 1 in Table \ref{table:aupc-single-subspace}). Each curve is an average over the means of these 50 classes, and shaded regions represent one standard error (over classes). The horizontal dashed line represents the AUPC of no subspace projection.}
    \label{fig:single-subspace-increasing-dimensions}
\end{figure}

Next, we investigate the influence of the subspace dimensions $d$ on the quality of the subspace. We analyze the AUPC score as a function of $d$.  We perform the experiment on  the VGG16-TV model with LRP. \mbox{Fig.\ \ref{fig:single-subspace-increasing-dimensions}} shows that PRCA has the lowest AUPC scores across different values of $d$. The result supports the conclusion that the top few principal components of PRCA accurately  capture the evidence the neural network builds in favor of the image's class. \mbox{The fact} that PRCA (and also Max-Rel) performs better than no subspace projection (i.e.\ retaining the whole activation space) suggests that there exists some amount of inhibitory signal in activations that conceals mildly relevant features in the original heatmaps. By construction, the first few PRCA components are maximally relevant directions, and projecting activations onto them decreases the amount of such inhibitory signal. The showcase we present in Section \ref{sec:showcase-3}, where we use PRCA to robustify an explanation under adversarial manipulations, corroborates the above interpretation.

\subsection{Evaluation of DRSA}
\label{sec:eval-drsa}

The second question we have considered in this paper (and for which we have proposed the DRSA analysis) is whether the explanation can be disentangled into multiple components that are distinctly relevant for the prediction. Specifically, we have set the problem of disentanglement as partitioning the space of activations into $K$ subspaces, defined by their respective transformation matrices $(U_k)_k$. From these $K$ subspaces, one can retrieve each component of the explanation as $\mathcal E(R_k, \bx)$. All methods in our benchmark yield such a decomposition onto a fixed number of $K$ components (they only differ in the choice of the matrices $U_k$'s).

For evaluation purposes, we propose an extension of the patch-flipping procedure used in Section \ref{sec:experiment-1}. The extended procedure allows us to quantify the level of disentanglement, specifically, verifying that the multiple explanation components highlight distinct (spatially non-overlapping) visual elements contributing to the neural network's prediction. Our extension consists of running multiple instances of patch-flipping in parallel (one per component $k$) and aggregating patch removals coming from each component. More specifically, denoting by $\boldsymbol{M}_k^\tau$ an indicator vector of patches removed based on the $k$th component after $\tau$ steps, we define the overall set of patches to be removed after these steps to be $\boldsymbol{M}^\tau = \cup_{k=1}^K \boldsymbol{M}_k^\tau$, where the union operator applies element-wise. An illustration of the modified patch-flipping procedure is given in Fig.\ \ref{fig:modpflip}. Similar to Section \ref{section:prca}, as the patch-flipping proceeds, we keep track of the model output and compute the AUPC score. Our extended patch-flipping procedure shares similarity with the evaluation procedures in \cite{ghorbani19,Vielhaben2022} as the latter also remove features based on the distinct explanation components (or concepts).

\begin{figure}[b!]
    \centering
    \includegraphics[width=.75\linewidth]{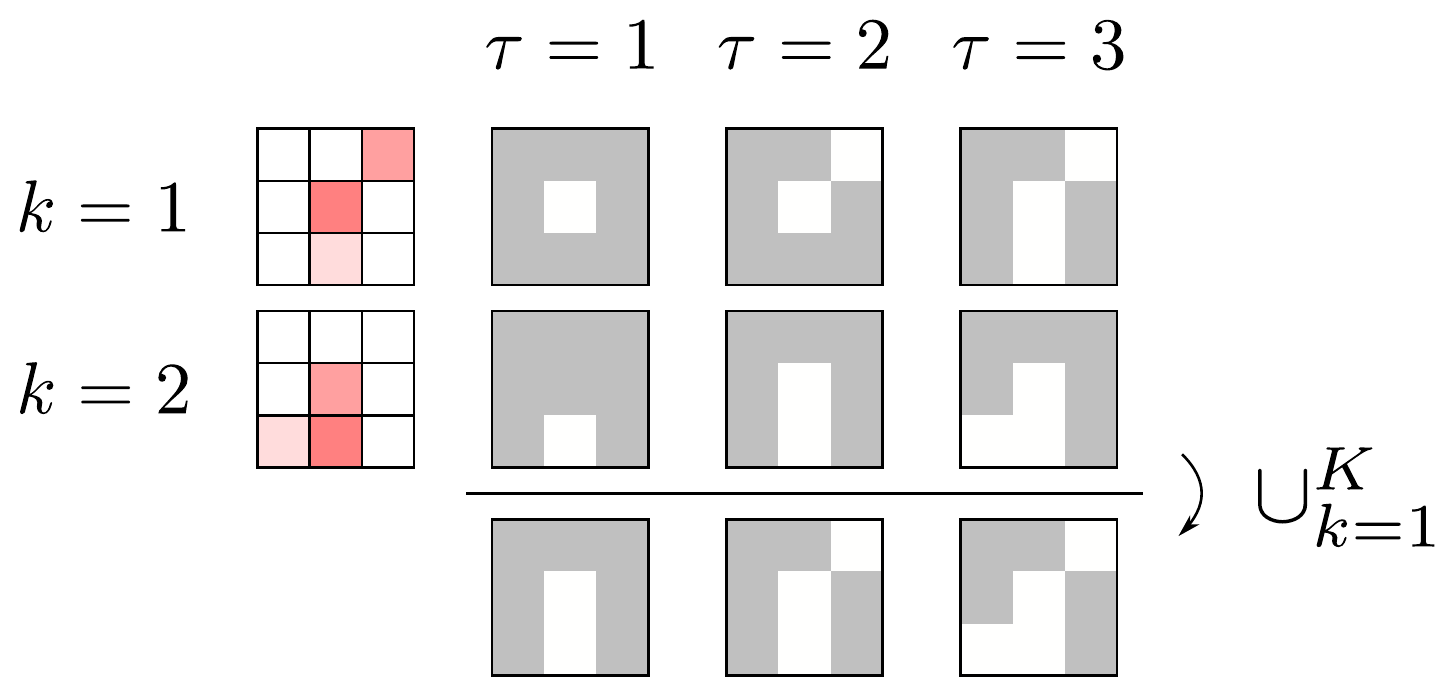}
    \caption{Illustration of our modified patch-flipping procedure to evaluate explanation disentanglement. The first column shows a two-concept explanation ($k=1,2$) with red color intensity indicating patch relevance. The next three columns show the indicator vectors for the first three steps ($\tau=1,2,3$) of the modified patch-flipping procedure, with white color indicating removed patches.}
    \label{fig:modpflip}
\end{figure}

\begin{table*}[]
    \centering
    
    \caption{Patch-Flipping evaluation of DRSA for $K=4$ subspaces. Evaluation is performed over the same dataset, model and attribution settings (columns) as in Table \ref{table:aupc-single-subspace}.  We report for each method in our evaluation (rows) the AUPC score. Like in Table \ref{table:aupc-single-subspace}, the AUPC scores are first averaged over instances and then over classes. The best method is shown in bold, and the second best with underline. Entries with `$\times$' are not applicable.  ($\dagger$) average from three seeds.}
    \label{table:aupc-multile-subspaces} 
    \begin{tabular}{lcccccc}
    \toprule
     & \multicolumn{5}{c}{ImageNet} & \multicolumn{1}{c}{Places365} \\
    \cmidrule(lr){2-6} \cmidrule(l){7-7}
    
     & \rotatebox{90}{\parbox{1.5cm}{VGG16-TV\\+ LRP}}
     & \rotatebox{90}{\parbox{1.5cm}{VGG16-ND\\+ LRP}}
     & \rotatebox{90}{\parbox{1.5cm}{NFNet-F0\\+ LRP}}
     & \rotatebox{90}{\parbox{1.5cm}{VGG16-TV\\+ Shapley}}
     & \rotatebox{90}{\parbox{1.5cm}{VGG16-ND\\+ Shapley}}
     & \rotatebox{90}{\parbox{1.5cm}{ResNet18\\+ LRP}}\\
     
     \midrule

%% Version: 2023-12-22 09:38:46.947095
%% function: make_table_multiple_subspaces_aupc
%% artifact-dir=../artifacts/2023-12-revision-v1.20.0/raw-main-experiment
%% metric=basis_auc
%% pixel-flipping: patch_size=16

% > identity-ns1-ss[SS]
              No subspace partitioning $(K=1)$ &   \textit{5.97} &   \textit{4.98} &   \textit{3.75} &   \textit{5.54} &   \textit{5.20} &   \textit{3.60}   \\
            \midrule

% > random*-ns4-ss[SS]
          Random subspaces$^\dagger$ &   4.07 &   3.47 &   3.39 &   2.96 &   2.78 &   3.23   \\

% > netdissect
               NetDissect &   3.36 &   3.03 &   $\times$  &   3.30 &   3.03 &   $\times$   \\

% > ibd
                      IBD &   $\times$  &   $\times$  &   $\times$  &   $\times$  &   $\times$  &   2.57  \\

% > learnt--dsa-ns4-ss[SS]-sm2-seed1
                      DSA &   \underline{3.20} &   \underline{2.80} &   \textbf{1.96} &   \underline{2.76} &   \underline{2.64} &   \underline{2.40}   \\

% > learnt--drsa-ns4-ss[SS]-sm2-seed1
              DRSA (Ours) &   \textbf{2.81} &   \textbf{2.53} &   \textbf{1.96} &   \textbf{2.52} &   \textbf{2.45} &   \textbf{2.16}   \\[1mm]

         % largest stderr &   0.20 &   0.16 &   0.08 &   0.25 &   0.21 &   0.41
%% end table %%
        
\textit{Error bars (max)} & $\pm$ 0.20 & $\pm$ 0.16 & $\pm$ 0.08 & $\pm$ 0.25 & $\pm$ 0.21 & $\pm$ 0.40\\

     \bottomrule
     \end{tabular}
\end{table*}

We evaluate our DRSA method against a number of baselines. The first  baseline is random subspaces constructed from a random $D \times D$ orthogonal matrix. The second  baseline, called DSA, is an ablation of the objective of DRSA where we replace the context vector with the activation vector itself.

For ImageNet experiments, the third baseline is NetDissect \cite{DBLP:journals/pami/ZhouBO019}, a state-of-the-art framework linking neurons to a large set of real-world concepts extracted from the Broden database \cite{BauZKO017}. The approach associates each filter in the layer's feature map to a concept available in the dataset.  For each identified concept, we define its subspace as the span of the standard basis vectors of the associated filters. We refer to Supplementary Note H for our reproduction details of NetDissect.
For Places365 experiments, the third baseline is  concept directions from IBD \cite{DBLP:conf/eccv/ZhouSBT18}.  Because these concept vectors do not form an orthogonal basis, we adapt our formulation of the virtual layer accordingly (details are provided in Supplementary Note I).

We set the number of subspaces to $K=4$. For the random subspaces, DSA, and DRSA, we choose the dimensions of each subspace to be $D/K$. To build the DSA and DRSA subspaces, we use each class's set of activation (and context) vectors similar to the setup in Section \ref{sec:experiment-1}. We provide  the training details of DSA and DRSA in   Supplementary Note E. Because Shapley Value Sampling is computationally demanding, we report for this method an average over only 10 validation instances per class.

Table \ref{table:aupc-multile-subspaces} shows the AUPC scores across setups. We observe that our proposed approach (DRSA) outperforms baseline methods by reaching the lowest scores across these setups. The observation  also aligns with the visual inspection of Fig.\ \ref{fig:overview} earlier in the paper, where spatially distinct concepts could be identified from the DRSA explanations.   The result suggests that the subspaces of DRSA capture distinctively relevant components in the decision of the neural network.

When ranking the 50 ImageNet classes based on the AUPCs score obtained by DRSA on  VGG16-TV, we observe that class `zebra' comes first. A subsequent visual inspection reveals that evidence for the class zebra arises from multiple, spatially disentangled, concepts such as the zebra's shape, its unique texture, and the environment in which they are located. We provide the details of the class comparison as well as qualitative examples in \mbox{Supplementary Note G}.

We further conduct ablation studies on the choice of layers and the number of subspaces. The results of these studies align with the conclusions from Table \ref{table:aupc-multile-subspaces}. In addition to these studies, we also perform experiments that verify certain intrinsic properties of the produced explanations. We refer to Supplementary Note F for these results.

\section{Application Showcases}

\begin{figure*}[b!]
    \centering
    \includegraphics[width=.875\textwidth,clip,trim=0 0 70 0]{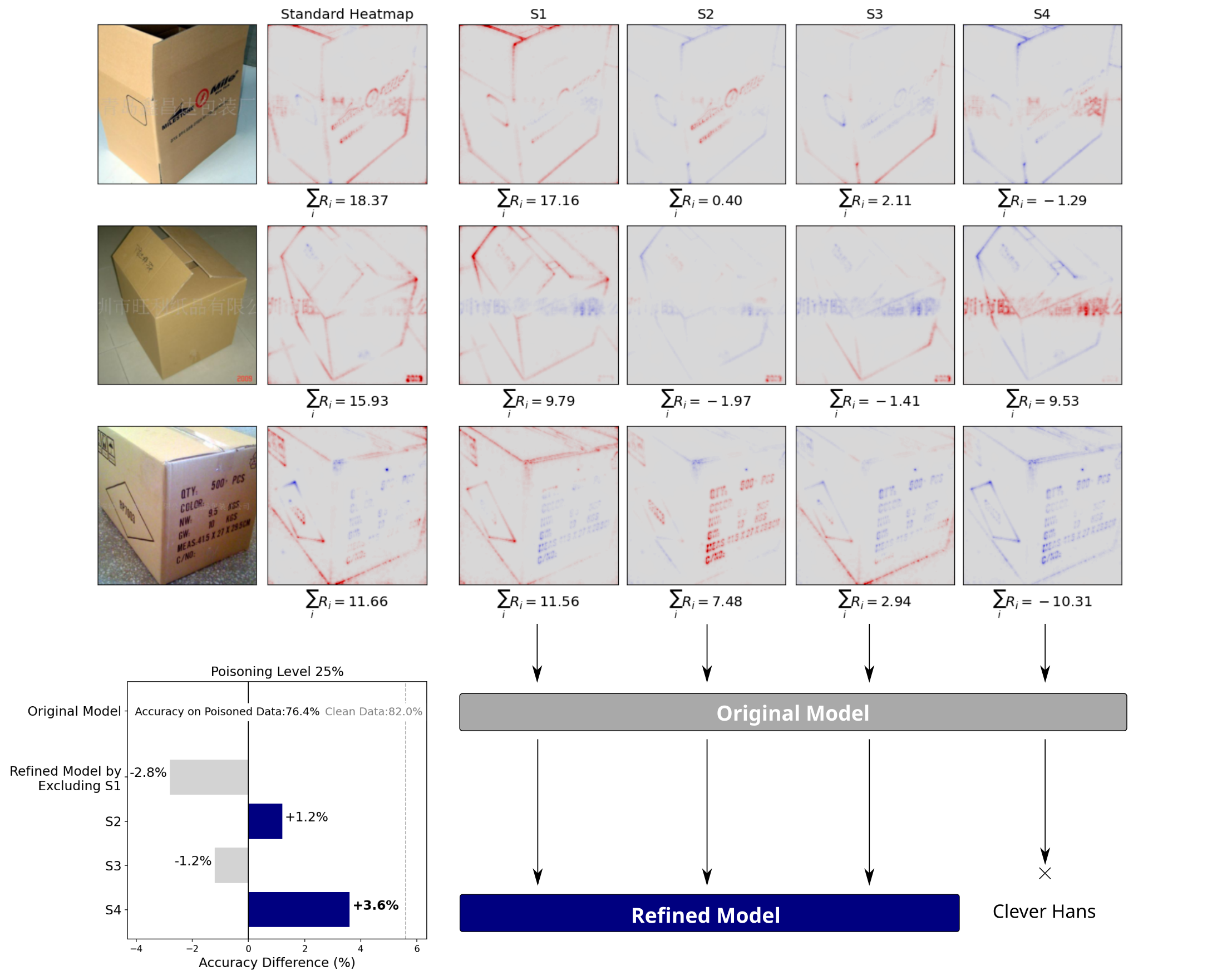}
    
    \caption{Top: Training images from class `carton' with their standard and DRSA heatmaps from VGG16-TV at Conv4\_3 using the LRP backend. The heatmaps are generated w.r.t. the logit of class `carton'. Red and blue colors indicate positive and negative pixel-concept contributions to the logit of the class. Bottom: Effect of model refinement (subtraction of each concept's contribution from the logit) on the accuracy when the validation data is poisoned with watermarks at a rate of 25\%.}
    \label{fig:showcase1-overview}
\end{figure*}

We showcase three possible applications of the proposed PRCA and DRSA methods, namely (1) building a more trustworthy ML model by detection and removal of Clever Hans strategies in the model, (2) getting better insights into the data by highlighting multiple ways input and output variables are related, and (3) bringing further understanding about the problem of adversarially manipulated explanations.

\subsection{Detecting and Mitigating Clever Hans Effects}
\label{sec:showcase-1}

A common issue with ML models is that they sometimes rely not on the true features---that should support the ML decision---but on artifactual features that spuriously correlate with the target labels on the available data. Such flawed strategies of the ML model are commonly referred to as `Clever Hans' \cite{Lapuschkin2019, DBLP:journals/inffus/AndersWNSML22}. Clever Hans models evade traditional model validation techniques, such as cross-validation, when the spurious correlation is present in both the training and test data. Nevertheless, Explainable AI can reveal these Clever Hans strategies; specifically, the user would inspect the explanation of a number of decision strategies and verify that artifactual features are not highlighted as `relevant' in the explanation.

We demonstrate in this showcase how the proposed DRSA analysis enables us to \textit{detect} and \textit{mitigate} Clever Hans effects in a highly efficient manner. In contrast to existing state-of-the-art approaches to Clever Hans detection \cite{Lapuschkin2019, bykov2022dora} and mitigation \cite{DBLP:journals/inffus/AndersWNSML22}, our approach can leverage the multiple sub-strategies readily identified by DRSA, some of which can be of Clever Hans nature. In particular, for \textit{detecting} Clever Hans strategies, one can let the user inspect one or a few representative examples of each decision strategy identified by DRSA.

We test our approach on some known example of a Clever Hans strategy: the reliance of VGG16-TV on Hanzi watermarks for predicting `carton'  \cite{DBLP:journals/inffus/AndersWNSML22}. Fig.\ \ref{fig:showcase1-overview} (top) shows three training images\footnote{We select the examples based on the procedure outlined in Supplementary Note E.2. The rationale for using training images rather than validation images is that the training set is more likely to contain features causing the CH effect, and thus more useful for inspection purposes.} from the class `carton' and their standard and DRSA heatmaps using LRP (DRSA is applied at Conv4\_3 with $K=4$). From the heatmaps, we can see that the subspace S4, unlike other subspaces, captures the Hanzi watermark quite prominently when the watermark is salient (e.g.\ the first and second examples). We therefore identify that S4 corresponds to the Hanzi Clever Hans strategy. We find that S4 is able to discriminate Clever Hans from non-Clever Hans instances with an AUROC of 0.909. We compare the Clever Hans detection ability of DRSA with that of SpRAy \cite{Lapuschkin2019}. The SpRAy method consists of performing a clustering of standard LRP heatmaps and inspecting individual clusters. In our case, we choose the same number of clusters as DRSA subspaces. In contrast, the SpRAy's most discriminative cluster achieves a slightly lower AUROC of 0.842. Details of the experiment and full ROC curves are provided in Supplementary Note J.1. We expect further gain from our approach over SpRAy when the Clever Hans features occur at different locations in the input images. Overall, our experiment demonstrates the effectiveness of DRSA at identifying Clever Hans effects.

When it comes to \textit{mitigating} Clever Hans strategies, we propose again to leverage DRSA. Specifically,  building on the subspace(s) we have identified using DRSA to be of Clever Hans nature, we propose to refine the  prediction of the class by subtracting the relevance scores associated to those subspaces from the prediction:
\begin{align}
\textstyle f^\text{(refined)}(\bx) = f(\bx) - \sum_{k \in \text{CH}} R_k(\bx).
\label{eq:showcase-1-defense}
\end{align}
In practice, we find that subspaces identified to be of Clever Hans type still contain residual non-Clever Hans contributions, especially negative ones. Hence, we propose to only consider \textit{excess} relevance given by \mbox{$R_k^\text{(excess)} = \max(0,R_k - \mathbb{E}[R_k])$}, where the expectation is computed over a set of training images from the class, here `carton'. We then use $R_k^\text{(excess)}$ in place of $R_k$ in Eq.\ \eqref{eq:showcase-1-defense}.

We now apply the proposed method to mitigate the influence of the Hanzi watermark, focusing on the class carton and other classes that VGG16-TV tends to confuse as `carton'.  We say VGG16-TV confuses a class with class `carton` if 
it has class carton in the top-3 predictions 
of its validation images with frequency at least 10\%. With the criteria, these classes are `crate', `envelope', `packet', and `safe'. Using validation images of these classes and class `carton', we then construct a classification problem in which some of the non-carton images are  poisoned with a random Hanzi watermark (from one of the three we have prepared; cf.\ Supplementary Note J.1). We apply 25\% poisoning, i.e.\ 25\% of non-carton images are inpainted with Hanzi watermarks. We observe that the classification accuracy of the original model decreases on the poisoned data (from around 82\% to 76\%). The decrease indicates that our poisoning procedure effectively fools \mbox{VGG16-TV}.

\begin{figure}
    \centering
    \includegraphics[width=.9\linewidth]{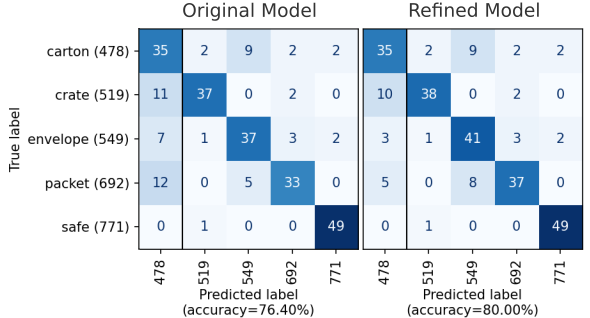}
    \caption{
        Confusion matrices from the original and refined models on clean and 25\%-poisoned data. The refined model excludes the evidence of the subspace S4 from the prediction using Eq.\ \eqref{eq:showcase-1-defense}.
    }
    \label{fig:showcase1-confusion-mat}
\end{figure}

Fig.\ \ref{fig:showcase1-overview} (bottom) shows the difference of accuracy between the original and refined models on the 25\%-poisoned data. We see that the refined model based on excluding the contribution of S4 has the highest classification accuracy (adding 3.6\% to the accuracy score of the unrefined model). We further investigate the structure of error in Fig.\ \ref{fig:showcase1-confusion-mat} which shows the confusion matrices between predicted and target classes for the original and refined model. After the refinement, we observe that the number of misclassified non-carton examples decreases substantially.  We finally compare our model refinement method with the method of \cite{DBLP:conf/iclr/KirichenkoIW23}, which consists of retraining the last layer of the model  on poisoned data. We find that the retraining approach achieves a gain of 4.8\% accuracy, slightly above our method based on DRSA. Our method, however, comes with the additional advantages of neither having to synthesize artificial Clever Hans instances nor having to choose a particular poisoning level for retraining. These advantages are decisive in the context where Clever Hans features are tightly interwoven with the other objects contained in the image. We refer to Supplementary Note J.1 for the details of the experiments, including different poisoning levels.

Overall, this showcase has demonstrated that DRSA can be an effective tool for detecting and mitigating Clever Hans effects in complex ML models. Furthermore, we stress that our approach is \textit{purely unsupervised}: It requires neither assembling a dataset of examples labeled according to the strategy the model employs to predict them, nor to generate synthetic examples where the Clever Hans features have been stripped or artificially added. Furthermore, our Clever Hans mitigation approach is \mbox{`post-hoc'}: except for the DRSA analysis, our method does not require any training or retraining of the neural network model.

\subsection{Better Insights via Disentangled Explanations}

\label{sec:showcase-2}

Explainable AI has been shown to be a promising approach to extract insights in the data and in the systems or processes that generates this data \cite{DBLP:journals/pieee/SamekMLAM21,DBLP:journals/access/RoscherBDG20}. Several recent works have shown successful usages in biomedical or physics applications. For example, Explainable AI enabled a better understanding of what geometrical aspects of molecules are predictive of toxicity \cite{DBLP:series/lncs/PreuerKRHU19} (or `toxicophores'). It also allowed to predict protein interactions in a human cell \cite{Keyl2022}, thereby supporting the research on identifying signaling pathways. There are many further examples of successful uses of Explainable AI for extracting scientific insights in geology \cite{EbertUphoff2020}, hydrology \cite{DBLP:series/lncs/KratzertHKHK19}, quantum chemistry \cite{schutt2017quantum, DBLP:series/lncs/SchuttGTM19}, neuroscience \cite{sturm-jnm16,thomas-fron19}, histopathology \cite{DBLP:journals/natmi/BinderBHWHHISHD21,klauschen2024toward}, etc. In these works, the authors often resort to standard heatmaps highlighting the extent by which one feature or a group of features contributes to the overall prediction.

\setcounter{figure}{9}

\begin{figure*}[b!]
    \centering
    \includegraphics[width=.975\textwidth]{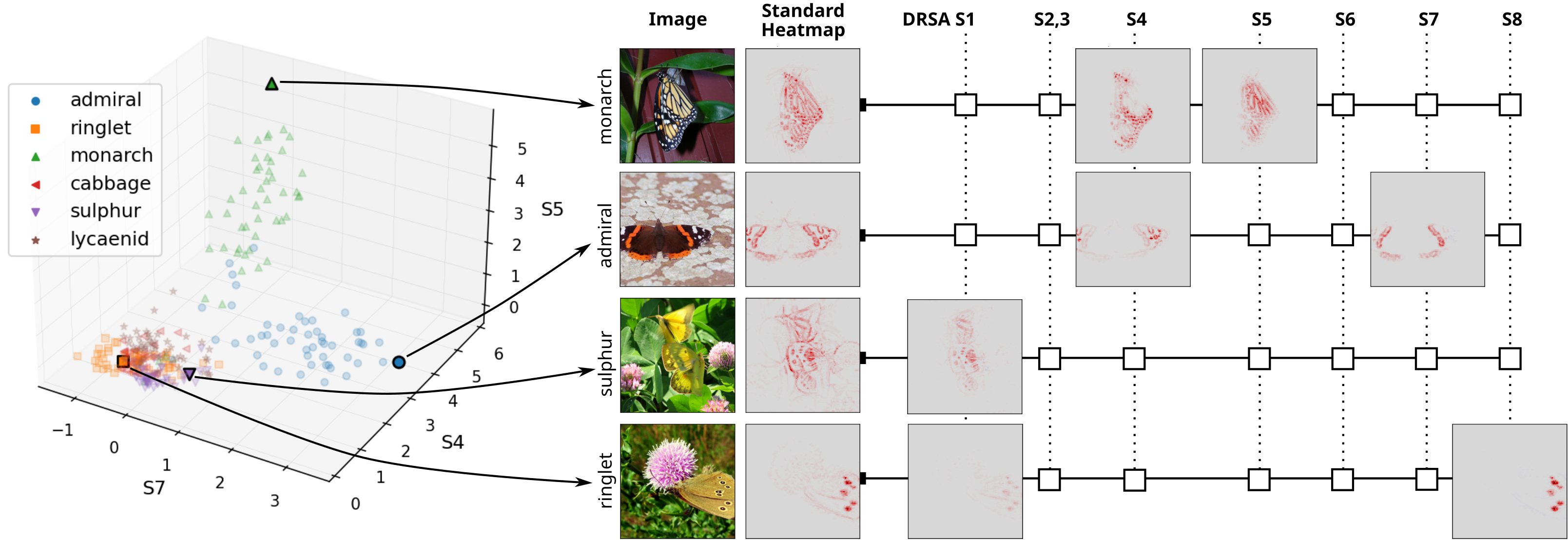}
    \caption{
        Left: Three-dimensional scatter plot resolving the relation between ImageNet's butterfly images and classes along three DRSA subspaces (S4, S5, and S7). Each point is an example, and its coordinates are the relevance scores of the three subspaces. Right: Prototypical examples and their standard and DRSA heatmaps. We show only the heatmaps of the class-subspace configurations that pass the selection criteria (Eq.\ \ref{eq:showcase-2-selection}).  We provide the complete set of heatmaps in Supplementary Note J.2. 
    }
    \label{fig:showcase-2-overview}
\end{figure*}

\setcounter{figure}{8}

\begin{figure}[t!]
    \centering
    \includegraphics[width=\linewidth,clip,trim=0.0cm 0cm 0.0cm 1cm]{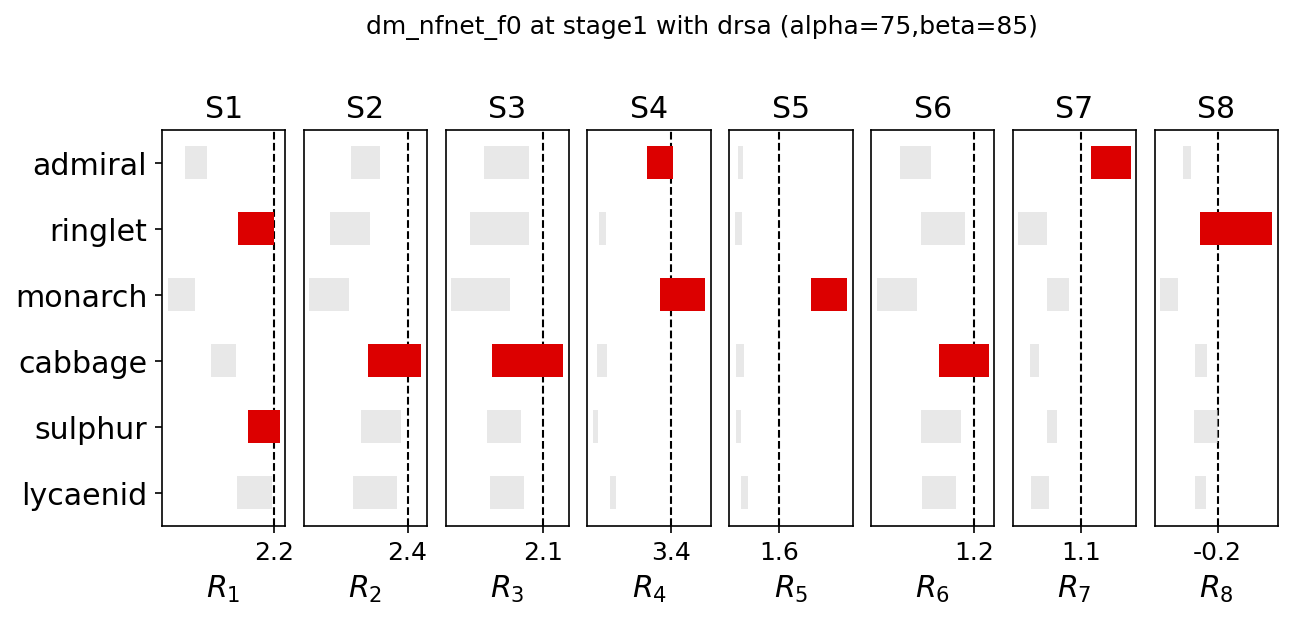}
    \caption{Statistics of DRSA relevance scores $R_k$'s from different butterfly classes.  The dashed lines indicate the $0.85$-quantile of relevance scores for each subspace. The left and right border of rectangles are the $0.25$- and $0.75$-quantiles of class-conditioned relevance scores. The selection criterion of Eq.\ \eqref{eq:showcase-2-selection} can be visually interpreted as the right side of the box surpassing the dashed line, and we highlight boxes in red if the corresponding quantile satisfies the  selection criterion.}
    \label{fig:showcase-2-boxplot}
\end{figure}

\setcounter{figure}{10}

The amount of insights one can extract from a standard explanation is however restrained by the fact that multiple concepts are entangled, and it is therefore difficult to gain a structured understanding of the relation between input and output. We showcase in the following how our proposed DRSA-LRP method enables the extraction of more sophisticated insights. We consider for an illustrative purpose the task of gaining insights into the visual differences between six classes of butterflies present in the ImageNet dataset: `admiral', `ringlet', `monarch', `cabbage', `sulphur', and `lycaenid' butterflies. 

For this showcase where the objective is for the user to gain insights from the model, it is natural to choose the best model available. We choose NFNet-F0, which achieves an overall top-1 accuracy of 82\% compared to VGG16-TV and -ND that achieve 72\% and 70\% respectively. We select 125 training images from each of these butterfly classes to form a training set. We use activation and context vectors from NFNet-F0 at Stage 1, and use LRP (with parameter $\gamma=0.1$ to compute the explanations). We extract eight subspaces using DRSA with the optimization details similar to Section \ref{sec:eval-drsa} (see also Supplementary Note E).

First, we would like to build a correspondence table between classes and concepts, indicating for each class which concepts are specific to it. We propose the following simple statistical test, which accounts for the fact that concepts are typically expressed only in a subset of images from the given class. Denote $\mathcal D_\omega$ to be the set of class $\omega$'s validation images and $\mathcal D$ the set of all validation images from the investigated classes (in our showcase, all butterfly images). We consider Subspace $k$ to be specific to \mbox{class $\omega$} if 
\begin{align}
     Q_{\alpha}[R_{k}|\omega] >  Q_{\beta}[R_{k}]
    \label{eq:showcase-2-selection},
\end{align}
where $Q_\alpha$ is the $\alpha$-quantile of the given distribution and $\alpha < \beta$. In our experiments, we choose $\alpha = 0.75$ and $\beta = 0.85$. In this equation, scores $R_k$ are measured via $\sum_i R_{ik}$.

Fig.\ \ref{fig:showcase-2-boxplot} illustrates the process of matching classes with DRSA subspaces. The right border of the rectangles and the dashed lines correspond to the left- and right-hand sides of Eq.\ \eqref{eq:showcase-2-selection}. The analysis reveals 10 class-subspace matchings (highlighted in red). We observe that each DRSA subspace is associated to one type of butterfly, except for the subspaces S1 and S4, which matches multiple classes, thereby indicating visual concepts that are shared between multiple classes. Furthermore, the number of concepts associated to a particular class vary from one (sulphur butterfly) to three (cabbage butterfly).

Fig.\ \ref{fig:showcase-2-overview} (left) explores, using a three-dimensional scatter plot, how the relation between butterflies and their respective classes is resolved by the subspaces S4, S5, and S7 of our DRSA analysis. Each point in this plot corresponds to one example, and its coordinate is given by the scores $R_k$'s. As already noted in Fig.\ \ref{fig:showcase-2-boxplot}, we observe that `monarch' is jointly expressed along axes S4 and S5, and `admiral' is jointly expressed along axes S4 and S7. These subspaces are not relevant for the other classes; hence, their respective examples appear near the origin.

\begin{figure*}[t!]
    \centering
    \includegraphics[width=0.9\textwidth]{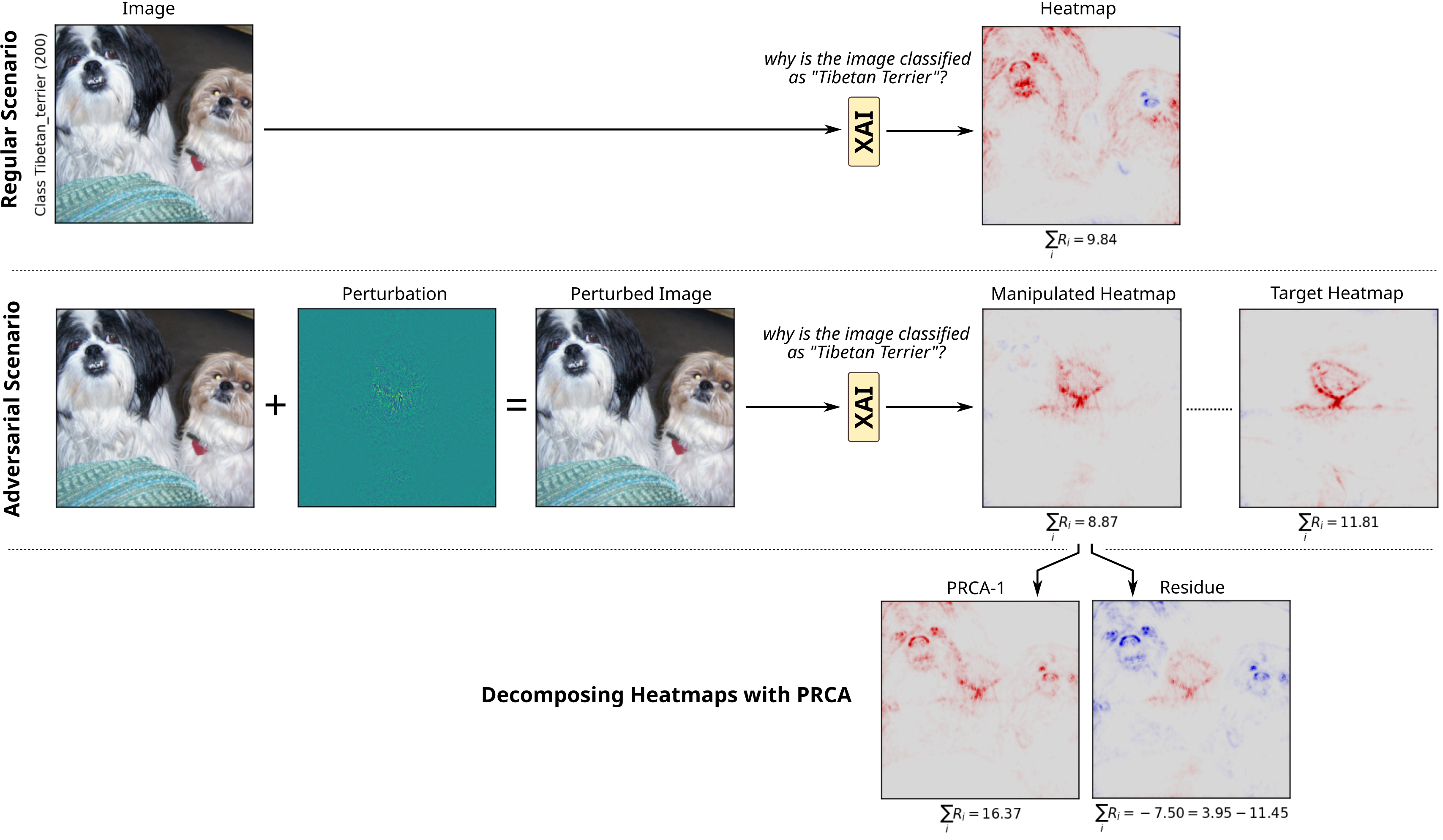}
    \caption{Application of PRCA for shedding light on a forgery in Explainable AI. Top: Regular scenario. Middle: Adversarial scenario, following the approach of \cite{DBLP:conf/nips/DombrowskiAAAMK19}, in which a perpetrator imperceptibly perturbs an image in a way that its heatmap (here the LRP heatmap associated to the output `tibetan terrier' of VGG16-TV) is steered maliciously towards an incorrect target heatmap. Bottom: PRCA of the perturbed image at Conv4\_3. The analysis decomposes the manipulated heatmap into the heatmap of class tibetan terrier's maximally relevant direction and its residue. Red and blue colors indicate positive and negative contributions.}
    \label{fig:showcase3-overview}
\end{figure*}

Fig.\ \ref{fig:showcase-2-overview} (right) shows pixel-wise explanations for the most prototypical examples of a few selected classes\footnote{We show for the selected classes the example $\arg\max_{n} \min_{k \in \mathcal{K}} R_{k,n}$, where $\mathcal{K}$ is the set of subspaces associated to the given class, and $R_{k,n}$ is the contribution of Subspace $k$ for example $n$ to its associated class.}. We observe that  S1 corresponds to yellow colored surfaces, which seems to be common of ringlet and sulphur butterflies. S4 corresponds to white-dot texture, which is found on monarch's wings and body and admiral's wings. S5's pattern is specific to the orange/black texture on the wings of monarch specie. S7 captures the prominent orange pattern on the wings of the admiral butterfly. Lastly, we find that S8 captures the distinct dotted pattern that appears on the wings of the ringlet species. We provide the complete set of these subspace heatmaps in Supplementary Note J.2.

Overall, throughout this showcase, we have demonstrated that our method is capable of providing further insights into the complex relation between visual features and class membership. In addition to highlighting features that are predictive of class membership, we have identified distinct visual concepts, such as dotted patterns or yellow textures, that are shared between multiple classes. These shared visual patterns provide a structured understanding of the nonlinear relation between butterfly species and their visual characteristics.

\subsection{Analyzing Manipulated Explanations using PRCA}
\label{sec:showcase-3}

One of the premises of Explainable AI is to facilitate trust to stakeholders, but previous works \cite{DBLP:conf/nips/DombrowskiAAAMK19,DBLP:conf/aaai/GhorbaniAZ19} show that explanation techniques are vulnerable to manipulation. More concretely, a slight perturbation of the input could lead to substantial changes in its explanation  while maintaining visual similarity to the input and other statistics  (e.g.\ model output). Crucially, \cite{DBLP:conf/nips/DombrowskiAAAMK19} shows that such perturbation can lead to arbitrary changes in explanations, having neither relation to the input nor the original heatmap. Fig.\ \ref{fig:showcase3-overview} contrasts such a scenario (where a perpetrator perturbs an image to manipulate its explanation) and a regular Explainable AI scenario.

Certainly, the vulnerability to perturbation does not only raise practical concerns, but also theoretical questions on how such a phenomenon could happen. As a result, a number of theoretical analyses \cite{DBLP:conf/nips/DombrowskiAAAMK19, DBLP:conf/aaai/GhorbaniAZ19} have been conducted to investigate the cause of the perturbation vulnerability. In particular, the investigation of  \cite{DBLP:conf/nips/DombrowskiAAAMK19} elucidates that the degree to which an explanation can change is partially upper-bounded by the principal curvature evaluated at the data point. Furthermore, \cite{DBLP:conf/nips/DombrowskiAAAMK19} shows that, for neural networks with ReLU, the principal curvature can be reduced by approximating ReLU with the softplus activation. By controlling the smoothness parameter of the softplus function, \cite{DBLP:conf/nips/DombrowskiAAAMK19} shows that the robustness of explanation manipulation can be effectively increased in a post-hoc manner.

Nevertheless, from the perspective of layer-wise representation, it is still unclear how perturbation causes such dramatic changes in explanation or how such changes manifest at a certain layer. We therefore aim to demonstrate that PRCA  might provide a clue  to   answer such questions.

As a proof of concept, we study the PRCA decomposition of LRP explanations from validation images of class `tibetan terrier' in the ImageNet dataset \cite{imagenet_cvpr09} on VGG16-TV at Conv4\_3. More precisely, we perform PRCA on a set of activation and context vectors from 500 training images of the class (details similar to the setup of Section \ref{sec:experiment-1}).

To manipulate explanations, we use the optimization procedure proposed by \cite{DBLP:conf/nips/DombrowskiAAAMK19} to find a perturbation that causes arbitrary changes in the explanation of each image, while retaining the same level of model response and visual similarity between the original and perturbed images. The arbitrary changes are induced by a target explanation, which is the explanation of a random image from a different class. 
In addition,  we also constrain  the original and manipulated explanations to have similar total relevance scores. We summarize the details of the algorithm in \mbox{Supplementary Note  J.3}.

Qualitatively, Fig.\ \ref{fig:showcase3-overview} (bottom)  shows that  the heatmap generated from the first PRCA component preserves features highlighted in the original heatmap, while the residual heatmap (orthogonal complement of the first PRCA component) contains features from both the original and target heatmaps. 

Looking at the positive and negative parts of the residual heatmap, we observe that the former substantially resembles the target heatmaps, while the latter is closely similar to part of the original heatmap expressed in the PRCA heatmap with opposite sign. When using more PRCA components, the PRCA heatmap  becomes similar to the target heatmap (see Fig.\ J.8 in the Supplementary Notes). The behavior suggests that, for VGG16-TV at Conv4\_3 and class \mbox{`tibetan terrier'}, the first PRCA component is the direction affected the least by perturbation.

\begin{figure}
    \centering
    \includegraphics[clip,trim=0 0 0cm 0.0cm,width=.8\linewidth]{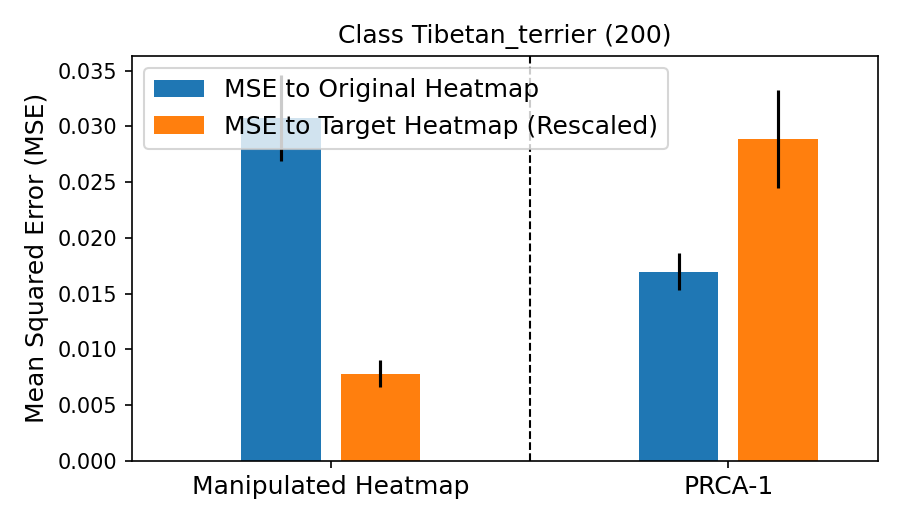}
    \caption{Mean squared error to original or target heatmaps and different sets of heatmaps averaged over 50 validation images of class `tibetan terrier'. These sets of heatmaps are manipulated heatmaps and their decomposition on the first PRCA component (PRCA-1). Two heatmaps are similar if the error between them is small. Vertical lines represent one standard error.}
    \label{fig:showcase3-distance-betweenheatmaps}
\end{figure}

Quantitatively, Fig.\ \ref{fig:showcase3-distance-betweenheatmaps} shows the mean squared error  between manipulated heatmaps (and their PRCA decomposition versions) and original or rescaled target heatmaps: the error is averaged over the 50 validation images of class `tibetan terrier'. We first observe that the manipulated heatmaps have lower error when comparing to the target heatmaps than the original heatmaps. It confirms that the optimization proposed by \cite{DBLP:conf/nips/DombrowskiAAAMK19} is indeed effective and also works well with the additional constraint we impose.

Secondly, when looking at the error from the manipulated heatmaps on the first PRCA component (PRCA-1), we observe that these heatmaps are closer to the original heatmaps than the target ones. The difference between the two errors is interesting because it indicates that PRCA indeed captures parts of the class-specific representation that are less affected by the perturbation. The insight may provide a new perspective towards understanding and increasing the robustness of explanation manipulation (cf.\ also \cite{dombrowski2022towards}).

\section{Conclusion and Discussion}

In this work, we have proposed to disentangle explanations of neural network models into multiple components in order to provide more useful information to the user compared to a standard explanation.

Technically, the desired disentanglement is achieved via an unsupervised analysis at some intermediate layer of the neural network model. A unique aspect of the proposed method is that it analyzes jointly the data and model response to the data. Hence, unlike a purely data-driven approach, our method enables a more focused disentanglement that efficiently ignores aspects of the data to which the model is invariant. Besides, our approach does not require any specialized datasets or concept annotations and can be applied to any deep neural network model whose predictions can be attributed to input features. Our method works together with a broad range of state-of-the-art attribution frameworks, such as the Shapley value and LRP.

We have demonstrated the high performance of our disentanglement approach, scoring significantly higher than other methods in our benchmark. Through an implementation of LRP we have contributed for the state-of-the-art NFNet model, we have further demonstrated that our approach can bring more light into highly sophisticated prediction functions. Building upon existing attribution techniques, our method also inherits challenges with the problem of attribution, such as the need to adapt to the rapidly increasing complexity of ML models.

On a practical note, we have demonstrated the use of our method on three application showcases: 1) detection and defusion of Clever Hans strategies in the popular VGG16 image classifier, 2) in-depth exploration of a complex nonlinear relation of interest, subsumed by a state-of-the-art ML model, in order to acquire new domain knowledge, and 3) investigation of the problem of adversarially manipulated explanations, for which we could gain new understanding.

In future work, we plan to apply our methods to analyze complex scientific data, and thereby, help a domain expert to obtain new scientific insights. Furthermore, our method could be extended towards the extraction of \textit{irrelevant} subspaces. The latter could then be pruned from the ML model, e.g.\ for compression purposes, or to robustify the model against unknown, potentially spurious, decision strategies \cite{DBLP:journals/inffus/LinhardtMM24}. Lastly, our proposed approach, which combines Explainable AI and representation learning, could be explored beyond the framework of attribution, for example, in the context of counterfactual explanations (e.g.~\cite{dombrowski2023diffeomorphic}).

\section*{Acknowledgments}

This work was supported by the German Ministry for Education and Research (BMBF) under Grants 01IS14013A-E, 01GQ1115, 01GQ0850, 01IS18025A, 031L0207D, and 01IS18037A, and by BASLEARN---TU Berlin/BASF Joint Laboratory, co-financed by TU Berlin and BASF SE. P.C.\ was supported by the Max Planck Society and the Konrad Zuse School of Excellence in Learning and Intelligent Systems (ELIZA) through the DAAD programme Konrad Zuse Schools of Excellence in Artificial Intelligence, sponsored by the German Ministry of Education and Research. K.R.M.\ was partly supported by the Institute of Information \& Communications Technology Planning \& Evaluation (IITP) grants funded by the Korea government (MSIT) (No. 2019-0-00079, Artificial Intelligence Graduate School Program, Korea University and No. 2022-0-00984, Development of Artificial Intelligence Technology for Personalized Plug-and-Play Explanation and Verification of Explanation). This work used the Scientific Compute Cluster at GWDG, the joint data center of Max Planck Society for the Advancement of Science (MPG) and University of G\"ottingen. We thank Thomas Schnake, Lorenz Linhardt, Simon \mbox{Letzgus}, and Ali Hashemi for helpful comments and feedback to improve the manuscript.

\section*{Data \& Code Availability}

We provide demonstration code at \url{https://github.com/p16i/drsa-demo}. The repository contains an implementation of LRP for VGG16 and NFNets that is compatible with our disentangled explanation framework and functionalities to perform the optimization of DRSA. In the `notebooks` directory, we provide two Jupyter notebooks demonstrating 1) the steps in the disentangled explanation framework and the reproduction of Fig.\ \ref{fig:overview}; and 2) the LRP implementation for NFNets.

\bibliographystyle{ieeetr}
\bibliography{ref}

\end{document}

\maketitle

\vspace{-1cm}

\tableofcontents

\begin{appendices}

\section{Overview of Attribution Techniques}
\label{sec:overview-attribution-methods}

Let $\x = (x_i)_i$ be a data point formed by the collection of input features. Attribution techniques aim to decompose the prediction of a ML model $f(\x)$ onto these features, i.e.\ producing scores $(R_i)_i$ where $R_i$ identifies the contribution of feature $i$ to the given prediction. We present three popular methods for performing attribution, and which we can use as part of our proposed framework for producing disentangled explanations.

\subsection{Shapley Value}
\label{appendix:attribution-shapley}

The Shapley Value \cite{shapley:book1952} is an attribution technique with foundations in game theory. It addresses the questions of how to redistribute a certain reward among a set of cooperating players, and was motivated as being the unique solution satisfying a set of basic axioms of the redistribution process. In machine learning terms, the players are the input features $i=1,\dots, d_x$, and the reward is the value $f(\x)$ at the output of the network \cite{DBLP:journals/jmlr/StrumbeljK10,DBLP:conf/nips/LundbergL17}. The Shapley value formula for attribution is:
\begin{align}
R_i = \sum_{\mathcal{S}:i \notin \mathcal{S}} \alpha_\mathcal{S} \cdot [f(\x_{\mathcal{S} \cup \{i\}}) - f(\x_\mathcal{S})],
\label{eq:shapley}
\end{align}
where $\sum_{\mathcal{S}:i \notin \mathcal{S}}$ is the sum of all subsets of input features that do not include feature $i$, $\x_\mathcal{S}$ is an artificial example where features other than $\mathcal{S}$ have been removed (e.g.\ set to zero), and $\alpha_\mathcal{S} = \frac{|\mathcal{S}|! \cdot (d_x - 1 - |\mathcal{S}|)!}{d_x!}$. Shapley values satisfy the conservation property $\sum_i R_i = f(\x) - f(\widetilde{\x})$ where $\widetilde{\x} = \x_{\{\}}$ an artificial example where all features have been removed, or `reference point'. While computation of the Shapley value is exponentially complex, various approximations have been proposed such as building local surrogates where Shapley values are easy to compute \cite{DBLP:conf/nips/LundbergL17}, or sampling approaches \cite{DBLP:journals/cor/CastroGT09, DBLP:journals/jmlr/StrumbeljK10}. In sampling approaches, the Shapley formula in Eq.\ \eqref{eq:shapley} is cast into an expectation over a probability distribution, and a random sample is drawn from that distribution.

\subsection{Integrated Gradients}

Integrated Gradients \cite{DBLP:conf/icml/SundararajanTY17} is another attribution technique, which leverages the fact that the gradient $\nabla f(\x)$ is a $d_x$-dimensional vector, and a few evaluations of that gradient potentially enables an attribution onto the many input features. The integrated gradient attribution is defined by the integral:
\begin{align}
R_i = \int_0^1 \frac{\partial f}{\partial x_i} \frac{\partial x_i}{\partial t} dt,
\end{align}
where the input $\x$ is a function of $t$, often a linear path from some reference point $\widetilde{\x}$ (e.g.\ the origin in input space) to the actual data point.  Integrated gradients satisfy the conservation property $\sum_i R_i = f(\x) - f(\widetilde{\x})$. In practice, the integral is discretized so that the function only has to be evaluated finitely many times, typically between 10 and 100 times. For increased robustness, averaging of results over multiple paths can be considered (see \cite{DBLP:journals/natmi/ErionJSLL21}), although it incurs an additional computational cost.

\subsection{Layer-wise Relevance Propagation}
\label{appendix:attribution-lrp}

Layer-wise Relevance Propagation (LRP) \cite{bach-plos15,DBLP:series/lncs/MontavonBLSM19} tackles attribution by performing a purposely designed backpropagation pass from the output of the network $f(\x)$ to the input features. A main advantage of LRP is that it computes explanations in the order of a single forward/backward pass, making it suitable for generating many explanations on large neural network models. Let $j$ and $k$ be neuron indices for two consecutive layers. The activations between these two layers are linked via the function
\begin{align}
z_k &= \textstyle \sum_{0,j} a_j w_{jk},\\
a_k &= \rho(z_k),
\end{align}
where $\rho(\cdot)$ is an activation function and $\sum_{0,j}$ runs over all neurons of the corresponding layer plus a bias $b_k$ (i.e.\ $b_k = w_{0k}$). The propagation at each layer is defined by means of a propagation rule. Examples of propagation rules are:
\begin{align}
\text{LRP-$0$ \cite{bach-plos15}:} \qquad R_j &= \sum_k \frac{a_j w_{jk}}{\sum_{0,j} a_j w_{jk}} R_k\\
\text{LRP-$\epsilon$ \cite{bach-plos15}:} \qquad R_j &= \sum_k \frac{a_j w_{jk}}{\epsilon_k + \sum_{0,j} a_j w_{jk}} R_k\\
\text{LRP-$\gamma$ \cite{DBLP:series/lncs/MontavonBLSM19}:} \qquad R_j &= \sum_k \frac{a_j (w_{jk} + \gamma w_{jk}^+)}{\sum_{0,j} a_j (w_{jk} + \gamma w_{jk}^+)} R_k\\
\text{generalized LRP-$\gamma$ \cite{Keyl2022}:} \qquad R_j &= \sum_k \frac{a_j^+ (w_{jk} + \gamma w_{jk}^+) + a_j^- (w_{jk} + \gamma w_{jk}^-)}{\sum_{0,j} a_j^+ (w_{jk} + \gamma w_{jk}^+) + a_j^- (w_{jk} + \gamma w_{jk}^-)} \cdot \indicator{z_k \ge 0}\cdot R_k \nonumber \\ 
    &\qquad + \sum_k \frac{a_j^+ (w_{jk} + \gamma w_{jk}^-) + a_j^- (w_{jk} + \gamma w_{jk}^+)}{\sum_{0,j} a_j^+ (w_{jk} + \gamma w_{jk}^-) + a_j^- (w_{jk} + \gamma w_{jk}^+)} \cdot \indicator{z_k < 0}\cdot R_k
    \label{eq:generalized-lrp}\\
\text{$z^\mathcal{B}$-rule \cite{DBLP:journals/pr/MontavonLBSM17}:} \qquad R_i &= \sum_j \frac{x_i w_{ij} - l_i w_{ij}^+ - h_i w_{ij}^- }{\sum_i x_i w_{ij} - l_i w_{ij}^+ - h_i w_{ij}^- } R_j,  \label{eq:lrp-B}
\end{align}
where we have used the notation $(\cdot)^+ =\max(0, \cdot)$ and $(\cdot)^- =\min(0, \cdot)$, and denoted by $\indicator{\cdot}$ an indicator function. These multiple rules have different application conditions. For example, the LRP-$\gamma$ rules requires positive input activations and an output activation function $\rho$ that maps negative values to zero and positive values to positive values, e.g.\ ReLU. This restriction is dropped for LRP-$0$, LRP-$\epsilon$ and generalized LRP-$\gamma$, where inputs can be both positive and negative, and where the only restriction is that the activation function $\rho$ is sign-preserving, e.g.\ tanh, LeakyReLU, GELU \cite{DBLP:journals/corr/HendrycksG16}, etc. The $z^\mathcal{B}$-rule is specialized for input layers receiving pixels, and the parameters $l_i \leq 0$ and $h_i \geq 0$ in this rule are the lowest and highest possible value of $x_i$'s (e.g.\ corresponding to the value of black and white pixels). The rules above address convolution and dense layers. Propagation through max-pooling layers typically follows a winner-take-all \cite{bach-plos15} (or winner-take-most) strategy. Batch normalization layers are typically fused with the preceding linear layer \cite{DBLP:series/lncs/MontavonBLSM19}. Backpropagation strategies have also been developed for LSTM and transformer architectures (cf.\ \cite{DBLP:series/lncs/ArrasAWMGMHS19,DBLP:conf/icml/AliSEMMW22}). In absence of neuron biases, most of these rules implement the layer-wise conservation $\sum_j R_j = \sum_k R_k$, and such layer-wise conservation implies the overall conservation $\sum_i R_i = f(\x)$.

\section{Disentangled Explanations with Various Attribution Techniques}
\label{sec:examples-subspace-attributions}
In this note, we show how to apply our approach to disentangling explanations in conjunction with attribution techniques other than LRP, specifically, \gradinput, \intgrad, and the Shapley value. For this, we need for each of them to express the two steps of attribution, and verify that relevance scores $R_k$'s have the structure required by the PRCA/DRSA analyses. In the following, we use the following decompositions of the neural network function $f$:
\begin{align*}
~\x~ 
\overbrace{
\overbrace{
\xmapsto{\displaystyle \big. \qquad \phi \qquad}
~\ba~
\xmapsto{\displaystyle \big. \qquad \bbu^\top \qquad}
}^{\displaystyle \big. f_1}
~\bh~
\overbrace{
\xmapsto{\displaystyle \big. \qquad \bbu \qquad}
~\ba^\prime~
\xmapsto{\displaystyle \big. \qquad g \qquad}
}^{\displaystyle \big. f_2}
}^{\displaystyle \big. f}
~y~
\end{align*}
with $\bx = (\bx_p)_{p=1}^P$ and $\bh = (\bh_k)_{k=1}^K$.
The first (coarse) decomposition is used for the derivation of the two-step attributions, and the second (fine) decomposition is used to analyze the structure of $R_k$.

\subsection{Deriving Two-Step Attributions}
\label{note:two-step}

We start with the two-step attribution process described generically in the main paper as:
\begin{align}
\boxed{
\begin{array}{lrl}
\big.\text{step 1:}~~~& (R_k)_k \!\!\!\!\! &= \mathcal{E}(y,\bh),\\
\big.\text{step 2:}~~~& (R_{pk})_p\!\!\!\!\! &= \mathcal{E}(R_k, \bx).
\end{array}
}
\end{align}
We have explained that, when using the LRP attribution technique, this two-step process can be readily implemented by filtering the backpropagation flow to only retain what passes through the neurons with index $k$. We demonstrate how to perform these two steps of explanation for non-backpropagation methods, in particular, \gradinput, \intgrad, and Shapley value.

\subsubsection{Application to \gradinput{}}
\label{note:hierarchical-gi}
Let $f_1$ and $f_2$ be two piecewise linear functions. Using \gradinput{} as an attribution method for each step, we get:
\begin{align}
R_k &= \frac{\partial y}{\partial \bh_k} \bh_k,\\
R_{pk} &= \frac{\partial}{\partial \bx_p}\Big(\frac{\partial y}{\partial \bh_k} \bh_k\Big) \bx_p.
\intertext{Observing that $\partial y /\partial \bh_k$ is piecewise constant w.r.t.\ $\bh$ (and therefore piecewise constant with w.r.t.\ $\x$), we take this expression out of $\partial / \partial \bx_p (\cdot)$, which gives us}
R_{pk} &= \frac{\partial y}{\partial \bh_k} \frac{\partial \bh_k}{\partial \bx_p} \bx_p
\label{eq:gradinput-Rik-final}
\end{align}
Hence, where both derivatives can be computed separately, and we can identify in Eq.\ \eqref{eq:gradinput-Rik-final} the multivariate chain rule for derivatives filtered to only include the term that depends on $k$. Thus, we can relate the two-step and one-step \gradinput{} as $\sum_{k} R_{pk} = R_p$. Furthermore, if $f_1$ and $f_2$ are first-order positively homogeneous, then we have that $\sum_p R_p = \sum_{pk} R_{pk} = \sum_k R_k = y$.

\subsubsection{Application to \intgrad{}}
\label{note:hierarchical-intgrad}
Let $f_1$ and $f_2$ be two differentiable functions. Using \intgrad{} as an attribution method for each step with a linear integration path starting at the origin, we get:
\begin{align}
R_k &= \int \frac{\partial y}{\partial \bh_k} \frac{\partial \bh_k}{\partial s} ds,\\
R_{pk} &= \int \frac{\partial}{\partial \bx_p} \Big( \int \frac{\partial y}{\partial \bh_k} \frac{\partial \bh_k}{\partial s} ds \Big) \frac{\partial \bx_p}{\partial t} dt. \label{eq:ig-nested}
\end{align}
The double integral is expensive in practice. For practical purpose, we can locally approximate the complicated function $R_k$ by a `relevance model' $\widehat{R}_k$, specifically, a linear function of $\bh_k$.
A possible relevance model is:
\begin{align}
\widehat{R}_k = \bh_k^\top \boldsymbol{1} \cdot [R_k / ( \bh_k^\top \boldsymbol{1})]_\text{cst.}
\end{align}
where $[\cdot]_\text{cst.}$ denotes a constant approximation of the expression evaluated at the current point, and where we use the convention $0/0=0$.
With this approximation, one can more efficiently attribute to the input features by performing the subsequent integrated gradient calculation:
\begin{align}
R_{pk} &= \int \frac{\partial \widehat{R}_k}{\partial \bx_p} \frac{\partial \bx_p}{\partial t} dt,
\end{align}
which unlike Eq.\ \eqref{eq:ig-nested} does not have a nested integral. Furthermore, assuming the functions $f_1$ and $f_2$ are zero at the origin, we get the conservation property $\sum_p R_p = \sum_{pk} R_{pk} = \sum_k R_k = y$, however, $\sum_k R_{pk} \neq R_p$ generally. Hence, we have a weaker form of conservation than the one obtained with \gradinput{} and LRP, and this is due to the relevance modeling step.

\subsubsection{Application to Shapley Values}
\label{note:hierarchical-shapley}

Using the Shapley value as an attribution method for each step (with reference points $\widetilde{\x}=0$ and  $\widetilde{\bh}=0$), we get:
\begin{align}
R_k &= \sum_{\mathcal{T}:k \notin \mathcal{T}} \beta_\mathcal{T} \cdot [y(\bh_{\mathcal{T} \cup k}) - y(\bh_{\mathcal{T}})],\\
R_{pk} &=
\sum_{\mathcal{S}:p \notin \mathcal{S}} \alpha_\mathcal{S} \cdot
\big[
R_k(\x_{\mathcal{S} \cup p})
- 
R_k(\x_{\mathcal{S}}) \big]\\
&=
\sum_{\mathcal{S}:p \notin \mathcal{S}} \alpha_\mathcal{S}
\sum_{\mathcal{T}:k \notin \mathcal{T}} \beta_\mathcal{T}
\big[
y\big(\bh_{\mathcal{T} \cup k}(\x_{\mathcal{S} \cup p})\big) -
y\big(\bh_{\mathcal{T}}(       \x_{\mathcal{S} \cup p})\big) -
y\big(\bh_{\mathcal{T} \cup k}(\x_{\mathcal{S}}       )\big) + y\big(\bh_{\mathcal{T}}(       \x_\mathcal{S}         )\big)
\big].
\end{align}
Like for Integrated Gradients, the nesting makes the computation expensive. We proceed similarly to Integrated Gradients by building a relevance model:
\begin{align}
\widehat{R}_k = \bh_k^\top \boldsymbol{1} \cdot [R_k / ( \bh_k^\top \boldsymbol{1})]_\text{cst.}
\label{eq:relmodel-shapley}
\end{align}
with $0/0=0$. Under this relevance model, we can more efficiently compute the joint relevance scores:
\begin{align}
R_{pk} = \sum_{\mathcal{S}:p \notin \mathcal{S}} \alpha_\mathcal{S} \cdot [\widehat{R}_k(\x_{\mathcal{S} \cup p}) - \widehat{R}_k(\x_{\mathcal{S}})].
\end{align}
Note that if the functions $f_1$ and $f_2$ are zero valued at the origin, then we have the conservation property $\sum_p R_p = \sum_{pk} R_{pk} = \sum_k R_k = y$. However, like for Integrated Gradients, the relevance modeling step implies that $\sum_k R_{pk}$ typically differs from the original Shapley value $R_p$. This weaker form of conservation is again due to the relevance modeling step.

\subsection{Verifying the Structure of $R_k$}
\label{sec:expressing-rel-factors}

Recall from the main paper that for PRCA/DRSA to be applicable, relevance scores $R_k$'s associated to the vector $\bh_k$'s should be expressible in terms of the orthogonal matrix $\bbu$ and activation/context vectors as:
\begin{align}
\boxed{
R_k = \big(U_k^\top \ba\big)^\top \big(U_k^\top \bc\big).
}
\end{align}
We verify that $R_k$'s have this structure for \gradinput{} and \intgrad{}, and show which approximation can be made for recovering such structure when using Shapley values.

\subsubsection{Structure of $R_k$'s with \gradinput{}}
\label{note:relk-gi}

When using \gradinput{}, the desired structured can be identified from an application of the chain rule for derivatives:
\begin{align}
R_k &= \frac{\partial y}{\partial \bh_k} \bh_k\\
&=  \frac{\partial y}{\partial \ba^\prime} \frac{\partial \ba^\prime}{\partial \bh_k} \bh_k,
\intertext{and observing that $\bh_k = U_k^\top \ba$ and that $\partial \ba^\prime / \partial \bh_k = U_k$, we get,}
&= \bc^\top U_k U_k^\top \ba\\
&= \big(U_k^\top \ba\big)^\top \big(U_k^\top \bc\big)
\end{align}
with $\bc = \partial y / \partial \ba^\prime$.

\subsubsection{Structure of $R_k$'s with \intgrad{}}
\label{note:relk-ig}

When using \intgrad{} with a linear integration path from $\boldsymbol{0}$ to $\bh_k$, we state the \intgrad{}  equation and apply the chain rule for derivatives:
\begin{align}
R_k &=
\int 
\frac{\partial y}{\partial \bh_k}
\frac{\partial \bh_k}{\partial t}
dt\\
&=
\int 
\frac{\partial y}{\partial \ba^\prime}
\frac{\partial \ba^\prime}{\partial \bh_k}
\frac{\partial \bh_k}{\partial t}
dt\\
&=
\int 
\frac{\partial y}{\partial \ba^\prime}
U_k
\bh_k
dt.
\intertext{Taking out constant terms from the integral and expressing $\bh_k$ as a function of $\ba$, we get:}
 &= \Big(\int \frac{\partial y}{\partial \ba^\prime} dt\Big) U_k U_k^\top \ba\\
&= \big( U_k^\top \ba\big)^\top\big( U_k^\top \bc \big) \label{eq:rk-intgrad}
\end{align}
with $c_j = \int \frac{\partial y}{\partial a_j^\prime} dt$. This is similar to the \gradinput{} case, except that the gradient $\partial y / \partial \ba^\prime$ is evaluated and averaged over all activations vectors encountered on the linear integration path. This specific structure of $R_k$ lets us revisit the relevance model $\widehat{R}_k$ proposed in Supplementary Note \ref{note:two-step} for performing the second step of attribution. In particular, an inspection of Eq.\ \eqref{eq:rk-intgrad} suggests the alternate relevance model
\begin{align}
\widehat{R}_k &= \big( U_k^\top \ba\big)^\top\big( U_k^\top [\bc]_\text{cst.} \big),
\end{align}
where $[\bc]_\text{cst.}$ is a constant approximation of $\bc$ evaluated at the current data point.

\subsubsection{Structure of $R_k$'s with Shapley Value}
\label{note:relk-shapley}

The Shapley value equation does not allow for factoring out the transformation matrices $U_k$'s as it was the case for the methods above. To incorporate Shapley value attribution (with baseline $\widetilde{\bh}_k = \boldsymbol{0}$), we have discussed in Supplement Note \ref{note:hierarchical-shapley}  how to perform Shapley value attribution on the activation layer.
Also, because the Shapley value method is typically slow for high dimensions, attribution can in practice be performed in terms of groups of activations (e.g.\ the collection of activations in a feature map $j$).

Let us denote by $\ba_j$ the activations in feature map\footnote{To illustrate, suppose the feature map of a given layer has $D$   channels and $h \times w$ spatial dimensions. The $j$th group of activations $\ba_j$ is a vector of size $hw$ for  $j \in \{1, \dots, D\}$.} $j$, and $R_j$ the attribution on this group of activations obtained with the Shapley value framework (with baseline $\widetilde{\ba} = 0$). We  apply LRP to further propagate to the concept $k$. More precisely, we use the standard LRP rule to redistribute   the contribution of the concept $k$ to the sum of activations in each group $j$:
\begin{align}
R_k &= \sum_j \frac{\bh_k^\top (U_{k})_j \boldsymbol{1}_j}{\sum_k \bh_k^\top (U_{k})_j\boldsymbol{1}_j} R_j,
\intertext{where $\boldsymbol{1}_j$ denotes a vector of ones of same size as $\ba_j$, and $(U_{k})_j$ is the $j$th row of the block $U_k$ in the orthogonal matrix $\bbu$ connecting concept $k$ to the group of activations $j$. The relevance score can then be further developed as:}
&= \sum_j \bh_k^\top (U_{k})_j \boldsymbol{1}_j \frac{R_j}{\ba_j^\top\boldsymbol{1}_j}\\
 &= \bh_k^\top U_{k}^\top \Big(\boldsymbol{1}_j \frac{R_j}{(\ba_j^\prime)^\top\boldsymbol{1}_j}\Big)_j\\
&= \big(U_k^\top \ba\big)^\top \big(U_k^\top \bc\big) \label{eq:rk-shapley}
\end{align}
with $\bc_j = \boldsymbol{1}_j \cdot {R_j} / ((\ba_j^\prime)^\top\boldsymbol{1}_j)$. Like for standard Shapley value, the scores $R_k$'s produced by our modified Shapley value formulation depend on input features in an intricate way (here, through the vector $\bc$ which itself depends on the regular Shapley attribution scores $R_j$'s). Similarly to the \intgrad{} case, an inspection of Eq.\ \eqref{eq:rk-shapley} suggests the relevance model:
\begin{align}
\widehat{R}_k &= \big( U_k^\top \ba\big)^\top\big( U_k^\top [\bc]_\text{cst.} \big),
\end{align}
where $[\bc]_\text{cst.}$ is a constant approximation of $\bc$ evaluated at the current data point.

\subsection{Analytical Calculation of Total Relevance}
\label{sec:analytic-relevance-model}
In practice, it can be useful to quickly predict certain properties of the explanation without computing the full explanation. For the Shapley Value and \intgrad{}, the sum of obtained relevances ($\sum_{pk} R_{pk}$) can be calculated without performing the second step of the explanation procedure. In particular, we can show that:
\begin{align}
\sum_{pk} R_{pk}
&= \sum_{pk} [\mathcal{E}(\widehat{R}_k,\bx)]_p \label{eq:shapt-0}\\
&= \sum_{pk} [\mathcal{E}(\ba^\top U_k U_k^\top [\bc]_\text{cst.},\bx)]_p \label{eq:shapt-1}\\
&= \sum_{pk} [\mathcal{E}(\ba,\bx)]_p^\top U_k U_k^\top [\bc]_\text{cst.} \label{eq:shapt-2}\\
&= \Big(\sum_p [\mathcal{E}(\ba,\bx)]_p\Big)^\top \Big(\sum_{k} U_k U_k^\top\Big) \bc \label{eq:shapt-3}\\
&=  [\phi(\x) - \phi(\widetilde{\x})]^\top \Big(\sum_k U_k U_k^\top\Big) \bc, \label{eq:shapt-4}
\end{align}
where we have denoted by $[\mathcal{E}(\ba,\bx)]_p$ the vector containing the attribution of all elements of $\ba$ onto feature $\bx_p$. From \eqref{eq:shapt-1} to \eqref{eq:shapt-2}, we have used the linearity of Shapley values to pull the constant multiplicative factors out of the attribution function. From \eqref{eq:shapt-3} to \eqref{eq:shapt-4}, we have used the conservation property of Shapley values to express the sum of scores forming the explanation as a difference of two function evaluations. Overall, the final formulation of total relevance $\sum_{pk} R_{pk}$ only involves---additionally to the prediction---the computation of the vector $\bc$ and activations associated to the reference point $\widetilde{\x}$.

\section{Proofs and Derivations}

In this note, we provide the proofs of Propositions 1 and 2 of the main paper, and the derivation of the eigenvalue formulation of our PRCA objective.

\setcounter{proposition}{0}

\subsection{Proofs of Propositions}
\label{section:propositions}

Recall that $\ba = (a_j)_j$ is a vector of activations at some layer of the neural network, $R_j$ is the relevance of neuron $j$ for the model output. Recall that $R_j$ decomposes as $R_j = a_j^\prime c_j$ and $\bc = (c_j)_j$. We restate the first proposition of the paper and provide the proof.

\begin{proposition}
Let $\bbu = (U_k)_k$ be an orthogonal matrix formed by $U_k$'s.  Using the formulation of relevance $R_k = (U_k^\top \ba)^\top(U_k^\top \bc)$ with $\bc$ such that $R_j = a_j^\prime c_j$, we have the conservation property $\sum_k R_k = \sum_j R_j$. Furthermore, when $\bc = \xi \ba$ with $\xi \geq 0$, then we necessarily have $R_k\geq 0$.
\label{proposition:conservation}
\end{proposition}

\begin{proof}
We get the conservation result by observing that
\begin{align}
    \textstyle \sum_k R_k
    &= \sum_k (U_k^\top \ba)^\top (U_k^\top \bc) \\
    &= \ba^\top \Big(\sum_k U_k U_k^\top\Big) \bc\\
    &= \ba^\top (\bbu\bbu^\top) \bc\\
    &= \ba^\top \bc\\
    &= \textstyle \sum_j a_j c_j\\
    &= \textstyle \sum_j R_j.
\end{align}
For the positivity property, we first recall the assumption $\bc = \xi \ba$ with $\xi > 0$. Then, we have
\begin{align}
    R_k
    &= (U_k^\top \ba)^\top (U_k^\top \xi \ba)\\
    &= \xi (U_k^\top \ba)^\top (U_k^\top \ba)\\
    &= \xi \|U_k^\top \ba\|^2\\
    &\geq 0.
\end{align}
\end{proof}

\begin{proposition}
When the context vector $\bc$ is equivalent to the activation vector $\ba$, the PRCA analysis reduces to uncentered PCA. Furthermore, if we assumed whitened activations, i.e., $\mathbb{E}[\ba] = \boldsymbol{0}$ and $\mathbb{E}[\ba\ba^\top] = I$, and each matrix $U_k$ projecting to a subspace of dimension $1$, then the DRSA analysis with parameter $q=2$ reduces to ICA with kurtosis as a measure of subspace independence.
\label{proposition:reduction}
\end{proposition}

\begin{proof} We divide the proof of the proposition into two parts: 1) the reduction from PRCA to uncentered PCA; and 2) the reduction from DRSA to ICA.

\smallskip

(Part 1: Reduction from PRCA to PCA) Let $U \in \R^{D\times d}$ and recall that the PRCA objective is to
maximize $\mathbb{E}[(U^\top \ba)^\top (U^\top \bc)]$ w.r.t.\ $U$ subject to $U^\top U = I_{d}$. Setting $\bc=\ba$, the objective can be further developed as:
\begin{align}
     \mathbb{E}[(U^\top \ba)^\top (U^\top \ba)] &= \mathbb{E}[\ba^\top U U^\top \ba] \\
     &= \mathbb{E}[\Tr(\ba^\top U U^\top \ba)] \\
    &= \mathbb{E}[\Tr(U^\top \ba \ba^\top U)] \label{eq:proof1-cyclic}\\
    &= \Tr(U^\top \Sigma U), \label{eq:proof1-sigma}
\end{align}
where \eqref{eq:proof1-cyclic} uses the cyclic permutation property of the trace operator $\Tr(\cdot)$, and where we define $\Sigma = \mathbb{E}[\ba\ba^\top]$ in \eqref{eq:proof1-sigma}. The last line is the canonical formulation for finding the $d$ leading principal components that minimizes the $l_2$ reconstruction error $\mathbb{E}[ \| \ba - U (U^\top \ba)\|^2 ] $. Therefore, it shows that PRCA becomes equivalent to uncentered PCA in this special case. 

\smallskip

(Part 2: Relation between DRSA and ICA)  Let  $\ba_n \in \R^D$ and $\bc_n \in \R^D$ be the activation and context vectors of a data point $n \in \mathcal D$. Note that because we consider the case where each subspace is $1$-dimensional, there are therefore $K=D$ such subspaces.  We denote the collection of these subspaces by $\mathcal K = \{1, \dots, D \}$. We also fix $q=2$. A reduction of the DRSA objective to this setting gives:
\begin{align}
\maximize_{\bbu} \mathbb{M}_{k \in \mathcal K}^{2} &\mathbb{M}_{n \in \mathcal D}^2 [ R_{k,n}^+ (\bbu) ]
\label{eq:proof2-objective}
\end{align}
subject to:
\begin{align}
\bbu^\top \bbu = I_D,
\end{align}
where 
\begin{align}
R_{k,n}^+ (\bbu) = \max(0, (U_k^\top \ba_n)^\top (U_k^\top \bc_n))
\label{eq:proof2-Rkn}
\end{align}
is the rectified relevance on the subspace $k$ of the data point $n$; where the matrix $\bbu = (U_k \in \R^{D\times 1})_{k\in\mathcal{K}}$ is the concatenation of $D$ one-dimensional transformation matrices $U_k$'s; and where $\mathbb{M}^p$ is a generalized F-mean with function $F(t) = t^p$. We now replace the context vector $\bc_n$ in \eqref{eq:proof2-Rkn} with $\ba_n$, which gives us:
\begin{align}
    R_{k,n}^+ (\bbu) &= \max(0, (U_k^\top \ba_n)^\top (U_k^\top \ba_n)) \\
    &= \max( 0, \| U_k^\top \ba_n \|^2) \\ 
    &= \| U_k^\top \ba_n \|^2  \\
    &= (U_k^\top \ba_n)^2,
\end{align}
where the last step follows from the fact that $U_k \in \R^{D \times 1}$. We then inject this expression in \eqref{eq:proof2-objective}, which gives
\begin{align}
    &\mathbb{M}_{k \in \mathcal K}^{2} \mathbb{M}_{n \in \mathcal D}^2 [ (U_k^\top \ba_n)^2 ]\\[1mm]&~=  \mathbb{M}_{k \in \mathcal K}^{2} \bigg[ \sqrt{ \mathbb{E}_{n\in \mathcal D}\big[ \big(U_k^\top \ba_n\big)^4\big] } \bigg]\\[1mm]
     &~=  \sqrt{\mathbb{E}_{k \in  \mathcal K} \mathbb{E}_{n\in \mathcal D} \big[\big(U_k^\top \ba_n \big)^4\big] }.
\end{align}
Optimizing the objective above is therefore equivalent to
\begin{align}
\maximize_{\bbu} \sum_{k\in \mathcal K}  \mathbb{E}_{n\in \mathcal D}\big[ \big(U_k^\top \ba_n \big)^4\big] 
\label{eq:drsa-kurtosis}
\end{align}
subject to $\bbu^\top \bbu = I_D$. Recall the  definition  of kurtosis  for a random variable $Y$ (see Eq. 8.5 in \cite{hyvarinen.icabook.ch8}):
$$
\text{kurt}(Y) = \mathbb{E}[Y^4] - 3(\mathbb{E}[Y^2])^2.
$$
Let $Y_k= U^\top_k \ba_n $ be a random variable associated to the (random) activation vector $\ba_n$. Because $U_k \in \R^{D\times 1}$ and the activation vectors $\ba_n$ are whitened, we have 
\begin{align}
    \mathbb{E}_n[Y_k^2] 
    &= \mathbb{E}_n[(U_k^\top \ba_n \ba_n^\top U_k ]  \\
    &= U^\top_k \mathbb{E}_n[ \ba_n \ba_n^\top  ]  U_k  \\
    &= U_k^\top U_k  \\
    &= 1.
\end{align}
Therefore, the maximization objective becomes  the sum of $\text{kurt}(Y_k)$.
\end{proof}
\begin{remark}
Unlike the setting of Proposition \ref{proposition:reduction} which uses $q=2$ (i.e.\ $\mathbb{M}_{k \in \mathcal K}^{2}$), we use in our DRSA models the parameter $q=0.5$ in order to balance the contribution of each subspace. Such balancing is however not necessary in ICA because the latter is always preceded by a whitening transform.    
\end{remark}

\subsection{Derivation of the PRCA Objective}
\label{sec:derivation-pca-prca-objectives}
\begin{proposition*}
    Suppose $\mathcal D = \{(\ba \in \R^D, \bc \in \R^D)\}$ is a set of  activation-context vector pairs. Let $V \in \R^{D \times d}$ where $d$ is the number of dimensions chosen by the user. Optimizing the objective of PRCA:
    \begin{align*}
        \maximize_{V}~\mathbb{E}_{\mathcal D}[(V^\top \ba)^\top (V^\top \bc\big)]\\
    \text{subject to:}~~V^\top V = I_{d}
    \end{align*}  is equivalent to solving the following eigenvalue problem 
    $$
         \mathbb{E}_{\mathcal D}[\bc\ba^\top + \ba \bc^\top ]  U = U\Lambda,
    $$
    where  $\Lambda \in \R^{d \times d}$ is a diagonal matrix containing the $d$ largest eigenvalues and $U \in \R^{D \times d}$ is the concatenation of the $d$ corresponding  eigenvectors.
\end{proposition*}
\begin{proof}
We divide the proof into three steps: 1) reformulating $(V^\top \ba)^\top (V^\top \bc)$ in terms of trace; 2) constructing a constrained optimization problem using the method of Lagrange multipliers; and 3) finding the critical points of the constructed Lagrangian.

\smallskip

(Step 1): Observing that $(V^\top \ba)^\top (V^\top \bc) \in \R$, we can write it as 
\begin{align}
    (V^\top \ba)^\top (V^\top \bc) &= \ba^\top V V^\top \bc \\
    &= \text{Tr}( \ba^\top V V^\top \bc  ) \\
    &= \text{Tr}( V^\top \bc \ba^\top V),
\end{align}
where the last step uses the fact that trace is invariant under cyclic permutation. 

\smallskip

(Step 2): Because the condition $V^\top V = I_{d}$ induces $d \cdot (d+1)/2$ equality constraints, the Lagrangian of the objective is therefore
\begin{align}
    \mathcal L(V, S ) = \text{Tr}( V^\top \mathbb{E}_{\mathcal D}[\bc \ba^\top] V) - \frac{1}{2} \text{Tr}( (V^\top V - I_{d}) S ),
\end{align}
where $S \in \R^{d \times d}$ is a symmetric matrix of Lagrange multipliers for the $d \cdot (d+1)/2$ equality constraints.

\smallskip

(Step 3): Taking the derivative of $\mathcal L(V, S )$ w.r.t.  $V$ and setting it to zero yields
\begin{align}
\mathbb{E}_{\mathcal D}[\bc \ba^\top + \ba\bc^\top] V = V S.
\label{eq:prca-eigval-problem}
\end{align}
Define $\Sigma = \mathbb{E}_{\mathcal D}[\bc \ba^\top + \ba\bc^\top]$. Because $S=S^\top$, it can be diagonalized. Suppose its diagonalization is $S=E\Lambda E^\top$ where
 $E \in \R^{d\times d}$ is an orthogonal matrix and $\Lambda \in \R^{d\times d}$ is a diagonal matrix.  Right multiplying Eq.\ \eqref{eq:prca-eigval-problem} with $E$ leads to
\begin{align}
   \Sigma VE&= V SE \\
    &= V(E\Lambda E^\top) E \\
    &= VE\Lambda.
\end{align}
Let $U=VE \in \R^{D \times d}$ and note that $U^\top U = (VE)^\top (VE) = E^\top V^\top V E = I_{d}$. We therefore arrive at the eigenvalue problem
\begin{align}
    \Sigma U = U\Lambda,
\end{align}
where each column $U_{:, \tau} \in \R^D$ is the corresponding eigenvector of the $\tau$-th largest eigenvalue $\Lambda_{\tau\tau}$.
\end{proof}

\section{Selection of the Attribution Method}
\label{sec:selection-explanation-backend}

As a starting point to our experiments, we need to select a suitable attribution method. It serves both to extract relevance scores necessary to build subspaces and then to compute disentangled explanations based on the learned subspaces.

\subsection{Evaluation Baselines}

We conduct the evaluation of attribution methods and parameter selection on the ImageNet dataset. We consider Shapley value, Gradient $\times$ Input, Integrated Gradients, and LRP attribution techniques.

For the \textbf{Shapley value}, we use the `Shapley Value Sampling' approximation from the Captum library \cite{DBLP:journals/corr/abs-2009-07896}.  Additionally, we also coarse-grain pixels into disjoint $16\times 16$ pixel patches\footnote{Suppose an input image has $224\times224$ spatial dimensions. Grouping the pixels into $16\times 16$-pixel patches leads to the attribution of only  $14 \times 14$ (coarse-grained) input features, instead of $224 \times 224$.} and perform attribution on these patches, making the attribution for one data point achievable within a reasonable time. We choose the number of permutations in Shapley Value Sampling to be $25$. We use the reference point $\widetilde{\x} = \boldsymbol{0}$ which corresponds to setting removed patches to uniform gray color\footnote{We note that, for input features, the reference point refers to the value after the  input standardization step. That is, if the pixel values are channel-wise standardized by means and standard deviations, the zero reference point corresponds to an input whose pixels are channel means. \label{footnote:input-reference-value}}, and to $\widetilde{\ba} = \boldsymbol{0}$ when applying the method to a function of activations.

For \textbf{Integrated Gradients}, we set the reference points\footref{footnote:input-reference-value} $\widetilde{\x}$ and $ \widetilde{\ba}$ to zero similarly to the Shapley value setting, and choose the linear integration path between the reference points and the actual points. We perform $10$ integration steps when attributing on input features, and $100$ steps when attributing on activations.

For \textbf{Layer-wise Relevance Propagation (LRP)}, we choose  for the VGG16 architecture the heuristics of \cite{bilrp}. Specifically, we use LRP-$\gamma$, where we set $\gamma$ to value $0.5$ in the first two convolution blocks, $0.25$ in the third one, $0.1$ in the fourth one, and $0.0$ in the last convolution block and in the classification head. For the NFNets architecture, we use our novel LRP implementation that utilizes the generalized LRP-$\gamma$ rule  \cite{Keyl2022}. We choose the generalized LRP-$\gamma$ because activations in NFNets can be positive or negative. We use the same value of $\gamma$ for all layers in NFNets and perform parameter selection based on the evaluation metric described in the following. We refer to Supplementary Note \ref{sec:implementing-lrp-nfnet} for the details of our NFNet-LRP implementation.

\subsection{Patch-Flipping Evaluation}
\label{section:pixel-flipping}
Similar to Section 4.1 in the main paper, we evaluate the basic (one-step) attribution techniques (and their parameters) using the patch-flipping method and  the area under the patch-flipping curves (AUPC) \cite{samek-tnnls17}.

\textit{Experimental Setup :} 
We construct a dataset  of $5000$ images in which we randomly select  $5$ validation images of each class in the ImageNet dataset \cite{imagenet_cvpr09}. We compare five attribution methods, namely random attribution, \gradinput, \intgrad, Shapley Value Sampling, and LRP. For computational reasons, we use for Shapley Value Sampling only 10\% of the dataset.  We perform the comparison on three ImageNet-pretrained models used in the main paper (VGG16-TV, VGG16-ND, and NFNet-F0).

Table \ref{table:pixelflipping} shows AUPC scores from different attribution methods and models. From the table, as expected, we first observe that all the methods are substantially better than the random baseline. Secondly, we see that Shapley Value Sampling and LRP performs better than \gradinput\ and \intgrad\ across the three models. Comparing to Shapley Value Sampling, LRP seems to be on par on VGG16-ND but slightly worse on VGG16-TV or better on NFNet-F0. We refer to Fig.\ \ref{fig:pixel-flipping-imagenet-patchsize4} for the corresponding patch-flipping curves.  For NFNet, we performed grid-search on $\gamma \in \{0, 0.001, 0.01, 0.1, 1.0\}$ and found that $\gamma=0.1$ yields a satisfying AUPC of $3.64$. A slightly better AUPC of $3.54$ could be obtained for $\gamma=0.01$ but with visually noisier (and less interpretable) heatmaps (cf.\ Fig.\ {\ref{fig:lrp-final-heatmaps}}).

\begin{table}[t!]
\centering
\caption{Area under the Patch-Flipping Curve (AUPC) of  different attribution methods for different models. We average the obtained score over $5000$ random validation images from the ImageNet dataset \cite{imagenet_cvpr09} for all methods, except Shapley Value Sampling, where we use only 10\% of the images for computational reasons. The largest error bars of Shapley Value Sampling and other methods are $\pm 0.13$ and $\pm 0.06$ respectively. We perform patch-flipping over patches of size $16\times16$.  We show the best AUPC score of each model in bold. 
}
\label{table:pixelflipping}
\begin{tabular}{l|ccc}\toprule
 & \parbox{1.5cm}{VGG16-TV}
 & \parbox{1.5cm}{VGG16-ND}
 & \parbox{1.5cm}{NFNet-F0}\\\midrule

% Table Meta Information
% notebook=viz-pixel-flipping.ipynb@table_flipping_all_for_patch_size
% artifact-dir=../artifacts/2023-12-revision-v1.20.0/pixel-flipping
% Generated at 2023-12-18 16:34:55.074895
% patch_size=16
%        torchvision-vgg16-imagenet & netdissect-vgg16-imagenet & dm_nfnet_f0
Random                         & 8.32 & 7.85 & 5.42  \\
\gradinput            & 7.08 & 6.24 & 5.16  \\
\intgrad                     & 6.18 & 6.23 & 4.88  \\
Shapley Value Sampling               & \textbf{4.91} & \textbf{4.64} & 3.77  \\
LRP                            & 5.78 & 4.91 & \textbf{3.64}  \\
%Largest stderr(not svs)=0.0540 and stdeer(svs)=0.1326
%patch_size=16
% end
\bottomrule
\end{tabular}
\end{table}

\begin{figure}
    \centering
    \includegraphics[clip,trim={15cm 0 3cm 0.5cm},width=.32\textwidth]{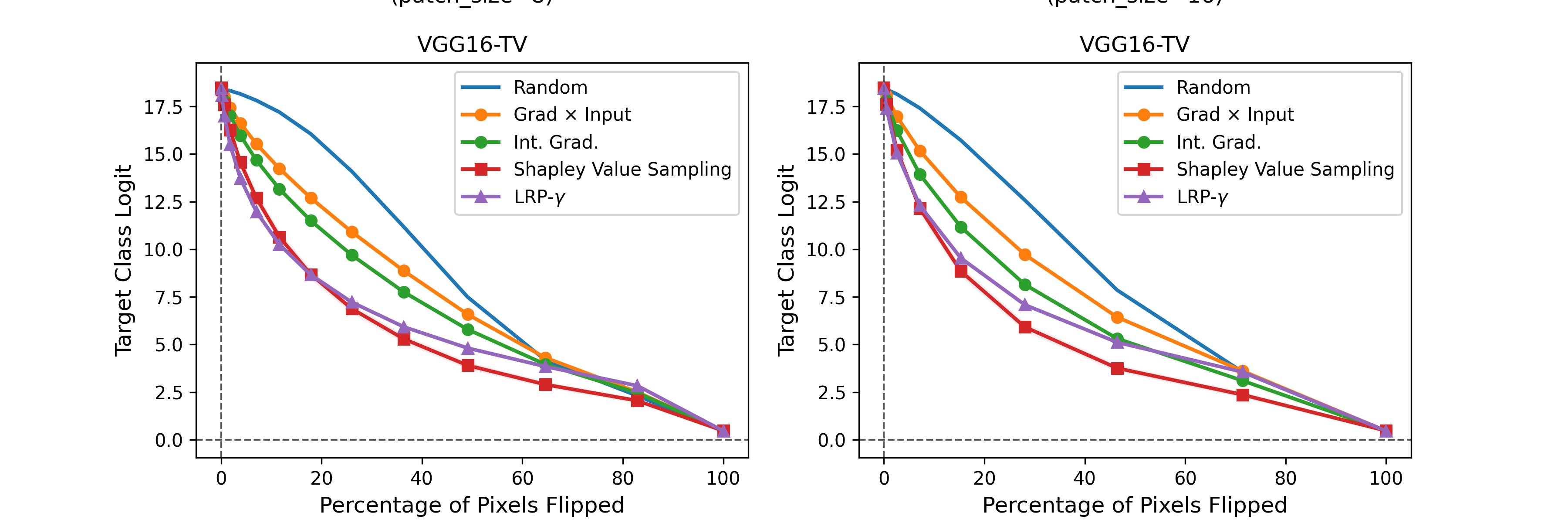} \hfill
    \includegraphics[clip,trim={15cm 0 3cm 0.5cm},width=0.32\textwidth]{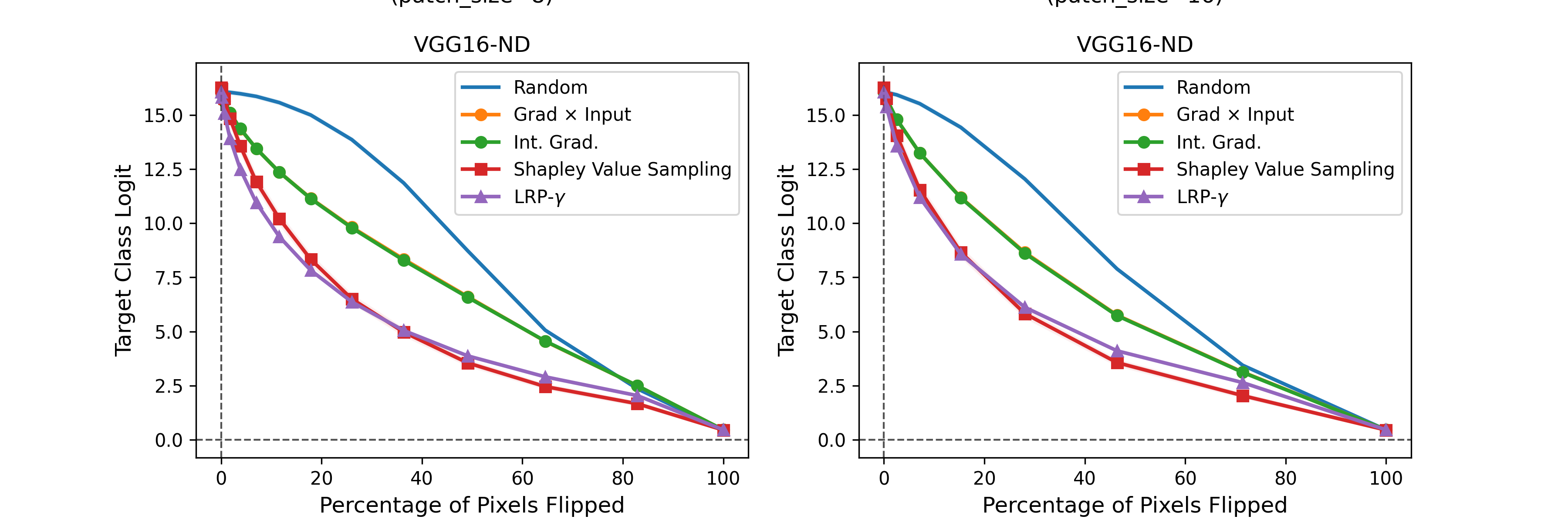} \hfill
    \includegraphics[clip,trim={15cm 0 3cm 0.5cm},width=0.32\textwidth]{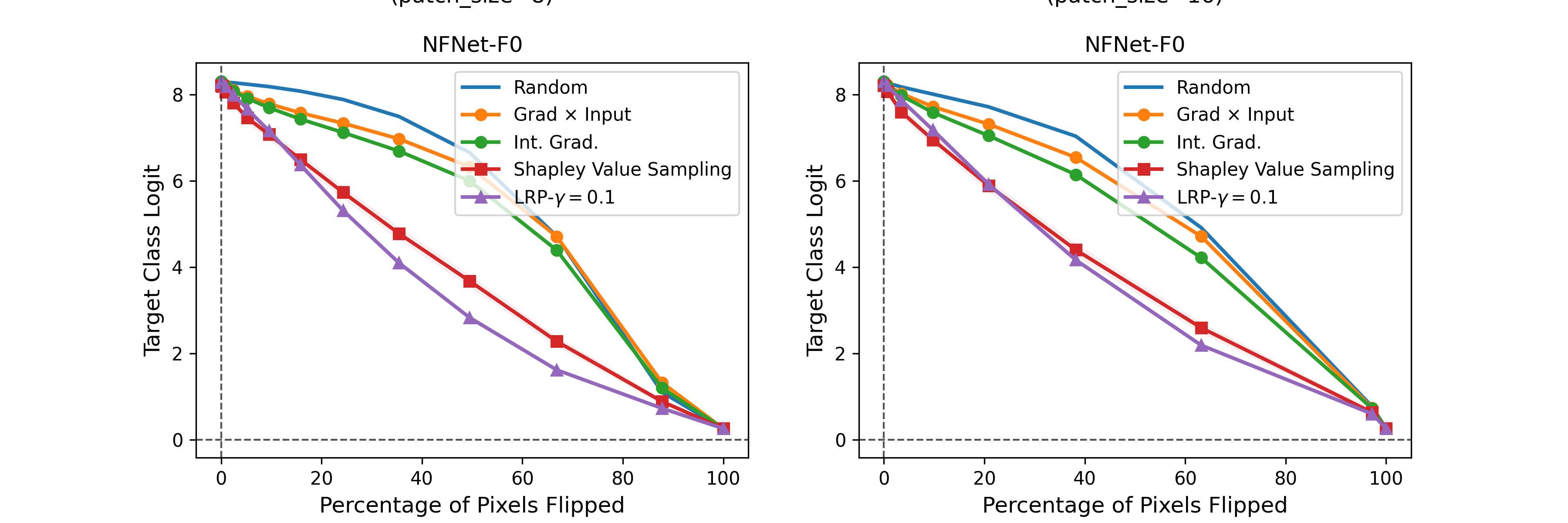}
    \caption{Patch-flipping curves from different attribution methods and models.  We average the scores  from $5000$ random validation images from the ImageNet dataset \cite{imagenet_cvpr09} for all methods, except Shapley Value Sampling, where we use only 10\% of these images for computational reasons. We perform patch-flipping over patches of size $16\times16$.}
    \label{fig:pixel-flipping-imagenet-patchsize4}
\end{figure}

\begin{figure}
    \centering
    \includegraphics[width=0.8\textwidth]{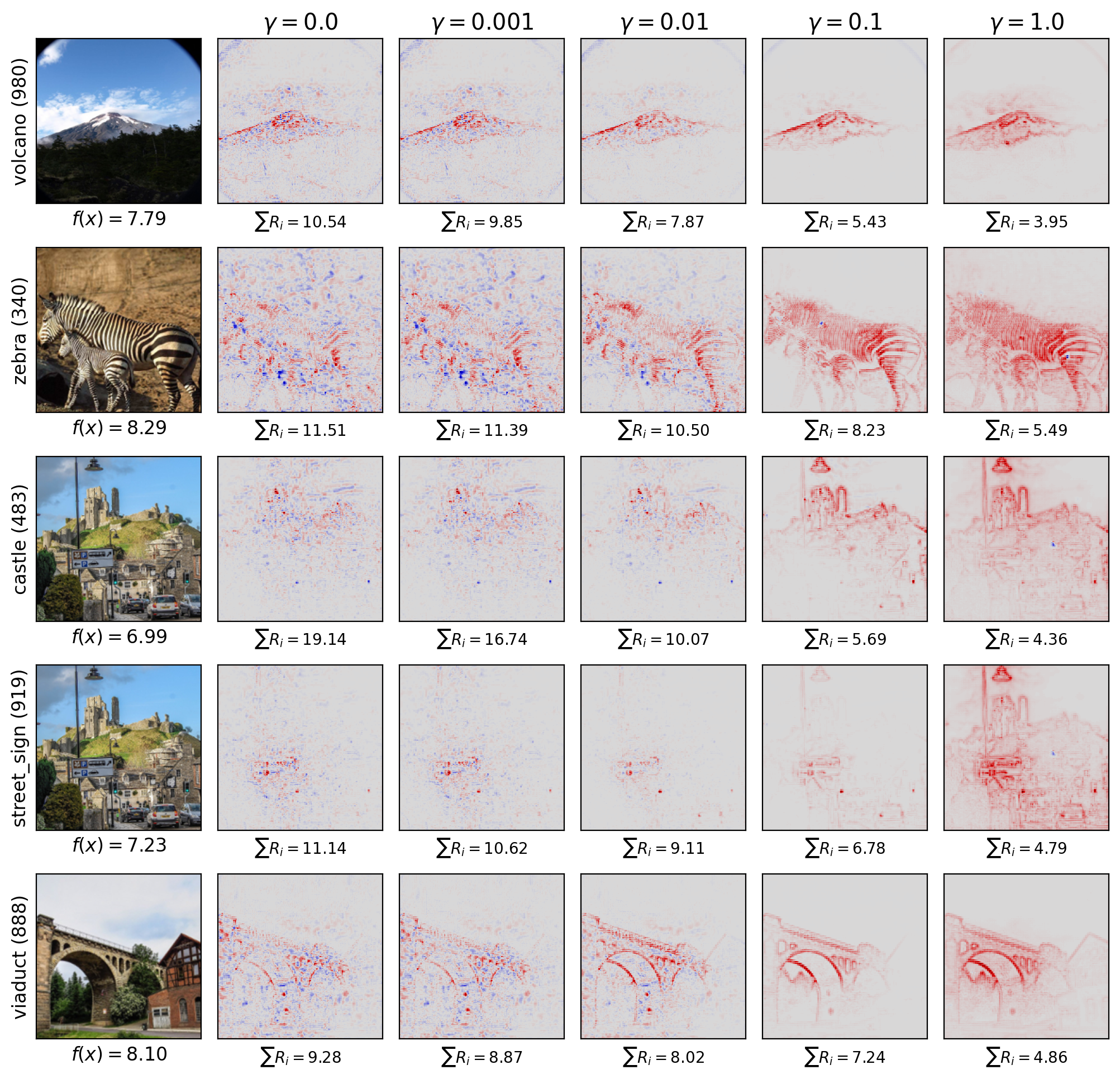}
    \caption{\lrpgamma{} heatmaps of different input images for NFNet-F0 and with different values of $\gamma$.}
    \label{fig:lrp-final-heatmaps}
\end{figure}

\subsection{Computational Efficiency of Attribution Methods}
\label{sec:remark-on-computationale-efficiency}
Computational efficiency of the explanation method is an important aspect in practice. Gradient$\,\times\,$Input and LRP perform favorably, with both approaches requiring only one forward/backward pass in the network. LRP comes with a small additional cost due to operationalizing LRP rules (e.g.\ via forward hooks). Although the cost is implementation-dependent, we find empirically that it does not exceed an order of magnitude of the original computation. In comparison, the cost of Integrated Gradient grows linearly with the number of integration steps, and that of Shapley Value Sampling linearly with the number of input features and sampled permutations.

We note that, unlike a typical (i.e.\ one-step) attribution scenario, the two-step explanation we consider in our paper generates not a single but $K$ explanations per data point. In other words, the overall runtime increases for all methods by a linear factor $K$ compared to the typical setup, making the computational efficiency an important criterion when selecting the underlying attribution method.

\section{Training DRSA and DSA}
\label{sec:training-irca}
\subsection{Preprocessing and Optimization Parameters}
Let $\mathcal X$ be a set of randomly selected training images of the class of interest with $|\mathcal X|=N$.  We take activation (and context) vectors  at $n$ random spatial locations from each of these $N$ training images. We generate the context vectors w.r.t.\ the logit of the class using a chosen attribution method. Denote $\mathcal A = \{ (\ba \in \R^D, \bc \in \R^D) \}$ to be the set of these activation and context vectors pairs. Suppose $i \in \{1, \dots, | \mathcal A |  \}$ and $j \in \{1, \dots, D \}$. We  optimize DRSA on $\hat {\mathcal  A} =\{ (\hat \ba, \hat \bc ) \}$ where
\begin{align}
    \hat{\ba} = \frac{1}{\sqrt[4]{D}} \frac{\ba}{\sqrt{\mathbb E _ {i,j} [a_{ij}^2 ]}}, \quad \hat{\bc} =  \frac{1}{\sqrt[4]{D}} \frac{\bc}{ \sqrt{\mathbb E_{i,j} [ c_{ij}^2 ]} }.
\end{align}
We found that the normalization helps stabilize the optimization process. We initialize the training of DRSA with a random $D\times D$-orthogonal matrix\footnote{We generate such an orthogonal matrix via the `ortho\_group' module from SciPy \cite{2020SciPy-NMeth}.}, which we partition into  $K$ blocks according to the numbers of dimensions $d_k$'s chosen by the user. As mentioned in the main paper, each optimization iteration contains two steps, namely batch gradient ascent and 2) orthogonalization.  Because the objective of DRSA (Eq.\ 10 in the main paper) is non-convex, we perform $\tau$  runs using different orthogonal matrices. Among these $\tau$ runs, we select the run  that achieves the highest objective value to be  the solution of the optimization. We sort the blocks $U_k$'s of the solution according to $\mathbb{E}_{\hat{\mathcal D}}[ ( U_k^\top \hat{\ba})^\top (  U_k^\top \hat{\bc} )]$ in descending order and form the final orthogonal matrix $\bbu$ accordingly.

%Table \ref{tab:training-irca} summarizes the values of all optimization parameters used in our experiments.
We select for $N=500$ examples activation vectors at $n=20$ different spatial locations. The optimization procedures are run for $5000$ iterations, and we perform $\tau=3$ runs, retaining the best solution. With these parameters, the optimization time of DSA/DRSA is approximately 10  and 60 minutes on Nvidia Quadro RTX 5000 for activations with $D=512$ and $1536$ respectively; the former is the case of VGG16 at Conv4\_3, while the latter is of NFNet-F0 at Stage 2.  Fig.\ \ref{fig:irca-optimization-convergence} shows the training curves of the optimization from the three models: VGG16-TV, VGG16-ND, and NFNet-F0.

\begin{figure*}[h]
    \begin{minipage}{.7\textwidth}
    \centering
    \includegraphics[width=\textwidth]{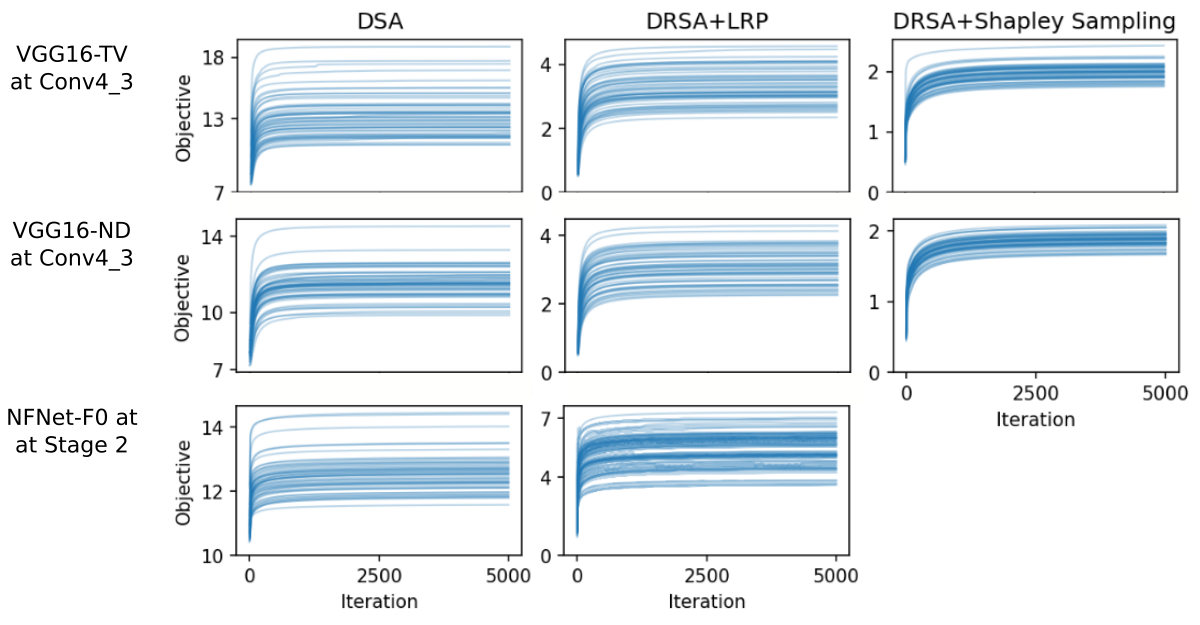}
    \end{minipage}\hfill
    \begin{minipage}{.25\textwidth}
    \caption{Training curves of the DSA and DRSA optimization across different models and attribution methods. In each plot, each curve corresponds to one of the 50 ImageNet classes used in the main experiments (see Section 5 of the main paper). 
    }
    \label{fig:irca-optimization-convergence}
    \end{minipage}
\end{figure*}

\subsection{Procedure for Selecting Subspace Prototypes}
\label{sec:algo-prototype}
We develop a procedure to select a set of prototypical images for visualizing what semantic features DRSA subspaces represent. For example, we use the procedure to generate Fig.\ 1 in the main paper. Objectively, we design the procedure such that all subspaces are approximately equally expressed. We achieve this goal by utilizing the objective of DRSA (Eq.\ 10 in the main paper).

Let $\bbu= (U_1|\dots |U_k| \dots | U_K)$ be the learned orthogonal matrix from  DRSA for a given class. Let   $n$ be the desired number of prototypes and $N$ be the number of random candidate subsets.  Denote $\mathcal X$ to be a set of some images from the class. The procedure goes as follows:
\begin{enumerate}
    \item We construct $N$ random candidate subsets, each containing $n$ images from $\mathcal X$;
    \item We compute, for each candidate subset, the DRSA objective using the activation and context vectors of the $n$ images in the subset;
    \item We take the subset with the largest objective to be the set of prototypical images.
\end{enumerate}
We use $N=1000$.

\section{Ablation Studies}
\label{sec:ablation-studies}

\subsection{Ablation on PRCA and DRSA Formulations}
The goal of PRCA is to find a $d$-dimensional subspace that is maximally relevant for the prediction.  As discussed in Section  4.1 of the main paper, the goal is equivalent to find a matrix $U \in \R^{D \times d}$ that defines a projection onto such a subspace and takes the most relevance, i.e.\ $\maximize_U \mathbb{E}[R]$ with $R = (U^\top \ba)^\top (U^\top \bc)$, and subject to $U^\top U = I_{d}$. As stated in the main paper and shown  in Supplementary Note \ref{sec:derivation-pca-prca-objectives}, the solution to the optimization is the $d$ eigenvectors associated to the largest eigenvalues of the symmetrized cross-covariance matrix  $\mathbb{E}[\ba\bc^\top + \bc\ba^\top]$. In this ablation analysis, we aim to verify that the solution of PRCA indeed preserves the most relevance. We quantify the property through the `\textit{Total Relevance}' score of  the input features $(\boldsymbol{x}_p)_{p=1}^P$: 
\begin{align}
     \text{TotalRelevance}(U)  = \mathbb{E} \Big[ \sum_{p=1}^{P} [\mathcal{E}(R, \x)]_p \Big].
     \label{eq:rtot}
 \end{align}
We remark that, for the Shapley value, we can compute this score  efficiently via $\text{TotalRelevance}(U) = \mathbb{E}\big[ (\phi(\x) - \phi(\widetilde \x))^\top U U^\top \bc \big]$, where $\phi$ is the function mapping the input $\x$ to the activation vector $\ba$, where $\widetilde \x = \boldsymbol{0}$, and where $\bc$ is the context vector computed from attributing the neural network output $f(\x)$ onto $\ba$. We refer to Supplementary Note \ref{sec:analytic-relevance-model} for the derivation. 
We compare PRCA with three ablations of its formulation:
\begin{itemize}
    \item Ablation 1: we replace the context vector $\bc$ in the objective function with the activation vector $\ba$; the optimization problem leads to the formulation of (uncentered) PCA (cf.\ Proposition \ref{proposition:reduction})
    \item Ablation 2: we construct  $U$ from only standard basis vectors, i.e.\ the matrix $U$ has only one non-zero entry in each row and column.
    \item Ablation 3: we use the first $d$ columns of a random orthogonal matrix, i.e.\ no training involved
\end{itemize}

\begin{table}
\centering
\caption{
Total relevance score (Eq.\ \eqref{eq:rtot}) for PRCA and three ablations. Results are shown for different
combinations of models, datasets, and underlying attribution techniques (columns).  We highlight for each column the best subspace method (highest total relevance score) in bold.  Results are averaged over 50 classes of
the ImageNet dataset or 7 classes of Places365. ($\dagger$) average from three seeds.
}
\label{table:maxcontribution}
    \begin{tabular}{lcccccc}
    \toprule

     & \multicolumn{5}{c}{ImageNet} & \multicolumn{1}{c}{Places365} \\
 
    \cmidrule(lr){2-6} \cmidrule(l){7-7}
    
     & \rotatebox{90}{\parbox{1.5cm}{VGG16-TV\\+ LRP}}
     & \rotatebox{90}{\parbox{1.5cm}{VGG16-ND\\+ LRP}}
     & \rotatebox{90}{\parbox{1.5cm}{NFNet-F0\\+ LRP}}
     & \rotatebox{90}{\parbox{1.5cm}{VGG16-TV\\+ Shapley}}
     & \rotatebox{90}{\parbox{1.5cm}{VGG16-ND\\+ Shapley}}
     & \rotatebox{90}{\parbox{1.5cm}{ResNet18\\+ LRP}}\\
     
     \midrule

%% Version: 2023-12-22 16:55:03.406430
%% artifact-dir=../artifacts/2023-12-revision-v1.20.0/raw-main-experiment

\em No subspace projection ($U=I_D)$  &  11.47 &  10.35 &   6.57 &  17.22 &  16.59 &   1.19 \\ \midrule 

% > prca-ns1-ss1
              PRCA &  \textbf{13.63} &  \textbf{13.72} &  \textbf{11.29} &  \textbf{44.69} &  \textbf{42.26} &   \textbf{1.39}   \\[1mm]

% > pca-ns1-ss1
                      Ablation 1: $\bc \gets \ba$ &   1.81 &   2.69 &  -2.22 &  21.81 &  18.99 &   1.04   \\

% > max-rel-ns1-ss1
         Ablation 2: $U$ standard basis  &   0.97 &   0.87 &   0.30 &   1.11 &   0.92 &   0.14   \\

% > random*-ns1-ss1
          Ablation 3$^\dagger$: $U$ random &   0.02 &   0.00 &   0.01 &   0.02 &   0.02 &   0.00   \\[1mm]

         % largest stderr &   0.66 &   0.61 &   0.58 &   1.76 &   1.41 &   0.12

\textit{Error bars (max)} & $\pm$ 0.66 & $\pm$ 0.61 & $\pm$ 0.58 & $\pm$ 1.76 & $\pm$ 1.41 & $\pm$ 0.12 \\
%% debug information
%% end table %%

\bottomrule
\end{tabular}
\end{table}

\noindent The experimental setup is similar to Section 5.1 in the main paper. From Table \ref{table:maxcontribution}, we observe that, as to be expected, PRCA has the highest total relevance score comparing to the three ablations. This observation confirms that two properties of PRCA: 1) \textit{maximizing} relevance, and 2) doing so over \textit{any} orthogonal projection, are both important in order to concisely capture the relevant part of the decision strategy.

\begin{figure}
    \centering 
    \begin{minipage}{.45\textwidth}
        \includegraphics[clip, trim={0.2cm 0 0cm 1.2cm}, width=\linewidth]{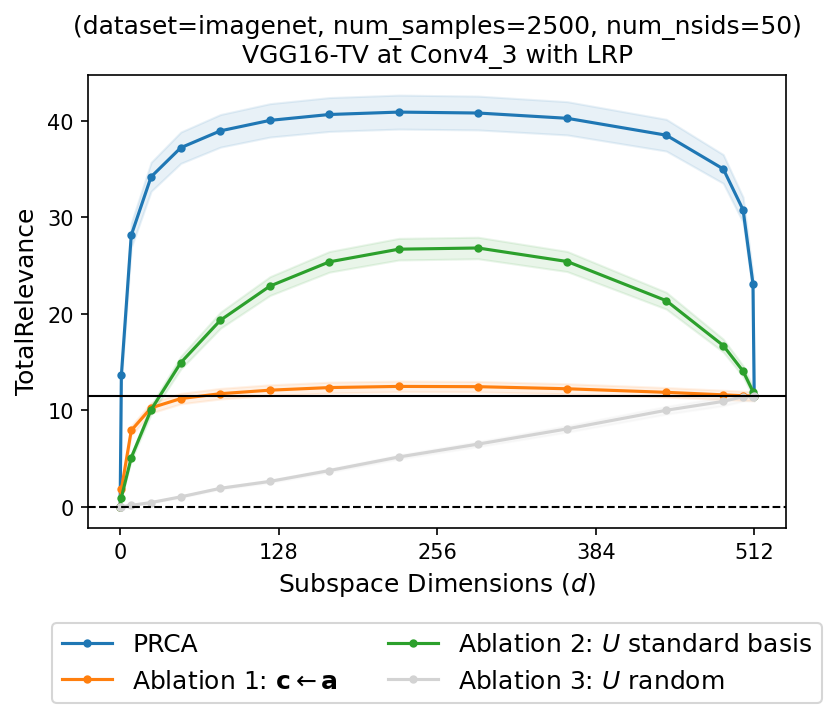}
    \end{minipage}
    \begin{minipage}{.53\textwidth}
        \caption{
        Total relevance of PRCA and three ablations when varying the subspace dimensionality (the variable $d$); higher is better. The analysis is performed on VGG16-TV with LRP
(same as column 1 in Table \ref{table:maxcontribution}). Each curve is an average over
the means of these 50 classes, and shaded regions represent
one standard error (over classes). The horizontal solid line
represents the total relevance of no subspace projection.
        }
        \label{fig:maxcontribution-varydim}
    \end{minipage}
\end{figure}

As a further experiment, we analyze the total relevance score  as a function of  the subspace dimensions $d$. We perform the experiment using  the VGG16-TV model. From  \mbox{Fig.\ \ref{fig:maxcontribution-varydim}}, we observe that PRCA is superior to the three ablations for every subspace size. In particularly, PRCA is able to 1) extract in the top-few principal components a large amount of positive evidence for the output neuron and 2) strongly suppresses negative contributions.

Together, the results of these experiments therefore substantiate that PRCA  indeed finds the subspace that is maximally relevant.

\medskip

 The second question we have considered in this paper (and for which we have proposed the DRSA analysis) is whether the explanation can be disentangled into  semantic components that contribute to the model's decision strategy. In Section 4.2 of the main paper, we have translated the question into the problem of finding an orthogonal matrix $\bbu = (U_1 | \dots | U_k | \dots | U_K)$ that partitions the $D$-dimensional activation space into $K$ subspaces.

Because our focus is on the case of CNNs---in that semantic patterns in the input are spatially separate---such a matrix $\bbu$  disentangles explanation into spatially non-overlapping components. The goal of this ablation study is therefore to verify the non-overlapping property of explanation components directly at the level of joint pixel-concept relevance scores (not at the change of the model's output like the AUPC score). 

To quantify the property, we propose \textit{separability} and \textit{peakness} scores of explanation components, which we define
\begin{align}
    \text{Separability}(\bbu) &= \mathbb{E} \Big[\sum_{p=1}^{P} \max_{k}  \big\{  R_{pk}  \big\} - \max_{k} \Big\{\sum_{p=1}^{P}  R_{pk} \Big\} \Big],
     \label{eq:metric-separability}\\
    \text{Peakness}(\bbu) &= \mathbb{E} \bigg[ \sum_{k=1}^K \bigg(\max_p \{ R_{pk} \} \bigg) \bigg ].
\end{align}

A low separability score occurs, for example, when only a single-component explanation (i.e.\ a standard explanation) is available or when all components of the explanation are the same. Conversely, the separability score is high when the contributions associated to different components correspond to different  input features; in other words, these components are spatially separated. A high peakness occurs when the components of the explanation focus strongly on distinct aspects of the decision strategy.

In the following, we consider DRSA and two ablations of its formulation, namely

\begin{itemize}
    \item Ablation 1: we substitute the context vector $\bc$ in the objective of  DRSA with the activation vectors $\ba$; this ablation is considered in the main paper and called DSA.
    \item Ablation 2: we use a random orthogonal matrix.
\end{itemize}

The experimental setup is similar to Section 5.2 in the main paper. Table \ref{table:maxseparation} shows the separability and peakness scores across setups. From the table, we observe that DRSA has the highest separability and peakness scores. This result reflects the visual inspection of \mbox{Fig.\ 1}  in the main paper, where we can identify distinct concepts  from the DRSA explanations. Furthermore, with different choice  of layers or  number of subspaces, the observation from Table \ref{table:maxseparation} still applies. We discuss these additional experiments in Supplementary Note \ref{section:ablation-different-layers-and-num-subspaces}.

\begin{table}
    \centering
    \caption{
     Separability and peakness scores of subspaces $\bbu$ extracted by DRSA and two ablations  (rows). Results are shown for the same dataset/model/attribution settings (columns) as in Table \ref{table:maxcontribution}. We highlight for each column the best method
(highest scores) in bold. Results are averaged over 50 classes of the ImageNet dataset or 7 classes of Places365. ($\dagger$) average from three seeds.
    }
    \label{table:maxseparation}
\begin{tabular}{lcccccc}
    \toprule

     & \multicolumn{5}{c}{ImageNet} & \multicolumn{1}{c}{Places365} \\
 
    \cmidrule(lr){2-6} \cmidrule(l){7-7}
    
     & \rotatebox{90}{\parbox{1.5cm}{VGG16-TV\\+ LRP}}
     & \rotatebox{90}{\parbox{1.5cm}{VGG16-ND\\+ LRP}}
     & \rotatebox{90}{\parbox{1.5cm}{NFNet-F0\\+ LRP}}
     & \rotatebox{90}{\parbox{1.5cm}{VGG16-TV\\+ Shapley}}
     & \rotatebox{90}{\parbox{1.5cm}{VGG16-ND\\+ Shapley}}
     & \rotatebox{90}{\parbox{1.5cm}{ResNet18\\+ LRP}}\\
    
    \midrule

%% Version: 2024-01-05 17:00:15.911739
%% artifact-dir=../artifacts/2023-12-revision-v1.20.0/raw-main-experiment

\multicolumn{1}{l}{}\\[-1mm]
\multicolumn{1}{l}{Separability} \\\midrule
% > learnt--drsa-ns4-ss[SS]-sm2-seed1
              DRSA  & \textbf{7.2360} & \textbf{6.6688} & \textbf{1.9419} & \textbf{12.1734} & \textbf{8.7456} & \textbf{0.2204}   \\[1mm]
%     stderr [0.4749 0.4173 0.1122 0.6662 0.4823 0.0216]

% > learnt--dsa-ns4-ss[SS]-sm2-seed1
                      Ablation 1: $\bc \gets \ba$ (DSA) & 5.0253 & 4.9317 & 1.1329 & 7.4581 & 5.7264 & 0.1187   \\
%     stderr [0.2629 0.2551 0.0789 0.345  0.2275 0.0139]

% > random*-ns4-ss[SS]
          Ablation 2$^\dagger$: $\bbu$ random & 2.1638 & 2.1438 & 0.1381 & 6.7845 & 5.4216 & 0.0396   \\[1mm]
%     stderr [0.0679 0.065  0.0041 0.2554 0.1706 0.0022]

% > largest stderr=0.6662
\textit{Error bars (max)} & $\pm$ 0.4749 & $\pm$ 0.4173 & $\pm$ 0.1122 & $\pm$ 0.6662 & $\pm$ 0.4823  & $\pm$ 0.0216 \\\midrule

\multicolumn{1}{l}{}\\[-1mm]
\multicolumn{1}{l}{Peakness} \\\midrule
% > learnt--drsa-ns4-ss[SS]-sm2-seed1
              DRSA  & \textbf{0.0524} & \textbf{0.0420} & \textbf{0.0337} & \textbf{0.0072} & \textbf{0.0055} & \textbf{0.0014}   \\[1mm]
%     stderr [0.0036 0.0027 0.0032 0.0004 0.0002 0.0001]

% > learnt--dsa-ns4-ss[SS]-sm2-seed1
                      Ablation 1: $\bc \gets \ba$ (DSA) & 0.0367 & 0.0320 & 0.0222 & 0.0043 & 0.0035 & 0.0011   \\
%     stderr [0.0023 0.0022 0.0019 0.0002 0.0002 0.0001]

% > random*-ns4-ss[SS]
            Ablation 2$^\dagger$: $\bbu$ random & 0.0259 & 0.0231 & 0.0140 & 0.0035 & 0.0030 & 0.0009   \\[1mm]
%     stderr [0.0019 0.0017 0.0011 0.0002 0.0002 0.0001]

% > largest stderr=0.0688
\textit{Error bars (max)} & $\pm$ 0.0019 & $\pm$ 0.0017 & $\pm$ 0.0011 & $\pm$ 0.0002 & $\pm$ 0.0002 & $\pm$ 0.0001\\

    \bottomrule
    \end{tabular}
\end{table}

\subsection{Choosing Different Layers or Number of Subspaces}

\label{section:ablation-different-layers-and-num-subspaces}

We investigate the effect of  the two important parameters in our evaluations between DRSA and other baselines, namely the choice of layers and  the number of subspaces. Because DSA is  the strongest baseline from the evaluation in Section 5.2 of the main paper, we compare DRSA with it on different values of these parameters.

We perform the  experiments with LRP-$\gamma$ and the VGG16-TV and VGG16-ND models with the 50 classes of the ImageNet dataset, similar to 
the main experiments 
(cf.\ Section
5.2
of the main paper
). 
We report the \textit{Area Under the Patch-Flipping Curve (AUPC)}, and the \textit{Separability} and \textit{Peakness} scores defined above. We remark that, for AUPC, a lower score is better, while, for separability and peakness, a higher score is better

\subsubsection{Comparison between DSA and DRSA at Different Layers}
\label{section:ablation-different-layers}
We investigate whether the difference in AUPC, separability and peakness scores of DSA and DRSA remains consistent when choosing different layers.
We consider three different layers of the VGG16 architecture, namely Conv3\_3, Conv4\_3, and Conv5\_3. We fix the number of subspaces to $K=4$.

\textit{Results : }
Fig.\ \ref{fig:comparing-irca-ica-across-layers} shows the AUPC, separability, and peakness scores  of DSA and DRSA at the three different layers of the VGG16 models and across 50 ImageNet classes.
We observe that  DRSA generally has better scores than DSA. The results  indicate that the superiority of DRSA is not  due to the choice
of layers.

\begin{figure}[t!]
    \centering
    \begin{subfigure}{\textwidth}
          \includegraphics[width=\textwidth,clip,trim=0 27.7cm 0 0] {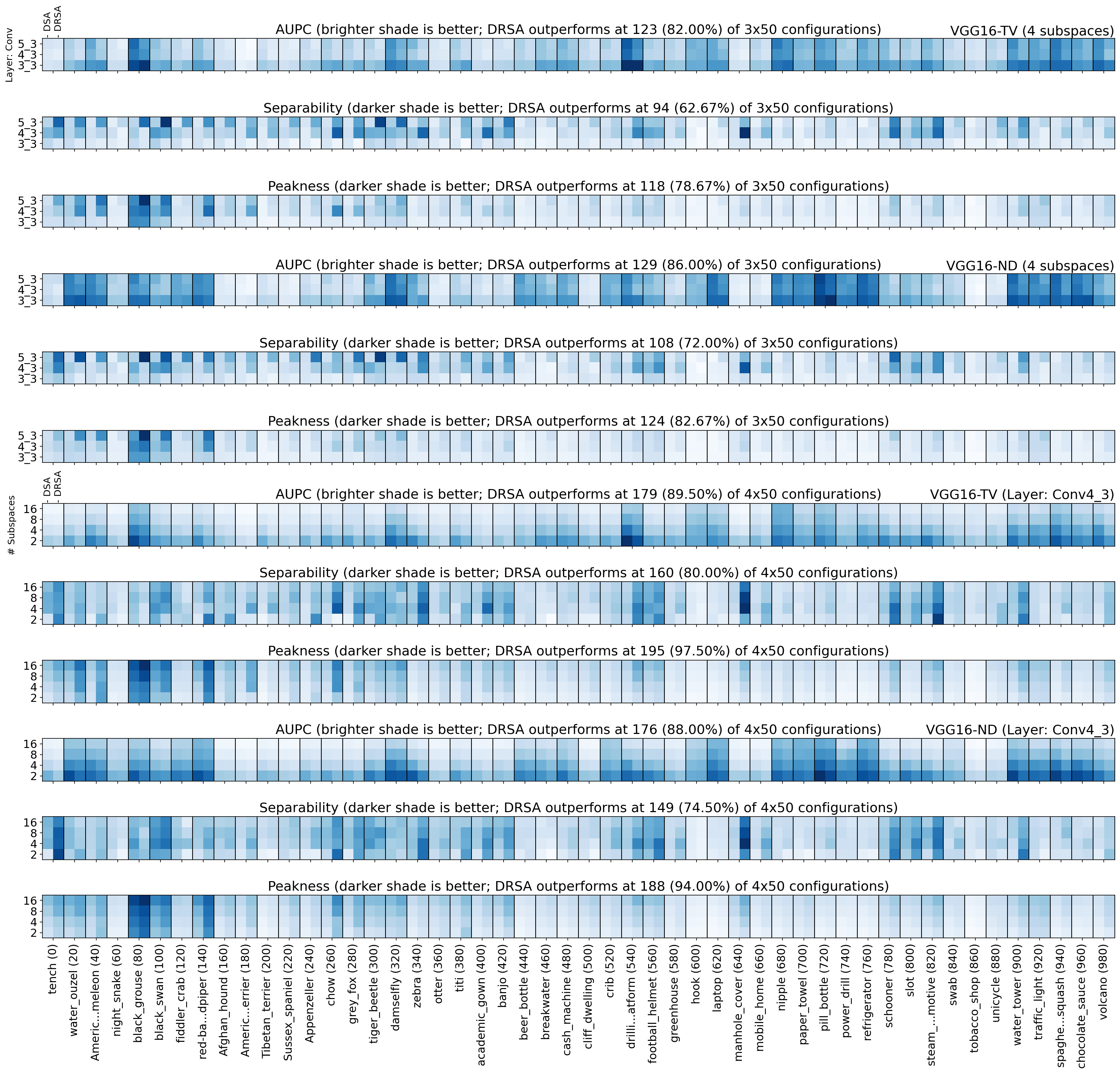}
            \caption{}
            \label{fig:comparing-irca-ica-across-layers}
    \end{subfigure}
    \begin{subfigure}{\textwidth} 
          \includegraphics[width=\textwidth,clip,trim=0 0 0 21.35cm] {figures/ablation/comparison_between_different_subspace_sizes_and_layers.png}
        \caption{}
        \label{fig:comparing-irca-ica-across-subspaces}
    \end{subfigure}
        \caption{Area Under the Patch-flipping Curve (AUPC), separability and peakness scores of subspaces $\bbu$ extracted by DSA and DRSA at (a) different layers with 4 subspaces or (b) Conv4\_3 with different numbers of subspaces.}
    
\end{figure}

\subsubsection{Comparison between DSA and DRSA on Different Numbers of Subspaces}
We investigate whether the difference in AUPC, separability and peakness scores of DSA and DRSA remains consistent when varying the number of subspaces $K$. We consider $K \in \{2, 4, 8, 16\}$ and fix the layer of interest to be Conv4\_3.

Fig.\ \ref{fig:comparing-irca-ica-across-subspaces} shows the AUPC, separability, and peakness scores from disentangled explanations from different collections of subspaces.  For the majority of configurations, we observe again that DRSA yields higher scores than DSA.  This result suggests that the superiority of DRSA is not due to  the number of subspaces.

\section{Class-wise AUPC Score Analysis}
\label{section:class-wise-analysis}
We conduct an additional analysis on the experiment results presented in Section 5.2 of the main paper. Our goal is to  investigate the level of explanation disentanglement across different classes. Because the outputs of different classes  are often in  different scales, to account for such an effect, we use the following class statistics
\begin{align}
    \Delta \text{AUPC}(\bbu) = \text{AUPC}(I_{D, K=1}) - \text{AUPC}(\bbu)
\end{align}
for comparing the disentanglement level of different classes; the higher the $\Delta$AUPC score, the more disentangled explanation components the class has.

\textit{Results : } Fig.\ \ref{fig:class-wise-analysis} shows the $\Delta$AUPC scores across different classes from VGG16-TV. We find that the top three classes are class `zebra`, `drilling platform`, and `steam locomotive`. Fig.\ \ref{fig:class-wise-analysis-heatmap} shows the DSA and DRSA explanation components of  three validation images from class zebra. From the figure, we observe that DRSA decomposes  the prediction strategy of the class zebra into four sub-strategies, namely the detection of the zebra body, the legs, the top part of the body, and other features.
\begin{figure}[t!]
    \centering
    \includegraphics[width=\textwidth]{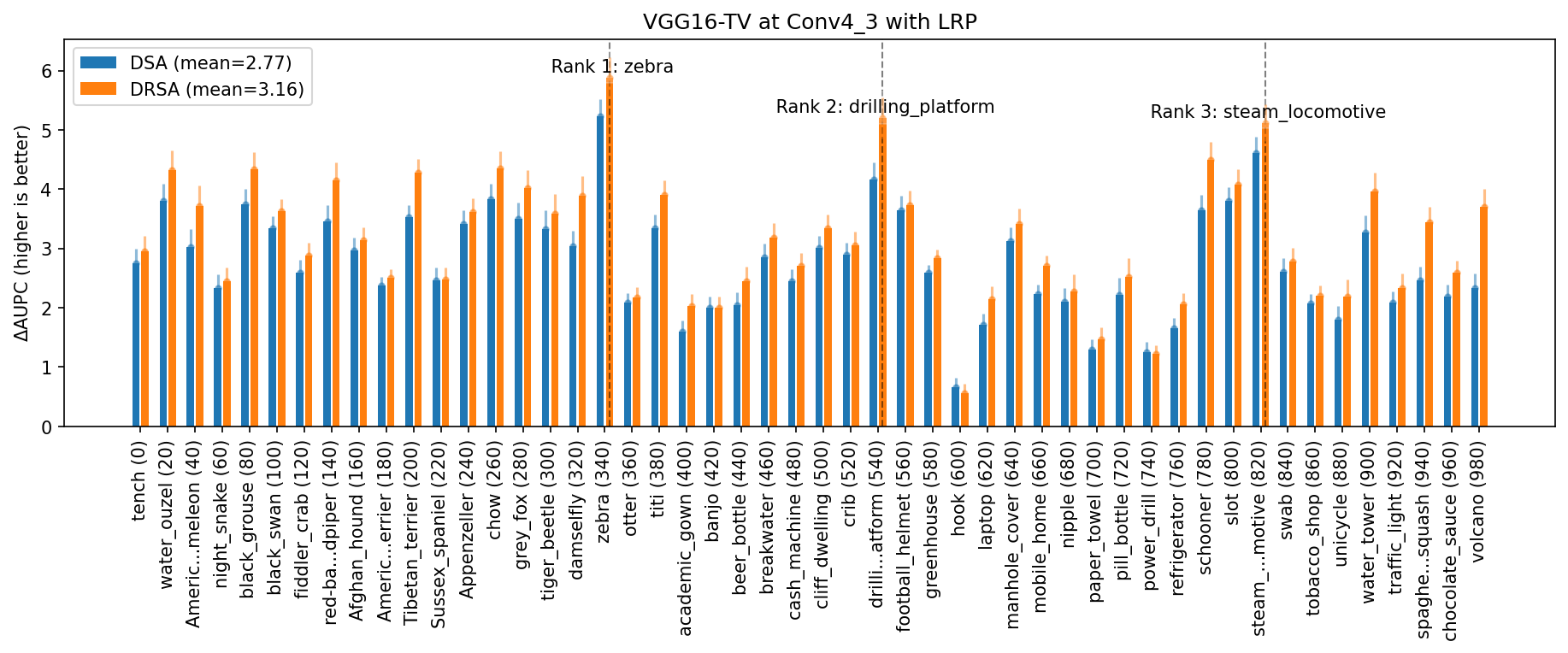}
    \caption{Per class $\Delta$ Area Under the Patch-Flipping Curves ($\Delta$AUPC) of disentangled explanations produced by DSA and DRSA for VGG16-TV using LRP; higher is better.}
    \label{fig:class-wise-analysis}

    \bigskip
    
    \includegraphics[height=0.201\textheight,clip,trim=0 12cm 14.7cm 1cm]{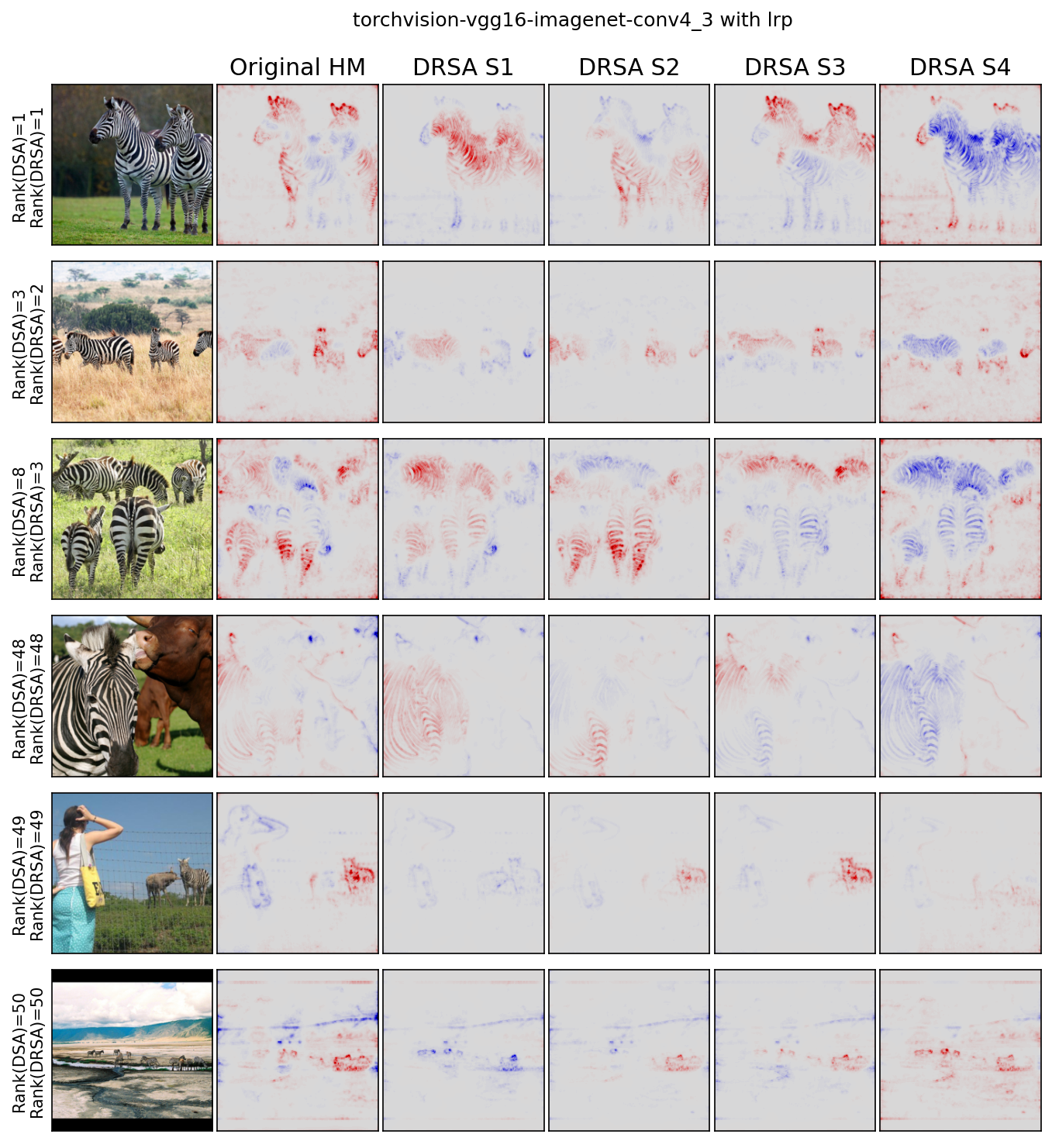}
    \hfill
    \includegraphics[height=0.201\textheight,clip,trim=8.3cm 12cm 0.2cm 1cm]{figures/analysis-zebra/heatmaps-drsa.png}
    \hfill
    \includegraphics[height=0.201\textheight,clip,trim=8.3cm 12cm 0 1cm]{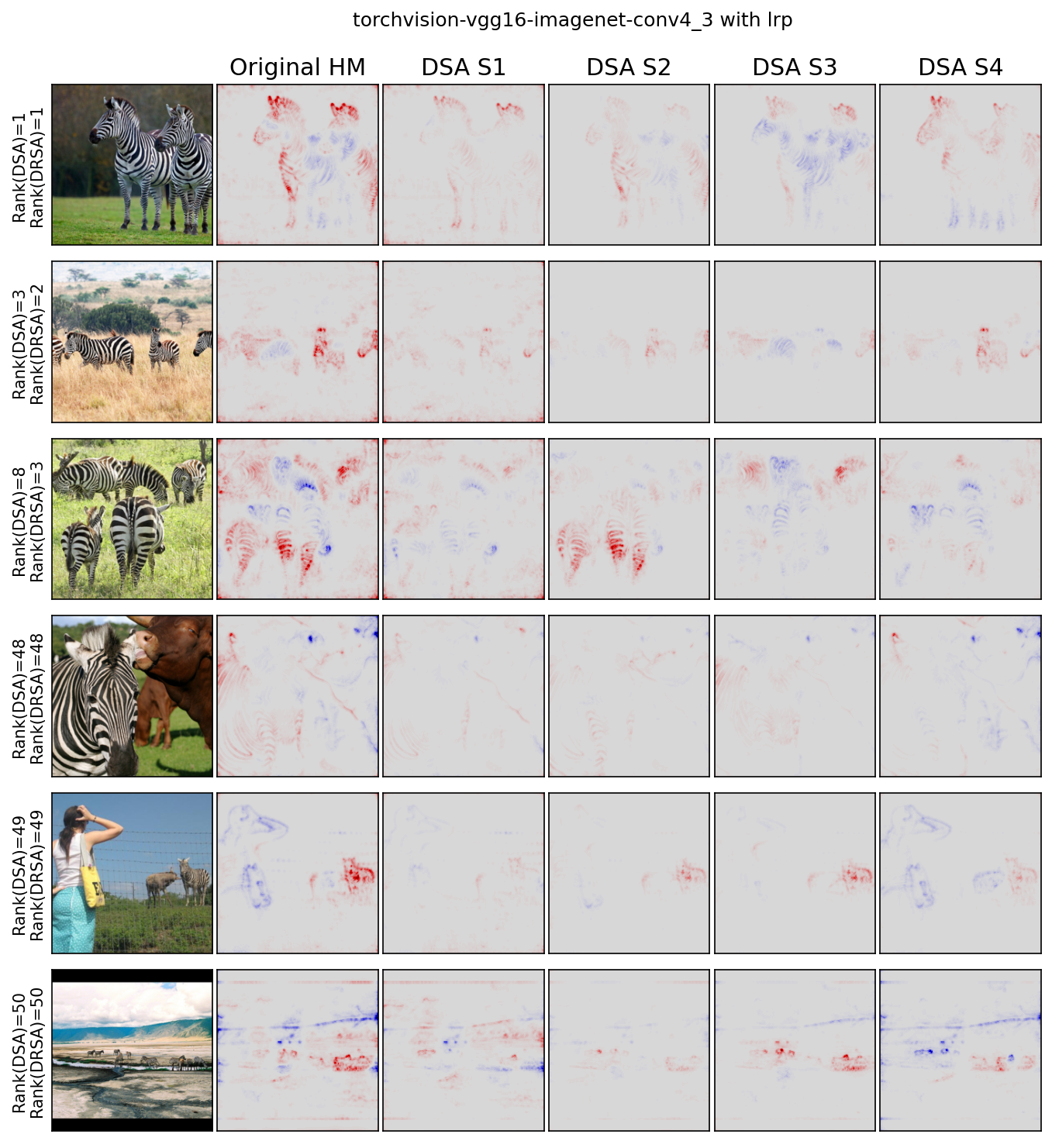}
    \caption{Qualitative comparison between heatmaps  produced from DSA and DRSA subspaces on VGG16-TV at Conv4\_3 with the LRP backend. Red color indicates pixels that contribute evidence for class `zebra`, while blue color indicates pixels that speak against it.  These images are from the validation set of the ImageNet dataset, and the ranking is based on the $\Delta$AUPC scores of DSA and DRSA and relative to all validation images in the class.
    }
    \label{fig:class-wise-analysis-heatmap}
\end{figure}

\section{Details about NetDissect}

\label{sec:baselines-netdissect}

We briefly describe the NetDissect method which we use in our benchmark evaluations in the main paper. NetDissect \cite{DBLP:journals/pami/ZhouBO019} is a framework that associates high-level concepts to units (filters in a given layer) in neural networks. The framework constructs a set of concepts $\mathcal K$ from the semantic categories of the  Broden dataset \cite{BauZKO017}. Let $\mathcal J = \{1, \dots, D\}$ be the set of  units in a given layer.  The framework performs three steps to associate a unit $j \in \mathcal J$ with a concept $k \in \mathcal K$:
\begin{enumerate}
    \item \textit{Gathering Activation Maps:}  images from the Broden dataset $\{ \x \in \R^{3 \times h \times w} \}$ are fed to the model. Their unit $j$ activation maps  $\{ A_j(\x) \in \R^{h' \times w' }\}$ are gathered to determine the 99.5-th percentile $\tau_j$ of the unit $j$'s overall response, i.e.\ $\mathbb{P}(a_j > \tau_j) = 0.005.$

    \medskip
    
    \item \textit{Producing Binary Response Mask:} each activation map $A_j(\x)$ is  resized to the  spatial dimensions of the input
    $S_j(\x) = \texttt{upsample}(A_j(\x)) \in \R^{h'\times w'},$
    and then binarized with the percentile $\tau_k$ to produce a  response mask, i.e.\ $M_j(\x) = {\indicator{S_j(\x) > \tau_j} } \in \{0, 1\}^{h \times w}$.

    \medskip
    
    \item \textit{Quantifying Alignment between `Concept $k$` and `Unit $j$`:} Let $\mathcal D_{k} \subset \mathcal D$ be the subset of Broden images whose pixels are annotated with the concept $k$. Denote the concept $k$ annotation mask of each image $\x$ as $L_k(\x) \in \{0, 1\}^{h \times w}$. NetDissect quantifies the alignment between the concept $k$ and unit $j$  using the Intersection over Union (IoU) ratio: 
    \begin{align}
    \text{IoU}(j, k) = \frac{\sum_{\x \in \mathcal D_k} | M_j(\x) \cap L_k(\x) | }{\sum_{\x \in \mathcal D_k} | M_j(\x) \cup L_k(\x) | } .
    \end{align}
    If the  criteria $\text{IoU}(j, k) > \alpha$  is satisfied\footnote{\cite{DBLP:journals/pami/ZhouBO019} uses $\alpha=0.04$.}, the concept $k$ is  added to the  unit $j$'s  concept set $\mathcal K_j \subseteq \mathcal K$.
    NetDissect then assigns the concept with the largest IoU score to be the concept of the unit $j$, i.e.
    \begin{align}
    \texttt{ConceptOf}(j) \leftarrow \argmax_{k   \in \mathcal K_j   } \text{IoU}(j, k). \label{eq:netdissect-step3}
    \end{align}
\end{enumerate}

We adapt `NetDissect-Lite`\footnote{ \url{https://github.com/CSAILVision/NetDissect-Lite}} (provided by  \cite{DBLP:journals/pami/ZhouBO019}) to dissect two publicly available ImageNet-pretrained VGG16 \cite{DBLP:journals/corr/SimonyanZ14a} models. These two models are from TorchVision \cite{torchvision2016} (VGG16-TV) and 
     NetDissect's model  repository\footnote{The model is available at \url{http://netdissect.csail.mit.edu/}. We note that  the model is available in the Caffe \cite{jia2014caffe} format, and we have converted it to the PyTorch \cite{Paszke_PyTorch_An_Imperative_2019} format.
    Our reproduction is on the PyTorch model.
    } (VGG16-ND).    
We reproduce NetDissect results at two layers, namely Conv4\_3 and Conv5\_3. We refer to our extended code repository\footnote{\url{https://github.com/p16i/NetDissect-Lite/wiki}} for the technical details of the reproduction. Fig.\ \ref{fig:netdissect-vgg16} shows the distribution of concepts assigned to filters in the two layers of VGG16-TV and -ND. We see that the concept distributions of the two models generally agree.  In particular, we observe that  lower layers tend to capture low-level concepts (e.g.\ `part' and  `texture'), while high-level layers capture high-level semantics (e.g.\ `object' and `scene'); our observations are similar to what was discussed in  \cite{DBLP:journals/pami/ZhouBO019}. For VGG16-ND specifically, our result is also similar to what \cite{DBLP:journals/pami/ZhouBO019} reports\footnote{\url{http://netdissect.csail.mit.edu/dissect/vgg16_imagenet/}}, indicating no difference between the original Caffe model and our PyTorch converted version. Complementary to Fig.\ \ref{fig:netdissect-vgg16}, Table \ref{tab:netdissect-unique-concept-counts} shows the  number of unique detected concepts in each concept category.

\begin{figure*}[t!]
    \centering
    \includegraphics[clip,trim=0 0 0cm 0,width=0.8\textwidth]{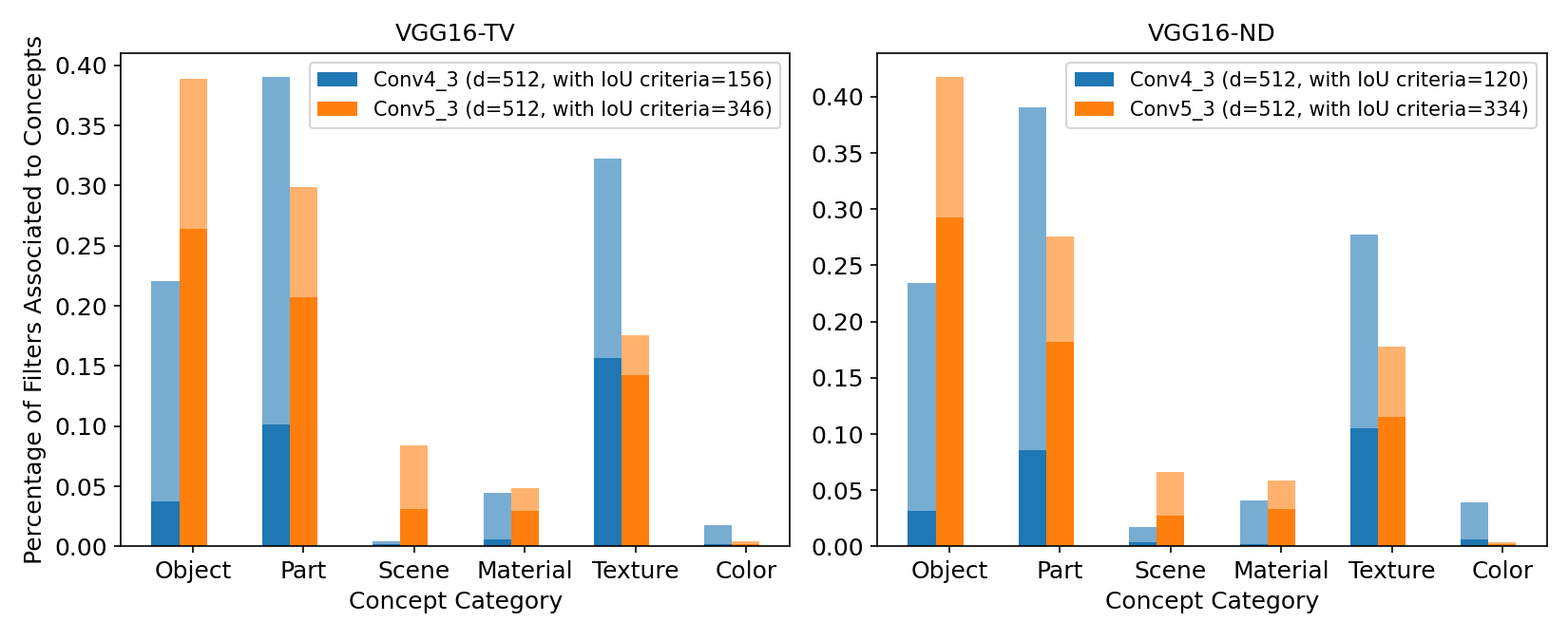}
    \caption{Percentage of filters from two layers of  VGG16-TV and -ND that   NetDissect  associates them to concepts from the Broden dataset. Area with light and dark shade indicate the percentages without and with  the $\text{IoU}$ criteria respectively. }
    \label{fig:netdissect-vgg16}
\end{figure*}

\begin{table*}[t!]
    \centering
    \caption{Number of unique concepts that NetDissect detects from filters in two layers of VGG16-TV and VGG16-ND. The numbers outside and inside parentheses are without and with the $\text{IoU} > \alpha$ criteria respectively.
    }
    \label{tab:netdissect-unique-concept-counts}
    \begin{tabular}{lcccccccc}
    \toprule
     &   & & & & Concept Category & & &  \\
       \cmidrule(lr){3-8} 
    Model & Layer   & Object & Part & Scene & Material & Texture & Color & Total  \\
     \midrule
     
% printed at: 2022-03-24 10:42:33.209521 
VGG16-TV & Conv4\_3 & 35 (10) & 34 (15) & 2 (1) & 8 (2) & 31 (22) & 4 (1) & 114 (51)   \\
 & Conv5\_3 & 47 (31) & 27 (24) & 31 (11) & 8 (5) & 29 (24) & 2 (1) & 144 (96)   \\[1mm]
 
VGG16-ND & Conv4\_3 & 35 (10) & 36 (14) & 6 (1) & 9 (1) & 26 (20) & 6 (1) & 118 (47)   \\
 & Conv5\_3 & 52 (38) & 25 (20) & 24 (10) & 8 (5) & 33 (28) & 2 (1) & 144 (102)   \\

\midrule
% printed at: 2022-03-24 10:42:33.268397  
     
    \end{tabular}
\end{table*}

    As a remark, for the experiments in Section 5.2 of the main paper,  we do not use the $\text{IoU} > \alpha$ criteria to construct subspaces from  NetDissect (i.e.\ $\mathcal{K}_j = \mathcal K$). As a result, each filter is associated to a concept.
    Furthermore, because the  assignment procedure of  NetDissect yields unequal numbers of filters per concept, we rank the concepts based on $R_k / N_k$ where  $R_k$ is the concept $k$ relevance  of a data point and $N_k$ is the number of filters corresponding to the concept.

\section{Details about Interpretable Basis Decomposition (IBD)}
\label{sec:ibd}

We briefly outline the IBD method which we use in our benchmark evaluations in the main paper.
IBD \cite{DBLP:conf/eccv/ZhouSBT18} is a framework that decomposes  the decision of the neural network into a linear combination of concept vectors. Suppose that there exists 1) a set of concepts $\mathcal K = \{ k \}$
and 2) concept vectors $\bu_k \in \R^D$ corresponding to each concept $k$.
Let $\w_t \in \R^D$ be the weight vector (in the last layer of the neural network) corresponding to the class $t$. IBD decomposes the weight vector as
\begin{align}
    \w_t  = \sum_{k \in \mathcal K_t} \alpha_t^k \bu_k  + \boldsymbol{r}_t,
    \label{eq:ibd-weight-decomp}
\end{align}
where $\mathcal K_t \subset \mathcal K$ is a set of class-compatible concepts, $\alpha^k_t \in \R$ is a class-concept coefficient,
and  $\boldsymbol{r}_t \in \R^D$ is a residue vector. 
By linearity, the output of the neural network for the class $t$ is decomposed into the contribution of concepts (via concept vectors)
\begin{align}
    \ba^\top \w_t = \bigg( \sum_{k \in \mathcal K_t} \alpha^k_t (\ba^\top \bu_k) \bigg) + \ba^\top \boldsymbol{r}_t.
\end{align}

In practice, IBD constructs the set of concepts $\mathcal K$ from the Broden dataset \cite{BauZKO017}. Each concept vector $\bu_k$ is based on the weight vector of a classifier that is trained to determine the presence of the concept $k$.
IBD uses a greedy algorithm to determine the  set of class-compatible concepts $\mathcal K_t$ and the class-concept coefficients  $\{ \alpha_t^k \}_{k \in \mathcal K_t}$.

To integrate IBD in our evaluation, we use the sets of class-compatible concepts $\mathcal K_t$'s and concept vectors $\bu_k$'s that are provided by the authors of IBD\footnote{\url{https://github.com/CSAILVision/IBD}}.

We now describe how we construct a virtual layer from IBD concept vectors. We assume that 1) we have the set of class-compatible concepts $\mathcal K_t$ and 2) the concept vectors $\bu_k$  are linearly independent. Let $U  \in \R^{D \times |\mathcal K_t|}$ be the matrix with the concept vectors in columns and $U^+=(U^\top U)^{-1} U^\top$ be its left-pseudo inverse matrix. Denote $\ba, \bc \in \R^D$ to be an activation vector and its context vector. To address the non-orthogonality of the concept vectors, we adapt our formulation of the virtual layer accordingly by expressing
\begin{align}
    \ba' &= UU^+ \ba + \ba^\perp \\
    &= (U^+)^\top U^\top \ba + \ba^\perp,
\end{align}
where $\ba^\perp$ is a residue vector. The concept relevance is
\begin{align}
    R_k = (\bu^\top \ba) (\bu^+ \bc),
\end{align}
where the row vector $\bu^+_k \in \R^{1 \times D}$ is the corresponding row in the left pseudo-inverse matrix $U^+$. We note that this adaptation of the virtual layer has no guarantee on `positivity' (cf.\ Proposition \ref{proposition:conservation}).

\section{Additional Details for Showcases}

\subsection{Showcase 1: Detecting and Mitigating Clever Hans Effects}
\label{sec:showcase-1-additional}

In this note, we provide additional information for Section 6.1 of the main paper. We first briefly outline the Hanzi Watermark Clever Hans strategy used by VGG16-TV for predicting class `carton'. We then describe the construction of the synthetic classification task and poisoning.  Lastly, we present additional experimental results for clean and 50\%-poisoned data.

\subsubsection{Hanzi Watermark Clever Hans Strategy}

Reference \cite{DBLP:journals/inffus/AndersWNSML22} shows that there are a number of classes in the ImageNet dataset \cite{imagenet_cvpr09} that  training images contain Hanzi watermarks. One of such class is class `carton`. To illustrate, Fig.\ \ref{fig:overview-ch} (i) shows three training images from class carton and their corresponding Hanzi watermarks. In general, these watermarks appear at the center of the image. \cite{DBLP:journals/inffus/AndersWNSML22} discusses that ML models can develop a  `Clever Hans' strategy by  making the prediction of class carton simply using features extracted from the collection of these watermarks.    In addition to Hanzi watermarks,  a domain name or timestamp also often appears in the bottom-right corner of carton images \cite{DBLP:journals/inffus/AndersWNSML22}; the three images in Fig.\ \ref{fig:overview-ch} (i) also have such a timestamp. Although ML models could develop `Clever Hans' strategies from these domain-name and timestamp features, it is unlikely because the features  lie at the location that is discarded by the common center-cropping pre-processing step.

 Using SpRAy \cite{Lapuschkin2019}, \cite{DBLP:journals/inffus/AndersWNSML22} identifies that  the ImageNet-pretrained VGG16 \cite{DBLP:journals/corr/SimonyanZ14a} from the PyTorch model repository, i.e.\ VGG16-TV, has such a Clever Hans strategy, exploiting features from Hanzi watermarks to make prediction for class `carton'.

\subsubsection{Synthetic Task and Poisoning}
\label{sup:showcase1-synthesis}
Our goal is to demonstrate that we can fool VGG16-TV with Hanzi watermarks, making it classify non-carton images as carton. To achieve this, we construct a  synthetic task of (1+$M$)-class classification: class `carton' and $M$ other classes. We choose these $M$ classes to be the classes that the percentage of their validation images containing class `carton' in the top-3 predicted classes is larger than $10\%$. We show these $M$ classes of VGG16-TV in Fig.\ \ref{fig:overview-ch} (ii).

\begin{figure}[t!]
    \centering
    \includegraphics[width=0.9\textwidth]{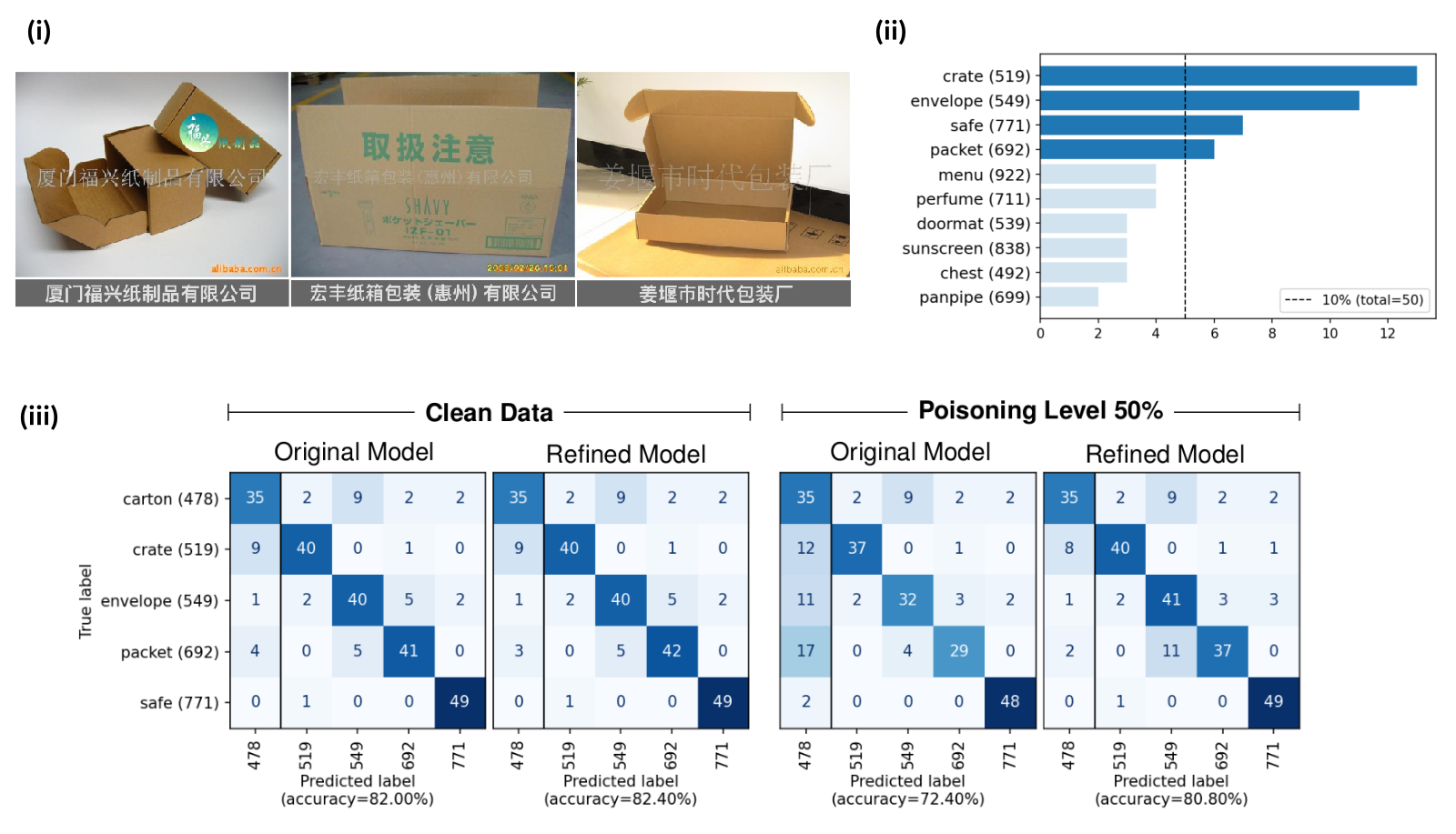}
    \caption{(i) Three training images from Class Carton and their corresponding Hanzi watermarks. (ii) ImageNet classes that VGG16-TV often confuses for `carton'. The horizontal axis is the number of validation images for which VGG16-TV prediction includes `carton' in the top-3. (iii) Confusion matrices from the original and refined VGG16-TV models on clean and 50\%-poisoned data. The refined model weakens the prediction of class `carton' using the excess relevance from the subspace S4 (Eq.\ 12 in the main paper).}
    \label{fig:overview-ch}
\end{figure}

Our poisoning procedure aims to increase the chance that VGG16-TV predicts non-carton images as `carton'. We achieve the goal by overlaying Hanzi watermarks on a number of non-carton images.  Let $\tau$ poisoning rate parameter (a percentage value) and $N$ be the number of validation images in each non-carton class. Our poisoning procedure is as follows:
\begin{enumerate}
    \item We randomly select $\tau\%$ of non-carton validation images;
    \item For each image, we select a random watermark (from the three extracted watermarks shown in Fig.\ \ref{fig:overview-ch} (i)) and overlay the watermark on the image with opacity $0.5$.
\end{enumerate}

\subsubsection{Confusion Matrices from Clean and 50\%-Poisoned Data}

Fig \ref{fig:overview-ch} (iii) shows confusion matrices on the clean and  50\%-poisoned data from the original VGG16-TV and its refined version. The refined model adjusts the prediction of class `carton' with the excess relevance of the subspace S4 (Eq.\ 12 in the main paper), which captures the relevance of the Hanzi watermarks (see \mbox{Fig.\ 7} in the main paper). Here, we observe that the refined model performs as good as the original model on the clean data, while it is substantially more accurate on the 50\%-poisoned data than the original model. More importantly, the performance difference between the two model is larger than what observed in the 25\%-poisoned data (see Fig.\ 8 in the main paper). The results thus assure that removing the excess evidence of S4 mitigates the Hanzi Clever Hans strategy of class `carton' from VGG16-TV.

\subsubsection{Comparison with Spectral Relevance Analysis in Detecting Hanzi Watermark Strategy}
\label{sup:sec-spray}
Spectral Relevance Analysis (SpRAy) \cite{Lapuschkin2019} is an analysis that clusters data points w.r.t. their explanations.
With the clustering structure, the user can then gain more comprehensive understanding on what or how the model makes the prediction of each cluster of data points.
\cite{Lapuschkin2019} shows that SpRAy can effectively uncover unintentional strategies, e.g., leveraging spurious correlation or Clever Hans features, of machine learning models.  The goal of this experiment is to compare the ability of our DRSA approach and the SpRAy baseline in detecting Clever Hans effects.

\textit{Experimental Setup for SpRAy}: We extract the heatmaps of carton training images (the same images that we use to train DRSA). We post-process these heatmaps with the sum-pooling of size $8 \times 8$,  rectification, and  the $\ell_2$ normalization.  To find clusters of these heatmaps, we use the $K$-Mean clustering algorithm with $K=4$ (instead of spectral clustering as in the original work \cite{Lapuschkin2019}), for its robustness and the possibility to easily derive test statistics; we use the implementation of the algorithm from \cite{scikit-learn}. Specifically, to quantify whether SpRAy can detect data points with the Hanzi watermark, we use the distance to each cluster as the test statistics.

\textit{Experimental Setup for DRSA}: As shown in Fig.\ 7 of the main paper, the heatmaps of the subspace S4 from DRSA highlight prominently the Hanzi watermark. We therefore quantify whether our DRSA approach can detect data points with the Hanzi watermark by using the rectified input relevance of S4.

\textit{Results}: We randomly take 100 training images of class carton that are not used in training DRSA and SpRAy. We manually annotate them according to whether they have the Hanzi watermark; in total, there are 42 images having the watermark. The annotation is then the target label for the detection task. We then quantify the detection capability of DRSA and SpRAy using the receiver operating characteristic (ROC) curve. Figure \ref{fig:spray} (top) shows  three carton training images that are closest to  each SpRAy cluster center. In particular, we observe that the prototypes of  Cluster 2 and 4 all contain the Hanzi watermark. Figure \ref{fig:spray} (bottom) shows the ROC curves of DRSA and the four SpRAy cluster centers.  We see that the ROC curve of DRSA is superior to the curves of SpRAy. This comparison suggests that DRSA is better than SpRAy in detecting spurious features. 

\begin{figure*}
    \centering
    \includegraphics[height=0.175\textheight]{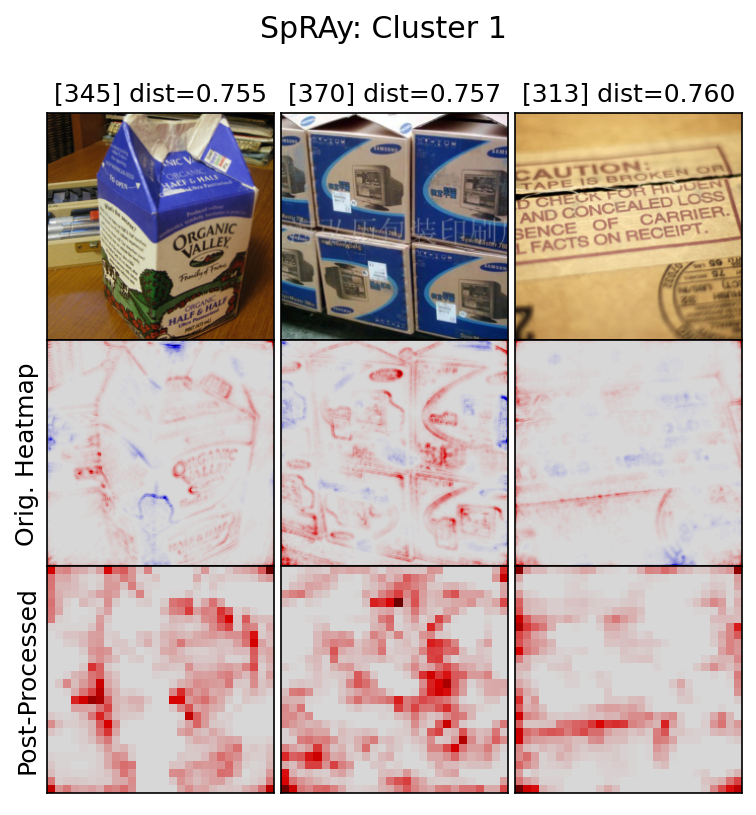}
    \includegraphics[height=0.175\textheight,clip,trim=0.75cm 0 0 0]{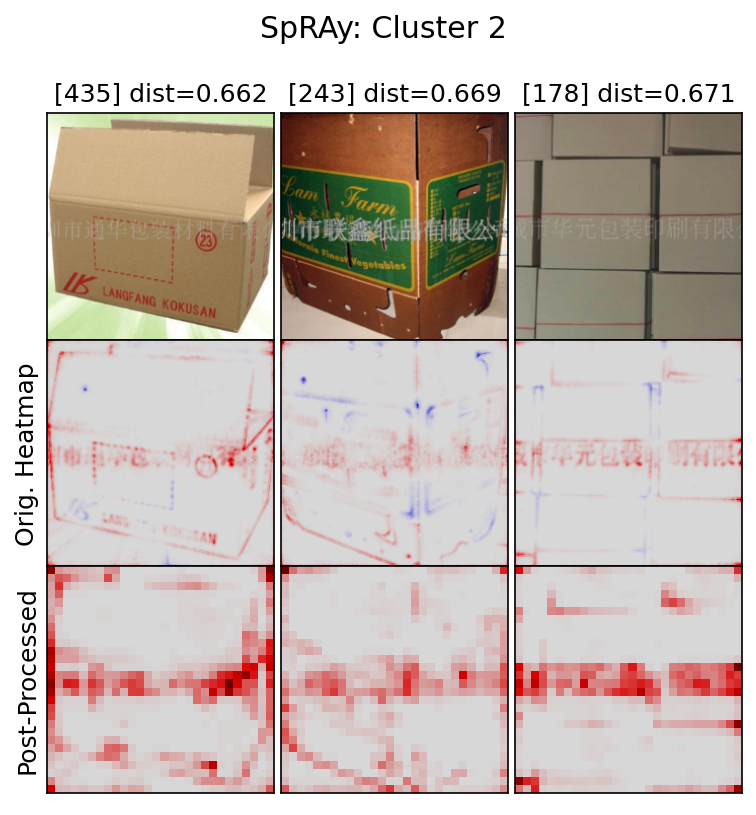}
    \includegraphics[height=0.175\textheight,clip,trim=0.75cm 0 0 0]{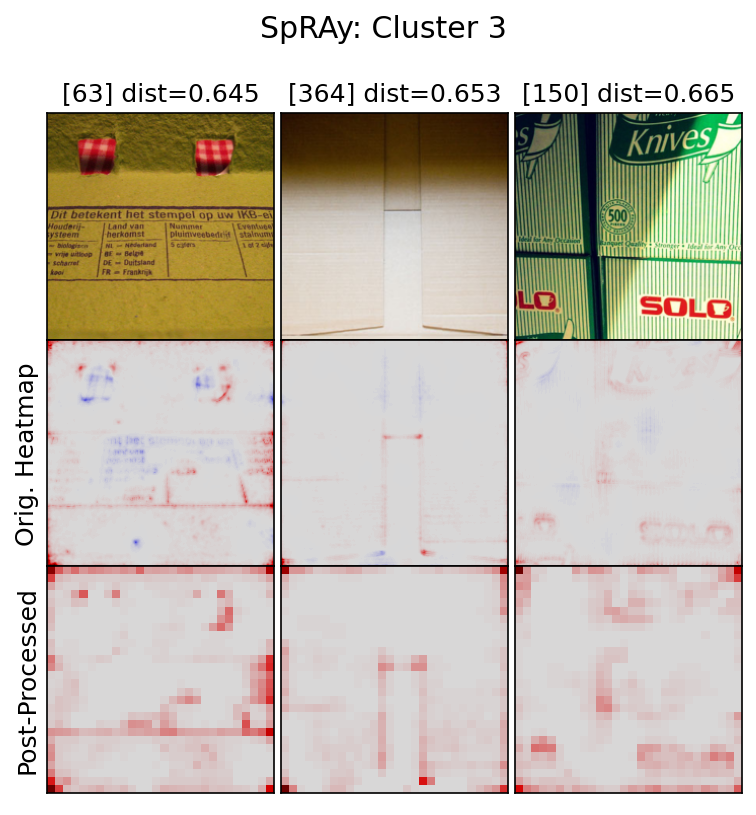}
    \includegraphics[height=0.175\textheight,clip,trim=0.75cm 0 0 0]{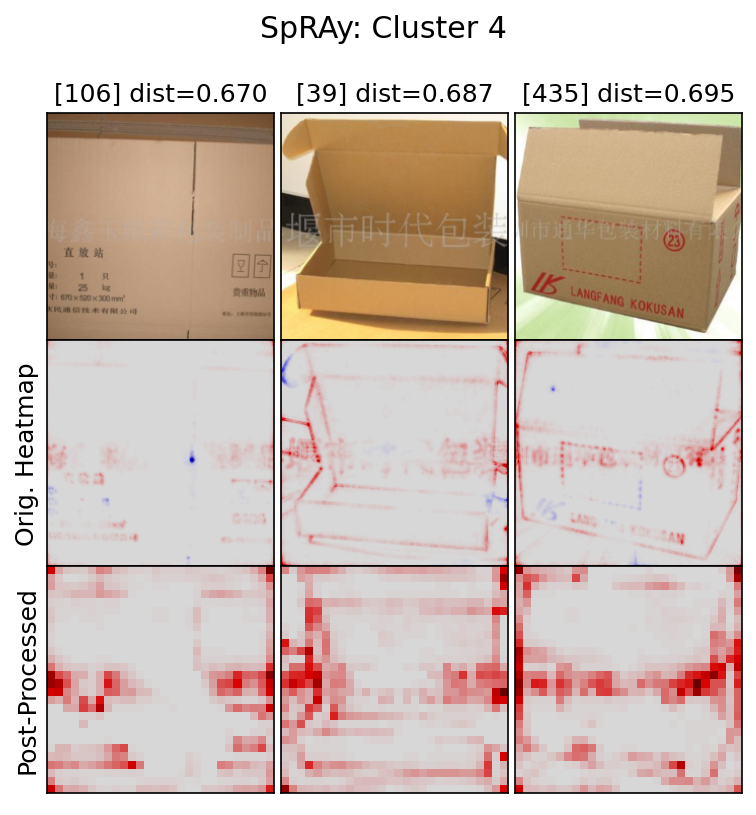}
\medskip

    \centering
    \includegraphics[height=0.175\textheight]{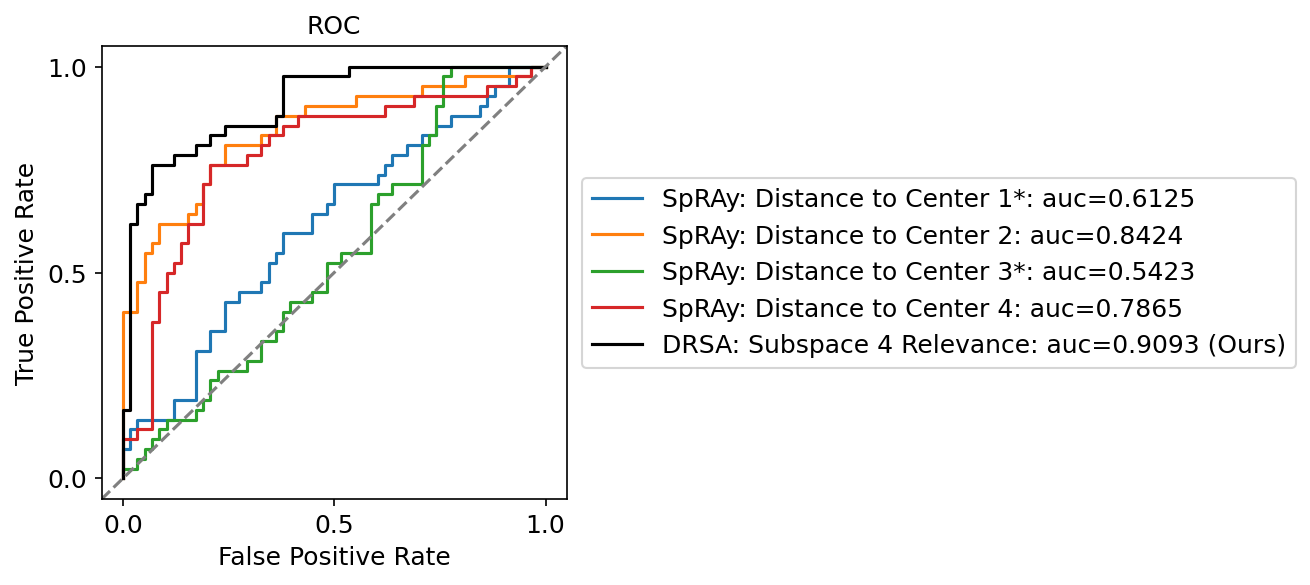}
    \caption{Top: Clusters of carton training images from Spectral Relevance Analysis \cite{Lapuschkin2019} (SpRAy). The second row shows the heatmaps of these training images, while the third row shows the ones after post-processing (cf.\ Section \ref{sup:sec-spray}). Bottom: Receiver Operating Characteristic (ROC) curves of DRSA and SpRAy in detecting the Hanzi watermark from a set of carton images. Asterisk indicates the setups that we use the negative of the distance to the corresponding cluster center.}
    \label{fig:spray}
\end{figure*}

\subsubsection{Comparison with Deep Feature Reweighting in Mitigating Effect of Hanzi Watermark}
Deep Feature Reweighting \cite{DBLP:conf/iclr/KirichenkoIW23} (DFR) is a competitive approach in mitigating the influence of spurious correlation. The approach is not only simple but also achieves state-of-the-art results. The approach consists of two steps: dataset construction and retraining the last layer of a standardly trained model.

For the first step, training data points are manipulated such that the appearance of  current spurious features are equally likely across data points. Then, the last layer of the underlying model is trained with the constructed data.

The goal of this experiment is to compare the effectiveness of DFR and our DRSA approach in mitigating the effect of the Hanzi watermark.

\textit{Experimental Setup for DFR}:  Similar to Section \ref{sup:showcase1-synthesis}, we take $500$ training images of class carton and the other four classes.  We construct a reweighting dataset by overlaying the Hanzi watermark on non-carton training images. We produce four such reweighting datasets with poisoning levels of $\{ 25\%, 50\%, 75\%, 100\%\}$. On each dataset, we train a new last layer of VGG16-TV using \texttt{MLPClassifier} of \cite{scikit-learn} with 100 epochs. We  select the value of the weight decay regularizer from $\{10^{-3}, 10^{-2}, \dots, 10^4 \}$  on another set of 500 training images.

\textit{Results:} We evaluate the effectiveness of DFR and our DRSA approach in mitigating the effect of the Hanzi watermark by measuring accuracy on the $25$\%- and $50$\%-poisoned validation sets (cf.\ Section \ref{sup:showcase1-synthesis}). Figure \ref{fig:showcase1-dfr} shows the accuracies of the DFR and DRSA approaches on these two poisoning levels. From the figure, we observe that the effectiveness of DFR on the poisoned data increases as the  poisoning level of the training data increases. Overall, DRSA and DFR both produce significant accuracy gains compared to the original model. We note that, unlike DFR, our approach requires neither creating a modified training set with artificially added artifacts, nor setting a poisoning rate hyperparameter, thereby making our approach easier to deploy.

\begin{figure}
    \centering
    \includegraphics[width=0.8\textwidth]{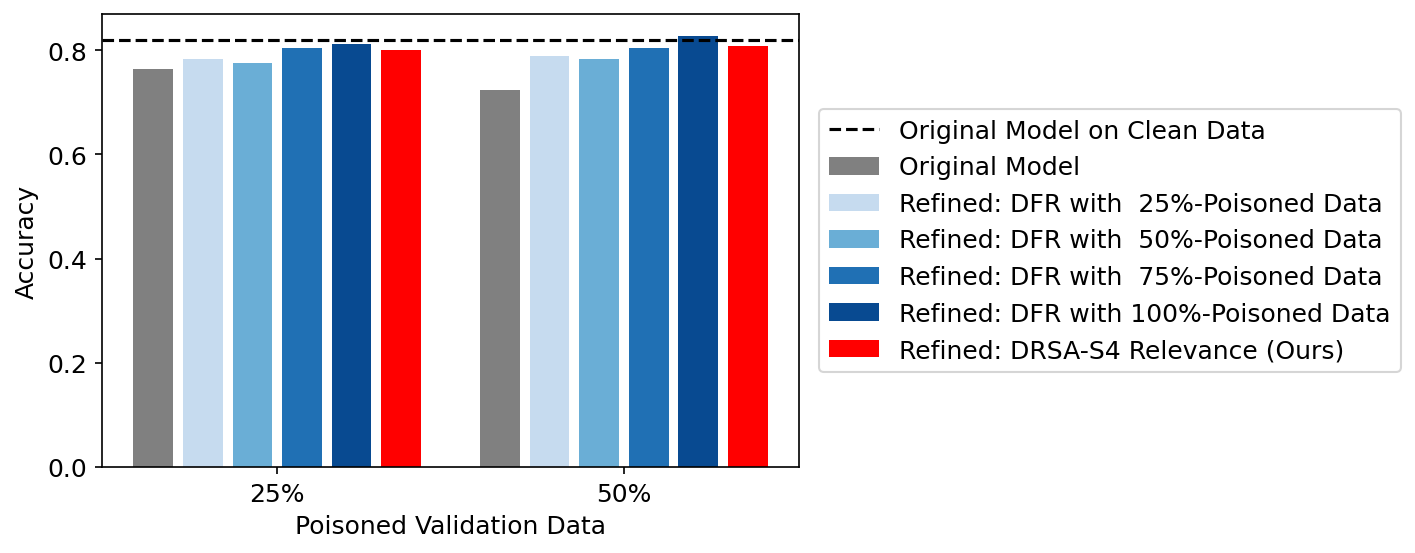}
    \caption{Accuracies of VGG16-TV models that are refined  using  Deep Feature Reweighting (DFR) or our DRSA approach to mitigate the influence of the Hanzi watermark on the prediction of non-carton images.}
    \label{fig:showcase1-dfr}
\end{figure}

\subsection{Showcase 2:  Better Insights via Disentangled Explanations}
\label{section:showcase2}
We provide additional results complementing the discussion in Section 6.2 of the main paper. Fig.\ \ref{fig:showcase-2-subspace-distribution} is a detailed version of Fig.\ 9 in the main paper. It illustrates the complete distributions of relevance scores across different subspaces and six butterfly classes. In the figure, we also highlight a prototypical  example of each  class, except class `lycaenid' that we randomly selected; we note that the ones of class `monach', `admiral', `sulphur', and `ringlet' are part  of Fig.\ 10 in the main paper.  Fig.\ \ref{fig:showcase-2-prototypes}  shows the standard and  DRSA heatmaps of these examples.

\begin{figure}[h!]
    \centering
    \includegraphics[clip,trim=0cm 0cm 0cm 1cm,width=0.8\textwidth]{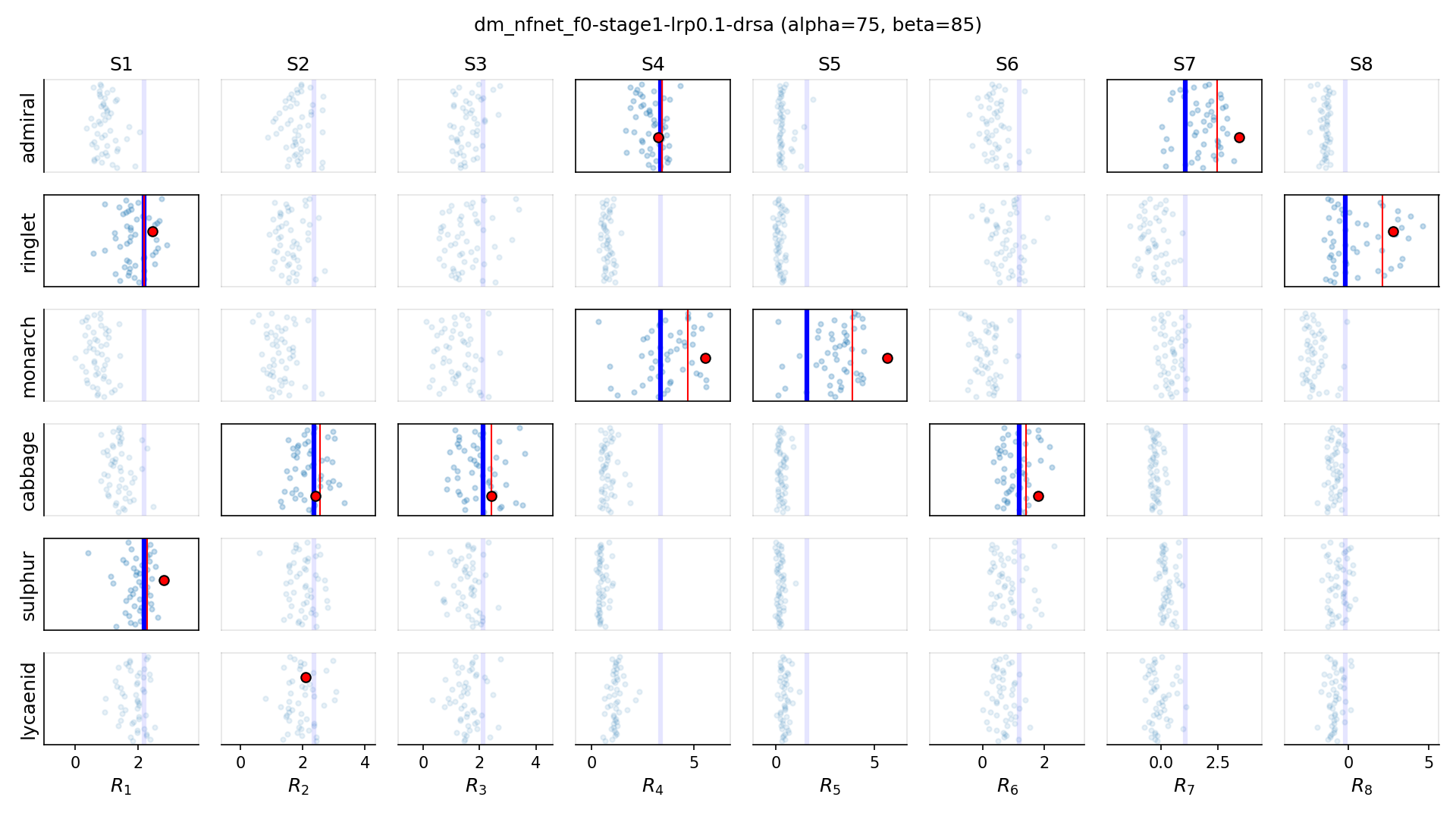}
    \caption{Distribution of relevance scores across six butterfly classes and subspaces extracted  using DRSA with activation and context vectors from NFNet-F0 at Stage 1. The red circles are  prototypical examples of each class, except the `lycaenid' image that we randomly selected.  The thick blue vertical lines represent the $\beta$-quantile  of the six-class distribution from each subspace, while the thin red lines indicate the $\alpha$-quantile of each class-subspace configuration. To ease visualization, we  dim the class-subspace configurations that do not pass the selection criteria (Eq.\ 13 in the main paper). }
    \label{fig:showcase-2-subspace-distribution}
\end{figure}

\begin{figure}[h!]
    \centering
    \includegraphics[clip,trim=7cm 2.5cm 8cm 2cm,width=\textwidth]{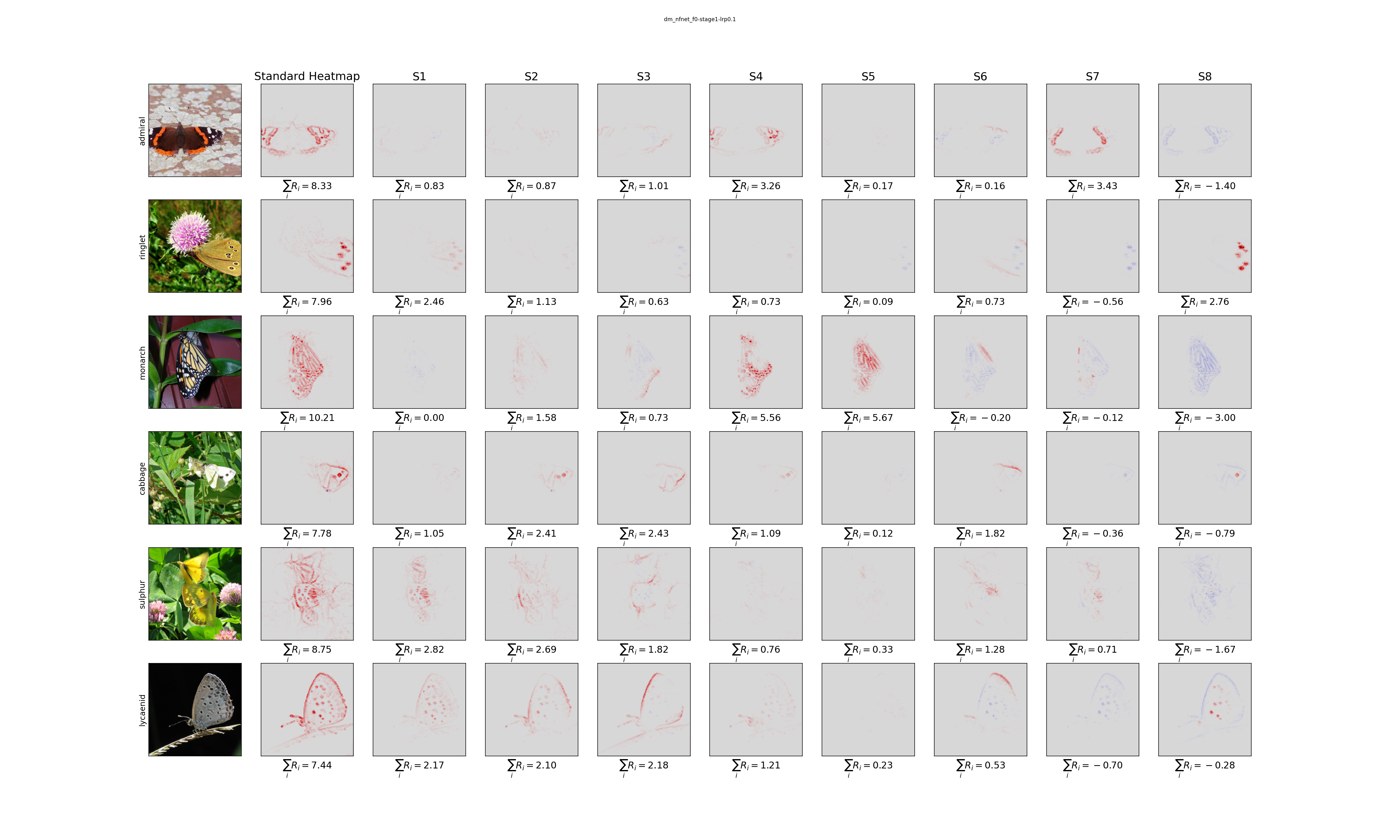}
    \caption{Prototypical butterfly examples with their standard  and DRSA heatmaps from \mbox{NFNet-F0} using \lrpgamma{}.  We extract subspaces  using  activation and context vectors from NFNet-F0 at Stage 1. We generate the heatmaps  w.r.t.\ the target class of each example.}
    \label{fig:showcase-2-prototypes}
\end{figure}

\subsection{Showcase 3: Analyzing Manipulated Explanations using PRCA}
\label{sec:showcase3}
We provide additional details  for Section 6.3 of the main paper. They include 1) the procedure that  perturbs an input and causes changes in its explanation; 2)  experimental setup; and 3) additional qualitative results.

\subsubsection{Finding Perturbation that Causes Arbitrary Changes in Explanations}
We use the optimization procedure proposed by \cite{DBLP:conf/nips/DombrowskiAAAMK19} to find a perturbation that leads to changes in explanation.
Consider an input $\boldsymbol{x} \in \mathbb{R}^P$ and its label $t \in \mathcal C$.
Let $f: \R^{P} \rightarrow \R^{|\mathcal{C}|}$ be a neural network and  $\mathcal E$  be an attribution method that produces an explanation $\mathcal E(f_t(\x), \x) \in \R^{P}$. Reference  \cite{DBLP:conf/nips/DombrowskiAAAMK19} formulates an optimization problem that finds  a perturbed version  $\hat{\x}$ of $\x$ such that
\begin{enumerate}
    \item the two inputs  $\x$ and $\hat{\x}$ are visually indistinguishable;
    \item  the model $f$ behaves approximately the same on these inputs, i.e.\ $\text{softmax}(f(\boldsymbol{x})) \approx \text{softmax}(f(\hat{\boldsymbol{x}}))$; 
    \item  but, the original $\mathcal E(f_t(\x), \x)$ and manipulated $\mathcal E(f_t(\hat{\x}), \hat{\x})$  explanations are substantially different. In practice, the difference is induced by making $\mathcal E(f_t(\hat{\x}), \hat{\x})$ similar to some  target explanation $\mathcal E_\text{target}$.
\end{enumerate}
More precisely, the objective of this constrained optimization problem is 
\begin{align}
\hat{\x} \leftarrow \argmin_{\x'} \| \mathcal E(f_t(\x'), \x') - \mathcal E_\text{target} \|^2 + \lambda \|\text{softmax}(f(\x')) - \text{softmax}(f(\x))  \|^2,
\label{eq:finding-adv-sample-objective}    
\end{align}
where $\lambda \in \R_+$ is a hyperparameter and where we choose $\mathcal E_\text{target}$ to be the explanation of a random input does not belong to class $t$. In addition, we rescale the target heatmap $\mathcal E_\text{target}$ with $\sum_p [\mathcal E(f_t(\x), \bx)]_p / \sum_p (\mathcal E_\text{target})_p$. The ratio constrains  the total relevance scores of the manipulated and original explanations are approximately the same, i.e.\ $ \sum_p [\mathcal E(f_t(\hat{\x}), \hat{\x})] \approx \sum_p [\mathcal E(f_t(\x), \x)]_p$.  

\subsubsection{Experimental Setup}

We focus on manipulating \lrpgamma{} heatmaps derived from VGG16-TV. More specifically, We perform the manipulation on the  50 validation images of class `tibetan terrier' from the ImageNet dataset \cite{imagenet_cvpr09}. We choose target heatmaps to be the heatmaps of  50 random images from other classes in the dataset; the heatmaps are produced w.r.t. the class of each random image. We extend the code provided by \cite{DBLP:conf/nips/DombrowskiAAAMK19} to support 1) \lrpgamma{} with the same heuristics we use for VGG16-TV\footnote{Reference \cite{dombrowski2022towards} uses LRP with the z$^+$ rule \cite{DBLP:journals/pr/MontavonLBSM17} for intermediate layers in their experiments.} and 2) the relevance preservation constraint. Similar to the original work, we use the same gradient-based iterative procedure and parameters to perform the optimization. We refer to our extended code repository\footnote{\url{https://github.com/p16i/explanations_can_be_manipulated}} for the details. We perform PRCA using the activation and context vectors from 500 training images of  class `tibetan terrier', and we extract these vectors at Conv4\_3.

\subsubsection{Additional Results}

Fig.\ \ref{fig:showcase3-correlation-original-manipulated-heatmaps} depicts that the total relevance scores between original and manipulated heatmaps are highly correlated. It indicates that the two heatmaps are approximately in the same scale.  Fig.\ \ref{fig:showcase3-additional-manipulated-samples} shows a number of selected images and their manipulated heatmaps. We see that these manipulated heatmaps are different from the original heatmaps but similar to the target ones. The results confirm that the optimization procedure proposed by \cite{DBLP:conf/nips/DombrowskiAAAMK19} also works under  the  relevance preservation constraint we impose. 

\begin{figure}[h!]
    \begin{center}
    \begin{minipage}{.4\textwidth}
    \includegraphics[width=\textwidth,clip,trim=0 0.3cm 0 1cm]{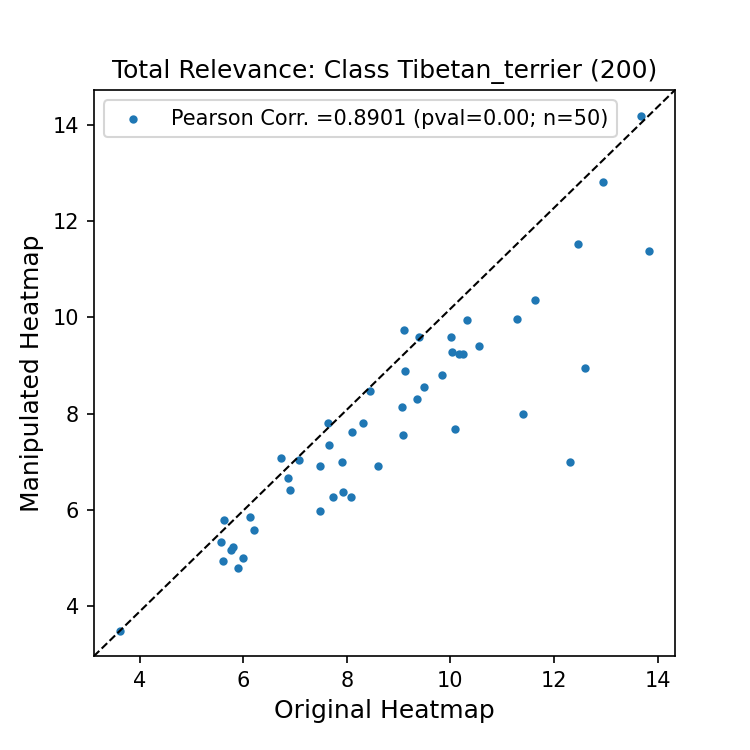}
    \end{minipage} 
    \begin{minipage}{.4\textwidth}
    \caption{Pearson correlation between the total \mbox{relevance} scores of original and manipulated heatmaps.}
    \label{fig:showcase3-correlation-original-manipulated-heatmaps}
    \end{minipage}
        
    \end{center}
    
\end{figure}

Fig.\ \ref{fig:showcase3-varying-prca-components} shows the decomposition  of manipulated heatmaps onto  PRCA with different subspace sizes and the correspondence residue (from the orthogonal complement of each  subspace). From the figure, we observe that the PRCA heatmaps generally preserve the main characteristics of the original heatmaps, while the positive part of the heatmap residues tends to contain the structure of the target heatmaps. More importantly, we also observe that incorporating more PRCA components makes the PRCA heatmaps become similar to the target heatmaps. The evidence suggests  that the first PRCA component is  the least affected by explanation manipulation.

\begin{figure}[h!]
    \centering
    \includegraphics[clip,trim=2.5cm 8cm 2.5cm 2cm,width=0.49\textwidth]{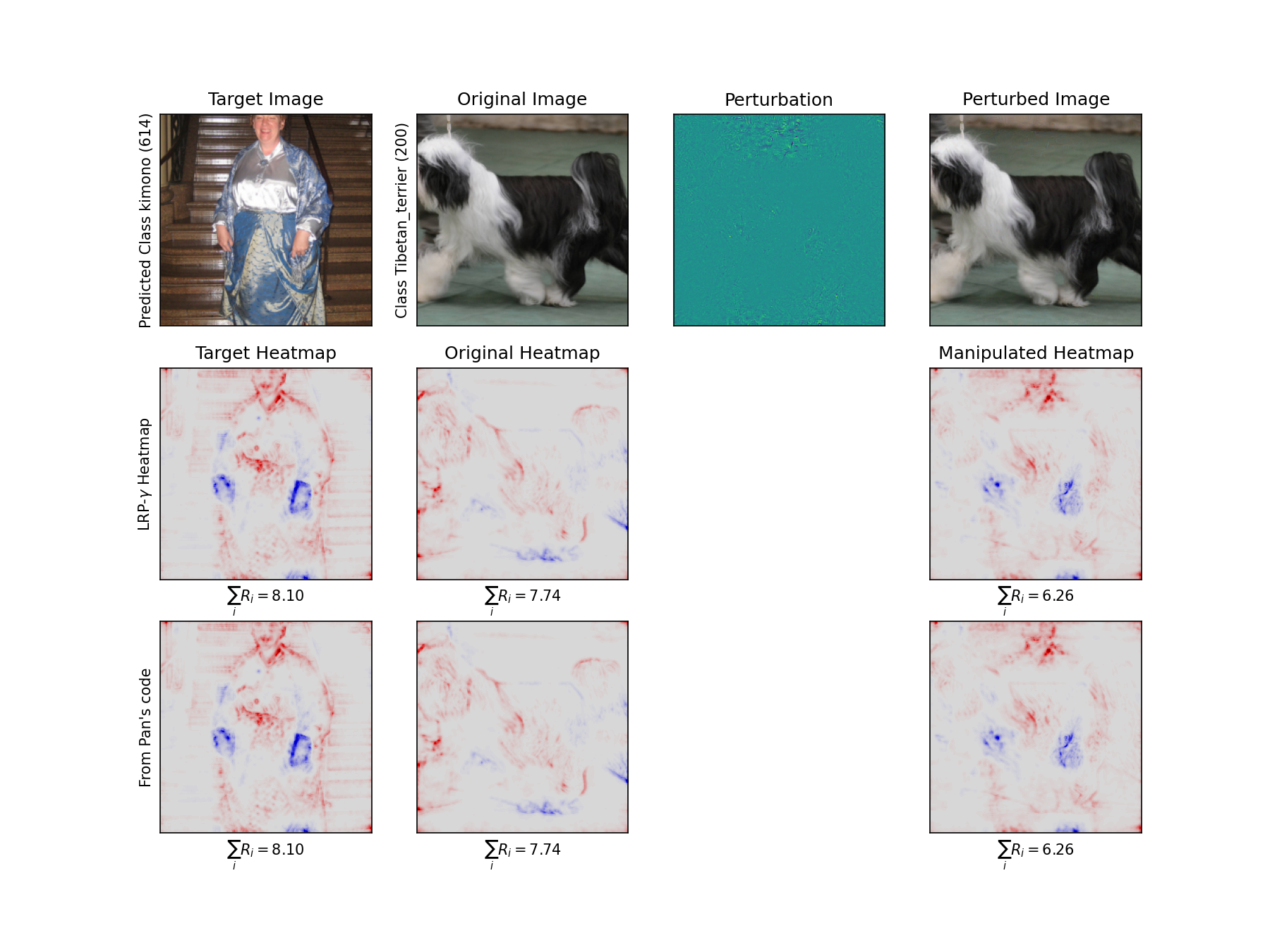} \hfill
    \includegraphics[clip,trim=2.5cm 8cm 2.5cm 2cm,width=0.49\textwidth]{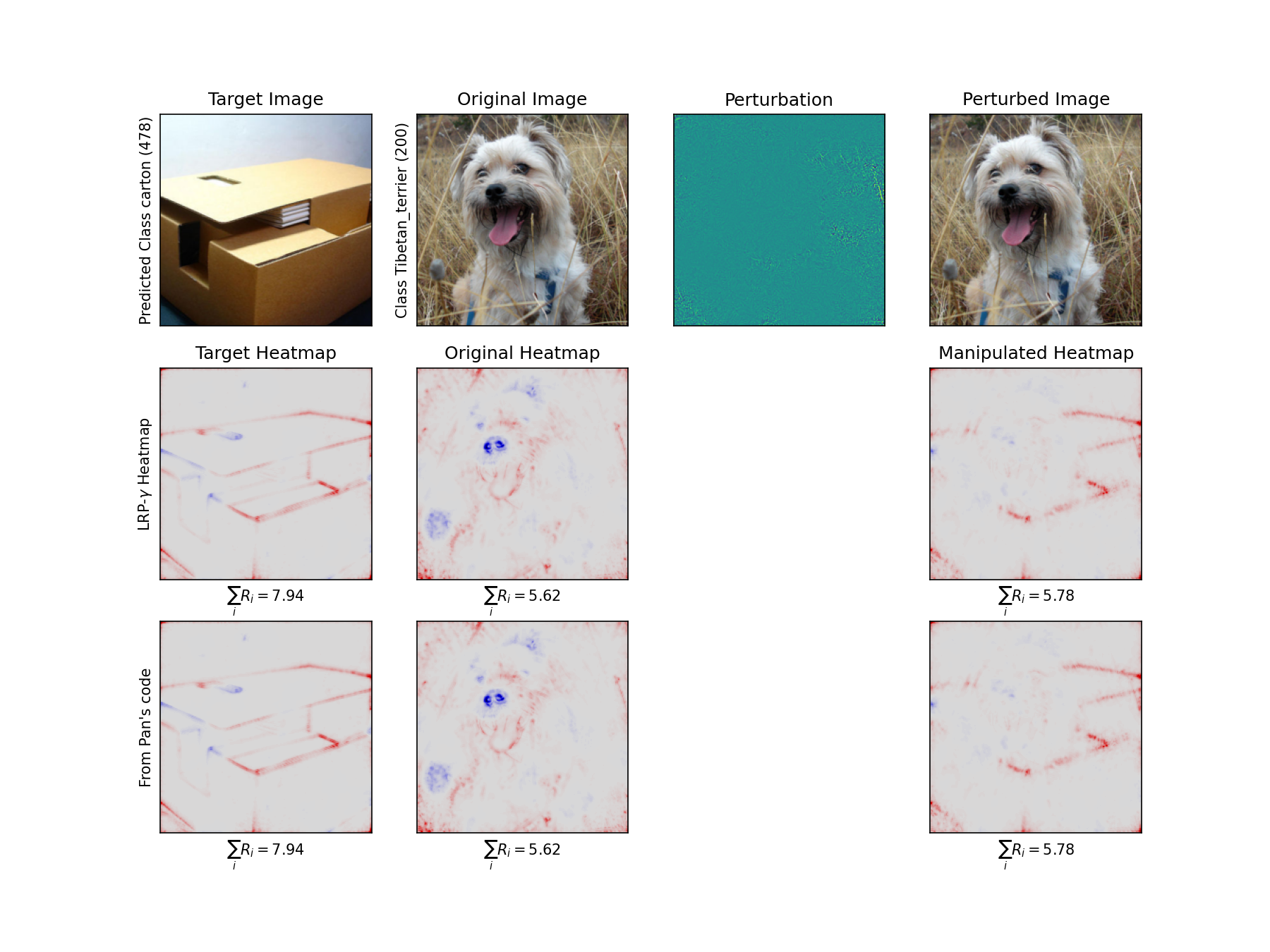}
    \caption{First row of each subplot: a random image used to produce a target heatmap, original image, perturbation noise, and perturbed image.  Second row of each subplot: target heatmap, heatmap of the original image, and heatmap of the perturbed image.}
    \label{fig:showcase3-additional-manipulated-samples}
\end{figure}

\begin{figure}[h!]
    \centering
    \includegraphics[clip,trim=5.5cm 1cm 4cm 1cm,width=.8\textwidth]{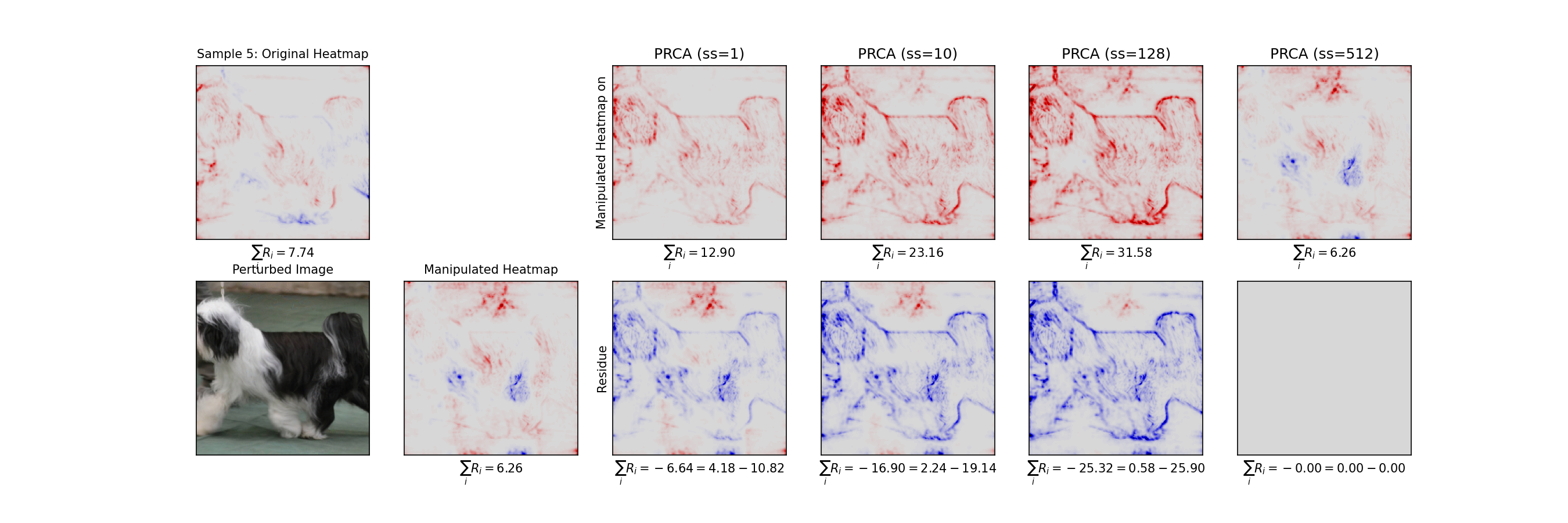}
    \includegraphics[clip,trim=5.5cm 1cm 4cm 1cm,width=.8\textwidth]{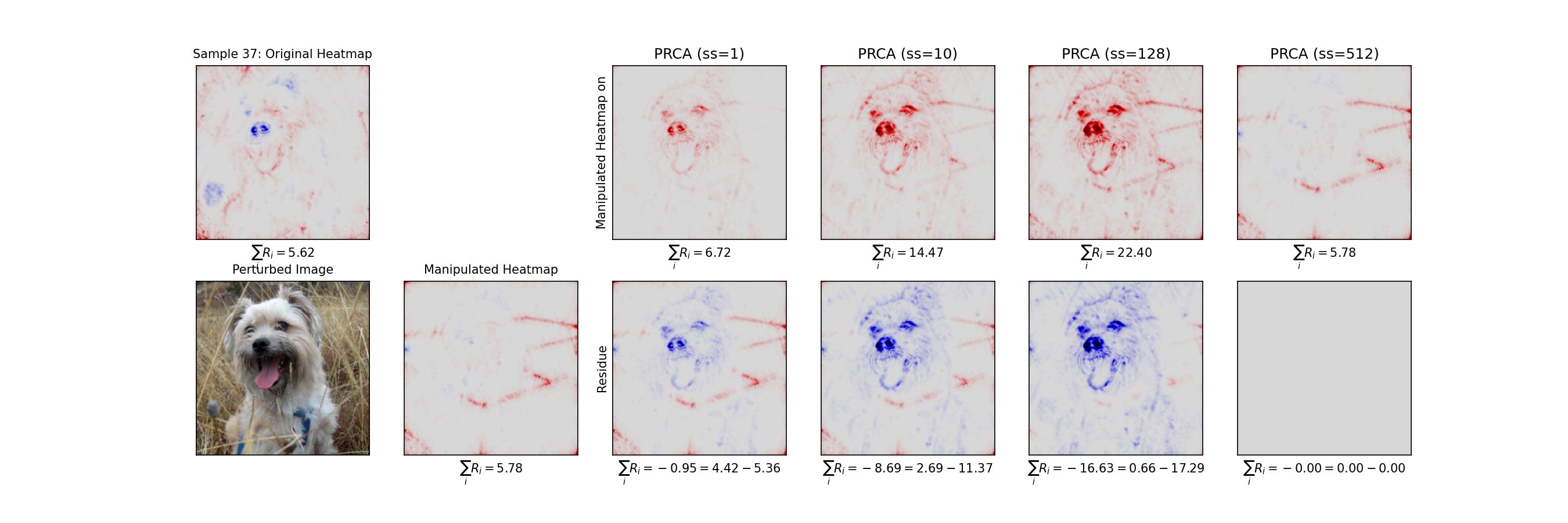}
    \caption{Decomposition of manipulated heatmaps using PRCA with different subspace sizes (`ss').}
    \label{fig:showcase3-varying-prca-components}
\end{figure}
\section{LRP Implementation for NFNets}
\label{sec:implementing-lrp-nfnet}
LRP is a propagation-based attribution method, and it comes with a number of propagation rules. One employs these rules to successively propagate a relevance quantity (e.g.\ logit value of a target class) from the last to the input layer. To use the method, one needs to choose an appropriate LRP propagation rule for each layer of the architecture.

Software packages like Innvestigate \cite{JMLR:v20:18-540}, Zennit \cite{anders2021software}, and Captum \cite{kokhlikyan2020captum} provide ready-to-use LRP implementation for common architectures. To the best of our knowledge, these packages have not yet implemented LRP for the recent state-of-the-art Normalizer-Free Network (NFNets) architecture \cite{pmlr-v139-brock21a}. Therefore, we close this technical gap by contributing a PyTorch \cite{Paszke_PyTorch_An_Imperative_2019} implementation of LRP for NFNets.

In the following, we first describe the NFNet architecture (Supplementary Note \ref{sec:nfnet-desc}).
Secondly, we categorize layer patterns in the  NFNet architecture into a number of cases (Supplementary Note \ref{sec:nfnetlrp-cases}); for each case, we define an appropriate LRP rule and outline its implementation. Finally, we discuss a technical step that improves the numerical stability of the  implementation (Supplementary Note \ref{sec:improve-numerical-stability-lrp}).  We refer to Fig.\ \ref{fig:lrp-final-heatmaps}  for example explanations from our LRP-NFNets  implementation.

\subsection{Normalizer-Free Networks (NFNets)}
\label{sec:nfnet-desc}
A number of state-of-the-art neural networks for image classification relies on BatchNorm \cite{DBLP:conf/icml/IoffeS15}. It is an important component that makes the training of these networks stable \cite{DBLP:conf/nips/SanturkarTIM18}. However, BatchNorm has a number of undesirable properties, e.g.\ breaking the independence assumption in the maximum likelihood principle \cite{pmlr-v139-brock21a} and introducing additional memory overhead \cite{DBLP:conf/cvpr/BuloPK18}.

As a result, \cite{pmlr-v139-brock21a} proposes the Normalizer-Free Networks (NFNets), whose design is based on the ResNet architecture \cite{DBLP:conf/cvpr/HeZRS16}. The unique aspect of the NFNets is the absence of BatchNorm. To this end, \cite{pmlr-v139-brock21a} proposes and employs a number of technical tricks to control the scale of the  activations and gradients, which used to be the role of BatchNorm. With these tricks, \cite{pmlr-v139-brock21a} demonstrates that,  for every training latency, NFNets have higher predictive performance than several state-of-the-art architectures (cf.\ Fig.\ 1 in \cite{pmlr-v139-brock21a}). 

We use the ImageNet-pretrained NFNets from the PyTorch Image Models package \cite{rw2019timm} (version `0.4.9'). These models use a customized activation function $\rho(x) = \tau \cdot \texttt{GELU}(x)$ where GELU is the Gaussian Linear Unit (GELU) \cite{DBLP:journals/corr/HendrycksG16} and $\tau$ is a positive constant.
We summarize the core components of NFNets in Fig.\ \ref{fig:nfnet}.

\begin{figure}[h]
    \centering
    \includegraphics[width=\textwidth]{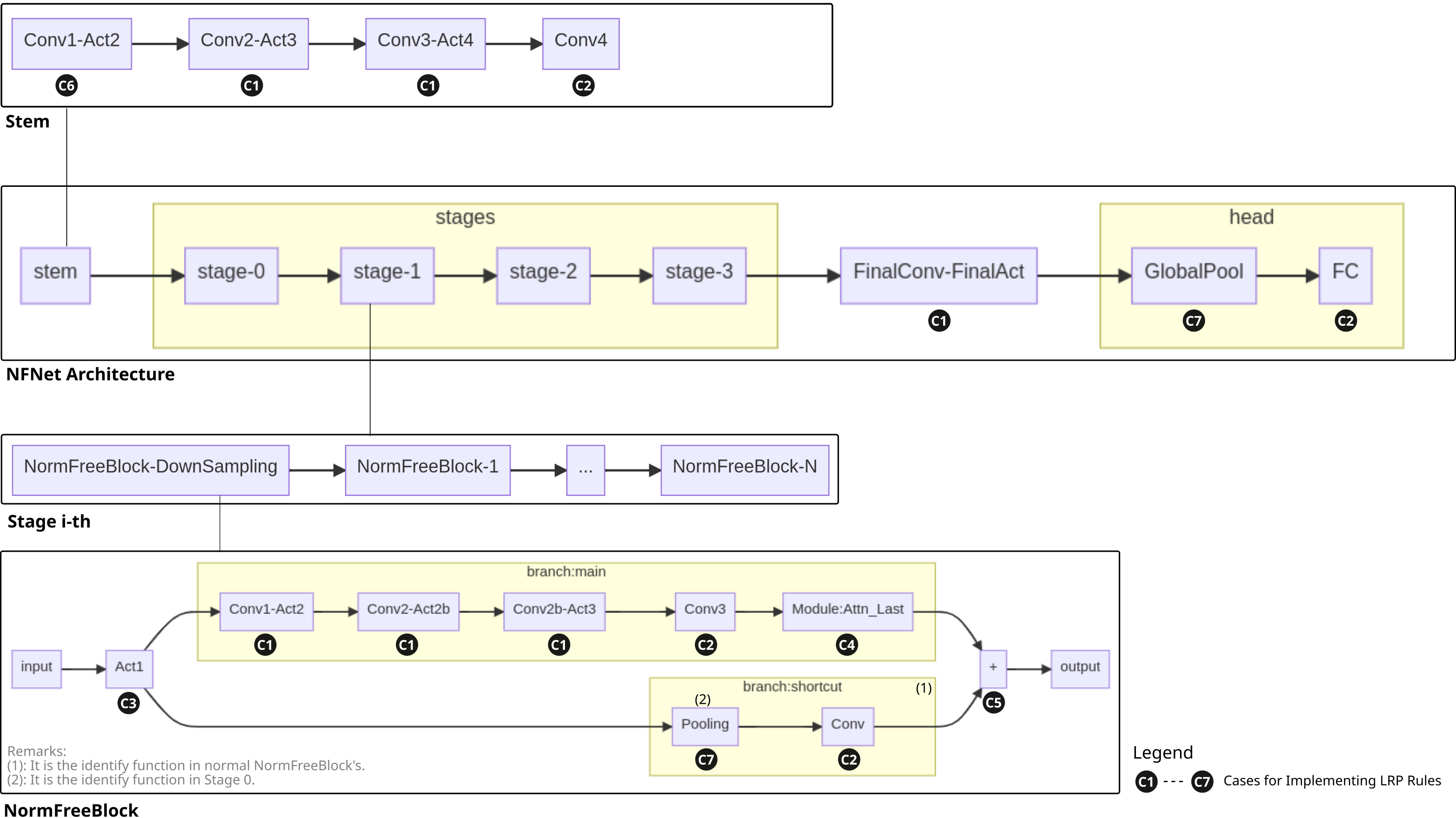}
    \caption{Components in Normalizer-Free Networks (NFNets \cite{pmlr-v139-brock21a}). The illustrated components  are extracted from the implementation of NFNets in the PyTorch Image Models \cite{rw2019timm} package (version `0.4.9'). We refer to Table \ref{tab:summary-lrp-nfnet-cases} for the summary of the depicted LRP implementation cases.}
    \label{fig:nfnet}
\end{figure}

\subsection{LRP Rules for NFNets}
\label{sec:nfnetlrp-cases}

\begin{table}[t!]
    \centering
    \caption{Summary of forward hook cases required to implement  LRP for NFNets}
    \label{tab:summary-lrp-nfnet-cases}
    \begin{tabular}{c|l|l}
    \toprule
    Supplementary Note & Pattern of Layers & Description  \\
    \midrule
    \ref{sec:lrpnfnet-case1} & Pairs of convolution and activation layers & Common layer pattern in NFNets  \\
    \ref{sec:lrpnfnet-case2} &  Orphan convolution layer or  last fully-connected layer &  \\ 
    \ref{sec:lrpnfnet-case3} &  Orphan activation layers  \\
    \ref{sec:lrpnfnet-case4} & Dimension-wise attention layers &  Squeeze-and-Excitation Block \cite{DBLP:journals/pami/HuSASW20} \\ 
    \ref{sec:lrpnfnet-case5} & Shortcut connections & Inline operator of the form $a + b$ \\ 
    \ref{sec:lrpnfnet-case6} & First pair of convolution and activation layers &\\ 
    \ref{sec:lrpnfnet-case7} & Pooling layers &  \\ 
    \bottomrule
    \end{tabular}
\end{table}

We  recall that the process of LRP can be seen as computing modified gradients of the output w.r.t. input features \cite{ancona17}. Leveraging the fact, one can  utilize the forward hook functionality of PyTorch \cite{Paszke_PyTorch_An_Imperative_2019}  to implement LRP rules \cite{samek-tnnls17,DBLP:series/lncs/MontavonBLSM19}. We refer to Algorithm \ref{algo:forwardpass-for-lrp} for a generic implementation of such forward hooks.

To this end, we analyze the NFNet architecture that is implemented in the PyTorch Image Models package \cite{rw2019timm} (version `0.4.9'). Our analysis shows that we require to construct seven forward hook cases to implement LRP for NFNets. Because this version of NFNets uses a modified  GELU activation function, these cases are some variants of the generalized \lrpgamma{} rule (Eq. \ref{eq:generalized-lrp}). We summarize these cases  in  \mbox{Table  \ref{tab:summary-lrp-nfnet-cases}}.

We  first introduce some notations: 
\begin{itemize}
    \item the vector $(a_j)_{j} \in \R^{d}$  is a $d$-dimensional activation vector;
    \item the vector $(z_k)_k \in \R^{d'}$  is a $d'$-dimensional pre-activated vector, i.e.\ $a_k = \rho(z_k)$;
    \item the matrix $(w_{jk})_{jk} \in \R^{d\times d'}$ is the weight matrix of a convolution or fully-connected layer;
    \item the variable $R(a_j), R(a_k) \in \R$ are the relevance received by the neurons $a_j$ and $a_k$ respectively;
    \item the functions $(\cdot)^+ = \max(0, \cdot)$ and  $(\cdot)^- = \min(0, \cdot)$.
\end{itemize}
  We also write layers in NFNets with gray-boxed text: for example, \torchmodule{Conv} is a convolution layer, while \torchmodule{Act} is a $\rho$-activation layer. We depict `FH-\torchmodule{Conv}' to be the forward hook of \torchmodule{Conv}. We use $(\cdot)^\dag$ to denote the quantity $(\cdot)$ is overridden in a forward hook and \texttt{[$\cdot$].detach()} to denote the `detach' functionality in PyTorch \cite{Paszke_PyTorch_An_Imperative_2019} that detaches the variable \texttt{[$\cdot$]} from the underlying computational graph. In the following, we assume that
\begin{assumption}
    \label{assumption:r-zk-the-same-r-ak}
    Let $\rho: \R \rightarrow \R$ be a sign-preserving activation function taking $z \mapsto \rho(z) = a$. The input $z$ and output $z$ have the same relevance, i.e.
    $R(z) = R(a).$
\end{assumption}

\begin{algorithm}[t]
\DontPrintSemicolon
  \caption{PyTorch Forward Hook for LRP \cite{DBLP:journals/pieee/SamekMLAM21}}
  \label{algo:forwardpass-for-lrp}
  \KwData{\texttt{module: torch.nn.Modulue}, \texttt{input: torch.Tensor}, \texttt{output: torch.Tensor}, \texttt{gamma: float}}

  \texttt{\# $\varphi(...)$ implements  the LRP rule of each case (see  Table \ref{tab:summary-lrp-nfnet-cases}).} \\
  \texttt{z} $\gets$ \texttt{$\varphi$(input, module, gamma)}
  
  \texttt{overridden\_output} $\gets$ \texttt{z * (output / z)}$\detach$

    \texttt{assert\_equal(overridden\_output, output, "sanity check")}

    \KwRet \texttt{overridden\_output}
\end{algorithm}

\subsubsection{Case : Pairs of Convolution and Activation Layers}
\label{sec:lrpnfnet-case1}

A pair of convolution and activation layers  is a common layer pattern in the NFNet architecture. The  computation of the two layers is
\begin{align*}
    z_k &= \sum_{j=1}^d w_{jk} a_j,  \tag*{\torchmodule{Conv}} \\
    a_k &= \rho(z_k)        \tag*{\torchmodule{Act}}.
\end{align*}
Using the generalized \lrpgamma{} and Assumption \ref{assumption:r-zk-the-same-r-ak}, the relevance of $a_j$ is  
\begin{align}
    R(a_j) &= \sum_k \frac{a_j^+(w_{jk} + \gamma w_{jk}^+) + a_j^- (w_{jk} + \gamma w_{jk}^-)}{\sum_j a_j^+(w_{jk} + \gamma w_{jk}^+) + a_j^- (w_{jk} + \gamma w_{jk}^-)}\cdot  1_{[a_k \ge 0]} \cdot  R(a_k) \nonumber \\
&\ \ \ \ \ \ + \sum_k \frac{a_j^+(w_{jk} + \gamma w_{jk}^-) + a_j^- (w_{jk} + \gamma w_{jk}^+)}{\sum_j a_j^+(w_{jk} + \gamma w_{jk}^-) + a_j^- (w_{jk} + \gamma w_{jk}^+)}\cdot   1_{[a_k < 0]} \cdot   R(a_k).
\label{eq:case-1}
\end{align}
Using forward hooks, one can achieve the  rule above by overriding the output of the two layers to to be 
    \begin{align*}
        z_k^\dagger &\gets \bigg [\sum_j\varphi^+_{\gamma, jk} \bigg] \cdot \bigg[ \frac{\rho(z_k)^+}{\sum_j\varphi^+_{\gamma, jk}} \bigg ]_\detach  + \bigg [\sum_j\varphi^-_{\gamma, jk} \bigg] \cdot \bigg[ \frac{\rho(z_k)^-}{\sum_j\varphi^-_{\gamma, jk}} \bigg ]_\detach \tag*{FH-\torchmodule{Conv}}\\
        a_k^\dagger & \gets  z_k^\dagger, \tag*{FH-\torchmodule{Act}}
    \end{align*}    
where $\varphi^+_{\gamma, jk} = a_j^+(w_{jk} + \gamma w_{jk}^+) + a_j^- (w_{jk} + \gamma w_{jk}^-)$ and $\varphi^-_{\gamma, jk}= a_j^+(w_{jk} + \gamma w_{jk}^-) + a_j^- (w_{jk} + \gamma w_{jk}^+)$.

\subsubsection{Case : Orphan Convolution Layers or Last Fully-Connected Layer}
\label{sec:lrpnfnet-case2}
Orphan convolution layers are  convolution layers that have no activation layer followed.  Commonly, these layers are   
\torchmodule{Conv3}  of  \torchmodule{NormFreeBlock}. The NormFreeBlock layers  are  in every stage of the NFNets (see Fig.\ \ref{fig:nfnet}). Computationally,  the last fully-connected layer \torchmodule{FC} (also known as the logit layer) is equivalent to those orphan convolution layers. The computation of these layers is
\begin{align*}
    z_k = \sum_{j=1}^d a_jw_{jk} \tag*{\torchmodule{Conv3} or \torchmodule{FC}}.
\end{align*}
Therefore, the relevance of $a_j$ and its forward hook are similar to what is described in Supplementary Note \ref{sec:lrpnfnet-case1}, except that we substitute $a_k$ in Eq.\ \eqref{eq:case-1} with $z_k$.

\subsubsection{Case : Orphan Activation Layers}
\label{sec:lrpnfnet-case3}
These orphan activation layers are \torchmodule{Act1} of every \torchmodule{NormFreeBlock}. The computation of these layers is  
\begin{align*}
a_k  = \rho(z_k). \tag*{\torchmodule{Act1}}
\end{align*}
With Assumption \ref{assumption:r-zk-the-same-r-ak}, we have $R(z_k) = R(a_k)$. Therefore, we override the output of these orphan activation layers to   be 
    \begin{align*}
        a_k^\dagger \leftarrow z_k \cdot \bigg[ \frac{a_k}{z_k} \bigg]_\texttt{.detach()}. \tag*{FH-\torchmodule{Act1}}
    \end{align*}

\subsubsection{Case : Attention Layer (also known as Squeeze-And-Excitation Block)}
\label{sec:lrpnfnet-case4}

The attention layer \torchmodule{AttnLast} is the second layer to last  of  \torchmodule{NormFreeBlock}. The layer is the implementation of the Squeeze-and-Excite layer \cite{DBLP:journals/pami/HuSASW20}, which performs dimension-wise modulation to the output. More specifically, the layer has a  function  $\mathcal A: \R^{d'} \rightarrow [0, 1]^{d'}$ and output
\begin{align*}
    a_k =  z_k \cdot \xi_k \tag*{\torchmodule{AttnLast}},
\end{align*}
where $ \xi_k = \mathcal A(\bz)_k$. This dimension-wise multiplicative structure is similar to the gating mechanism in LSTMs \cite{DBLP:journals/neco/HochreiterS97} and GRUs \cite{DBLP:conf/emnlp/ChoMGBBSB14}. With this structure,   \cite{arras-etal-2017-explaining,thomas-fron19} argue that  the modulation of  the function $\mathcal A$ already influences the relevance $R(z_k)$ in the forward pass and propose to directly compute the relevance $R(z_k)$   from $a_k$, without considering the  dependency to $a_k$ via the variable $\xi_k$.  We therefore have  
\begin{align}
    R(z_k) &= \frac{a_k}{z_k} \cdot R(a_k).
\end{align}
We can implement the rule by

\begin{align*}
    a_k^\dagger \leftarrow  z_k \cdot \bigg [ \frac{ a_k}{z_k} \bigg]_\texttt{.detach()}. \tag*{FH-\torchmodule{AttnLast}}
\end{align*}

\subsubsection{Case : Shortcut Connections}
\label{sec:lrpnfnet-case5}

The shortcut connection \torchmodule{Shortcut} is the last step of \torchmodule{NormFreeBlock}. The step combines  together inputs from two computational paths: the residual and shortcut paths. Let  $(z_k)_{k=1}^{d'}$ and $(s_k)_{k=1}^d$ be the input of these two paths respectively. The  input $(z_k)_k$ is from \torchmodule{AttnLast}, while the input $(s_k)_k$ is based on  the input of \torchmodule{NormFreeBlock}.  The computation of the step is
\begin{align*}
a_k = z_k   + s_k  \tag*{\torchmodule{Shortcut}}.
 \end{align*}
 We can interpret this step as a linear layer with $w_z = w_s = 1$, and the generalized LRP-$\gamma$ rule reduces to
\
 \begin{align}
        R(z_k) &= \bigg [ \frac{z_k^+ (1+\gamma) + z_k^-}{(z_k^+ +s_k^+) (1+\gamma) + (z_k^- +s_k^-)} \bigg ] \cdot 1_{[a_k \ge 0]}\cdot R(a_k) \nonumber \\ & \ \ \  + \bigg [ \frac{z_k^+ +z_k^- (1+\gamma)}{(z_k^+ + s_k^+) +(z_k^-  + s_k^-)(1+\gamma)} \bigg ] \cdot 1_{[a_k < 0]} \cdot R (a_k), \\\nonumber\\
        R(s_k) &= \bigg [ \frac{s_k^+ (1+\gamma) + s_k^-}{(z_k^+ +s_k^+) (1+\gamma) + (z_k^- +s_k^-)} \bigg ] \cdot 1_{[a_k \ge 0]}\cdot R(a_k) \nonumber \\ & \ \ \  + \bigg [ \frac{s_k^+ +s_k^- (1+\gamma)}{(z_k^+ + s_k^+) +(z_k^-  + s_k^-)(1+\gamma)} \bigg ] \cdot 1_{[a_k < 0]} \cdot R(a_k). 
    \end{align}
We can implement the rule above by 
\begin{align*}
        a_k^\dagger \leftarrow (\varphi^+_{\gamma, k} + \varphi^-_{0, k}) \cdot \bigg [\frac{a_k^+ }{\varphi^+_{\gamma, k} + \varphi^-_{0, k}}\bigg]_\texttt{.detach()} + (\varphi^+_{0, k} + \varphi^-_{\gamma, k}) \cdot \bigg [\frac{a_k^- }{\varphi^+_{0, k} + \varphi^-_{\gamma, k}}\bigg]_\texttt{.detach()}, \tag*{FH-\torchmodule{Shortcut}}
\end{align*}
where $\varphi_{\gamma,k}^+ = (z_k^+ + s_k^+)(1+\gamma)$ and $\varphi_{\gamma,k}^- = (z_k^- + s_k^-) ( 1 + \gamma)$. In practice, the shortcut connection is implemented as an in-place operation, which one can not attach any forward hook. To overcome the issue, one needs to replace this inplace operator with a PyTorch module that takes two inputs $(z_k)$ and $(s_k)$ and performs the addition.

\subsubsection{Case : First Convolution and Activation Layers}
\label{sec:lrpnfnet-case6}
These two layers are in the \torchmodule{Stem} block of NFNets. Denote \torchmodule{Stem.Conv1} and \torchmodule{Stem.Act2} to be these two layers. The computation is
\begin{align*}
    z_j &= \sum_p x_p w_{pj}, \tag*{\torchmodule{Stem.Conv1}}\\
    a_j &= \rho(z_j), \tag*{\torchmodule{Stem.Act2}}
\end{align*}
where $x_p \in \mathcal B$ is an input feature.
Although the computation of this case is similar to Supplementary Note \ref{sec:lrpnfnet-case1}, it requires a different treatment to properly attribute relevance to input features $x_p$'s. This is because in practice these input features are commonly normalized to be approximately in a range, i.e.\ $\mathcal{B} = [l_p, h_p]$ for some $l_p \le 0$ and  $h_p \ge 0 $.  Under this boundary condition, we utilize the LRP-$z^{\mathcal B}$ rule \cite{DBLP:journals/pr/MontavonLBSM17}, which lead to  
    \begin{align}
        R(x_p) = \sum_j \frac{x_p w_{pj} - (l_p w_{pj}^+ + h_p w_{pj}^-)}{\sum_p x_p w_{pj} - (l_p w_{pj}^+ + h_p w_{pj}^-)} \cdot R(a_j)
    \end{align}
We can implement the rule above by 
\begin{align*}
        z_j^\dagger &\gets \bigg[\sum_{p} \varphi^{\mathcal B}_{ pj} \bigg] \bigg[ \frac{\rho(z_j)}{\sum_{p} \varphi^{\mathcal B}_{pj} } \bigg]_\detach \tag*{FH-\torchmodule{Stem.Conv1}} \\ 
            a_j^\dag &\gets z_j^\dag \tag*{FH-\torchmodule{Stem.Act2}} ,
\end{align*}
where $\varphi^{\mathcal{B}}_{pj} = x_p w_{pj} - (l_p w_{pj}^+ + h_p w_{pj}^-)$. We refer to Supplementary Note \ref{appendix:attribution-lrp} on how the range $[l_p, h_p]$ can be chosen.

\subsubsection{Case : Pooling Layers}
\label{sec:lrpnfnet-case7}

These pooling  layers are the pooling layer in each  \torchmodule{NormFreeBlock.Downsample} or \torchmodule{GlobalPool}. Because these pooling are average pooling,  we can view their computation  a linear layer. More specifically, the computation is 
\begin{align*}
    z_k =    \sum_j a_j  w_{jk},  \tag*{\torchmodule{Pooling}}
\end{align*}
where $w_{jk} = 1 / n $ and $n$ is the size of the pooling kernel. Therefore, the relevance of $R(a_j)$ and its forward hook implementation is similar to Supplementary Note \ref{sec:lrpnfnet-case2}.

\subsection*{Improving Numerical Stability of LRP for NFNets}
\label{sec:improve-numerical-stability-lrp}
We have observed that the LRP implementation outlined in Supplementary Note \ref{sec:nfnetlrp-cases}
  produces  artifacts in  some images.  Our analysis suggested that these artifacts occur when $|\texttt{d}|$ in the  $[\texttt{n} /\texttt{d}]\detach$ operator is  small.  We found that setting  the detached variable to zero when  $|\texttt{d}| \le 10^{-3}$ mitigates the problem.

{
\bibliographystyle{ieeetr}
\bibliography{ref}
}

\end{appendices}